\documentclass[arxiv]{imsart}

\RequirePackage{amsthm, amsmath, bm}
\RequirePackage[numbers]{natbib}
\RequirePackage[colorlinks,citecolor=blue,urlcolor=blue]{hyperref}

\usepackage{commath}
\usepackage{mathtools}
\usepackage{multirow}
\usepackage{bbm}
\usepackage{comment}
\usepackage{amsfonts}       
\usepackage[utf8]{inputenc} 
\usepackage[T1]{fontenc}    
\usepackage{todonotes}
\usepackage{enumitem}
\usepackage{longtable}
\usepackage{array}
\usepackage{booktabs}
\usepackage{floatrow}
\usepackage{nicefrac}       
\usepackage{microtype}      

\startlocaldefs
    \newcommand{\pkg}[1]{{\normalfont\fontseries{b}\selectfont #1}}

    \DeclareMathOperator*{\argmax}{arg\,max}
    \DeclareMathOperator*{\argmin}{arg\,min}
    \DeclarePairedDelimiter\Abs{\lvert}{\rvert}%
    \DeclarePairedDelimiterX{\Norm}[1]{\lVert}{\rVert}{#1}

    \theoremstyle{plain}

    \newtheorem{theorem}{Theorem}[section]
    \newtheorem{lemma}[theorem]{Lemma}
    \newtheorem{definition}{Definition}[section]
    \newtheorem{assumption}{Assumption}[section]
    \newtheorem{condition}[theorem]{Condition}
    \newtheorem{corollary}[theorem]{Corollary}
    \newtheorem*{theorem*}{Theorem}
\endlocaldefs

\newcounter{reviewer}
\newcounter{point}[reviewer]
\renewcommand{\thepoint}{P\,\thereviewer.\arabic{point}} 

\newcommand{\shortreply}[2][]{\medskip \noindent \begin{sf}\textbf{Reply}:\  #2
	\ifthenelse{\equal{#1}{}}{}{ \hfill \footnotesize (#1)}%
	\medskip \end{sf}}

\begin{document}




\begin{frontmatter}

    \title{High-Dimensional Undirected Graphical Models for Arbitrary Mixed Data}
    \runtitle{High-Dimensional Mixed Graphs}

    \begin{aug}
        \author{\fnms{Konstantin} \snm{Göbler}
            \ead[label=e1]{konstantin.goebler@tum.de}}

        \address{Technical University of Munich \\
            \printead{e1}}

        \author{\fnms{Anne} \snm{Miloschewski}
            \ead[label=e2]{anne.miloschwski@dzne.de}}

        \address{German Center for Neurodegenerative Diseases (DZNE)\\
            Bonn, Germany\\
            \printead{e2}}

        \author{\fnms{Mathias} \snm{Drton}
            \ead[label=e3]{mathias.drton@tum.de}}

        \address{Munich Center for Machine Learning, \\ Technical University of Munich \\
            \printead{e3}}

        \author{\fnms{Sach} \snm{Mukherjee}
            \ead[label=e4]{sach.mukherjee@dzne.de}}

        \address{German Center for Neurodegenerative Diseases (DZNE)\\
        Bonn, Germany, \\ [1em]
        University of Cambridge,
        MRC Biostatistics Unit \\
        \printead{e4}}

        \runauthor{K. Göbler et al.}

    \end{aug}

    \begin{abstract}
        {
\noindent Graphical models are an important tool in exploring relationships between variables in complex, multivariate data. Methods for learning such graphical models are well-developed in the case where all variables are either continuous or discrete, including in high dimensions. However, in many applications, data span variables of different types (e.g., continuous, count, binary, ordinal, etc.), whose principled joint analysis is nontrivial. Latent Gaussian copula models, in which all variables are modeled as transformations of underlying jointly Gaussian variables, represent a useful approach. Recent advances have shown how the binary-continuous case can be tackled, but the general mixed variable type regime remains challenging. In this work, we make the simple but useful observation that classical ideas concerning polychoric and polyserial correlations can be leveraged in a latent Gaussian copula framework. Building on this observation, we propose a flexible and scalable methodology for data with variables of entirely general mixed type. We study the key properties of the approaches theoretically and empirically. 
}
    \end{abstract}




\end{frontmatter}


\section{Introduction}
\label{sec::intro}

Graphical models are widely used in the analysis of multivariate data, providing a convenient and interpretable way to study relationships among potentially large numbers of variables. They are key tools in modern statistics and machine learning and play an important role in diverse applications. Undirected graphical models are used in a wide range of settings, including, among others, systems biology, omics, deep phenotyping \cite[see, e.g.][]{dobra2004, Finegold11, monti2014}
and as a component within other analyses, including two-sample testing, unsupervised learning, hidden Markov modeling, and more \cite[examples include][]{wei2007,verzelen2009,stadler2013, stadler2015,perrakis2021}.

A significant portion of the literature on graphical models has concentrated on scenarios where either only continuous variables or only discrete variables are present. Regarding the former case, Gaussian graphical models have been extensively studied, including in the high-dimensional regime \cite[see among others][]{Meinshausen06, Friedman08, Banerjee08, Lam09, Yuan10, Ravikumar11, Cai11}. In such models, it is assumed that the observed random vector follows a multivariate Gaussian distribution, and the graph structure of the model is given by the zero pattern in the inverse covariance matrix. Generalizations for continuous, non-Gaussian data have also been studied \cite{Miyamura06, Liu09, Finegold11}. In the latter case, discrete graphical models -- related to Ising-type models in statistical physics -- have also been extensively studied \cite[see, e.g.][]{wainwright2006, ravikumar2010}.

However, in many applications, it is common to encounter data that entail \textit{mixed} variable types, i.e., where the data vector includes components of different types (e.g., continuous-Gaussian, continuous-non-Gaussian, count, binary, etc.). Such "column heterogeneity" (from the usual convention of samples in rows and variables in columns) is the rule rather than the exception. For instance, in statistical genetics, the construction of regulatory networks using expression profiling of genes may involve jointly analyzing gene expression levels alongside categorical phenotypes. Similarly, diagnostic data in many medical applications may contain continuous measurements such as blood pressure and discrete information about disease status or pain levels.

In analyzing such data, estimating a joint multivariate graphical model spanning the various variable types is often of interest. In practice, this is sometimes done using \textit{ad hoc} pipelines and data transformations. However, in graphical modeling, since the model output is intended to be scientifically interpretable and involves statements about properties such as conditional independence between variables, the use of \textit{ad hoc} workflows without an understanding of the resulting estimation properties is arguably problematic.

There have been three main lines of work that tackle high-dimensional graphical modeling for mixed data. The earliest approach is conditional Gaussian modeling of a mix of categorical and continuous data \cite{Lauritzen96} as treated by \citet{Cheng17, Lee15}. A second approach is to employ neighborhood selection, which amounts to separate modeling of conditional distributions for each variable given all others \cite[see, e.g.][]{Chen15, Yang14, Yang19}. A third approach uses latent Gaussian models, with a key recent reference being the paper of \citet{Fan17}, who proposed a latent Gaussian copula model for mixed data. The generative structure in their work posits that the discrete data is obtained from latent continuous variables thresholded at certain (unknown) levels.  However, in \cite{Fan17}, only a mix of binary and continuous data is considered. Their setting does not allow for more general combinations (including counts or ordinal variables) as found in many real-world applications.

This third approach will be the focus of this paper, which aims to provide a simple framework for working with latent Gaussian copula models to analyze general mixed data. To do so, we combine classical ideas concerning polychoric and polyserial correlations with approaches from the high-dimensional graphical models and copula literature. As we discuss below, this provides an overall framework that is scalable, general, and straightforward from the user's point of view.

Already in the early 1900s, \citet{Pearson1900, Pearson13} worked on the foundations of these ideas in the form of the tetrachoric and biserial correlation coefficients. From these arose the maximum likelihood estimators (MLEs) for the general version of these early ideas, namely the polychoric and the polyserial correlation coefficients. One drawback of these original measures is that they have been proposed in the context of latent Gaussian variables. A richer distributional family is the nonparanormal proposed by \citet{Liu09} as a nonparametric extension to the Gaussian family. A random vector $\boldsymbol{X} \in \mathbb{R}^d$ is a member of the nonparanormal family when $f(\boldsymbol{X}) = (f_{1}(X_{1}), \dots, f_{d}(X_{d}))^{T}$ is Gaussian, where $\{f_{k}\}_{k=1}^{d}$ is a set of univariate monotone transformation functions. Moreover, if the $f_j$'s are monotone and differentiable, the nonparanormal family is equivalent to the Gaussian copula family. As the polychoric and polyserial correlation assumes that observed discrete data are generated from latent continuous variables, they adhere to a latent copula approach.

We propose two estimators of the latent correlation matrix, which can subsequently be plugged into existing precision matrix estimation routines, such as the graphical lasso (glasso) \cite{Friedman08}, CLIME \cite{Cai11}, or the graphical Dantzig selector \cite{Yuan10}. The first is appropriate under a latent Gaussian model and unifies the aforementioned MLEs. The second is more general and is applicable under the latent Gaussian copula model. Both approaches can deal with discrete variables with arbitrarily many levels. We that both estimators exhibit favorable theoretical properties and include empirical results based on real and simulated data. The main contributions of the paper are as follows:
\begin{itemize}
    \item We posit that integrating polychoric and polyserial correlations into the latent Gaussian copula framework offers an elegant, straightforward, and highly effective approach to graphical modeling for comprehensively diverse mixed data sets.
    \item We present theoretical findings on the performance of the proposed estimators, encompassing their behavior in high-dimensional scenarios. The concentration results underscore the statistical validity of the introduced procedures.
    \item We empirically examine the estimators through a series of simulations and a practical example involving real phenotyping data of mixed types sourced from the UK Biobank. Our findings illustrate the practical utility of the proposed methods, demonstrating that their performance often closely aligns with an oracle model granted access to true latent data.
\end{itemize}

Our proposed procedure provides users with a method for conducting statistically sound graphical modeling of mixed data that is both straightforward to implement and carries no more overhead than conventional high-dimensional Gaussian graphical modeling approaches. Our procedure requires no manual specification of variable-type-specific model components, such as bridge functions.

The remainder of this paper is organized as follows. In Sections \ref{sec::gaussian} and \ref{sec::nonparanormal}, we present the estimators based on polychoric and polyserial correlations, including theoretical guarantees in terms of concentration inequalities. In Section \ref{sec::numerical_results}, we describe the experimental setup used to test the proposed approaches on simulated data together with the results themselves. 
We conclude with a summary of our findings in Section \ref{sec::conclusions} and point towards our \texttt{R} package \pkg{hume}, providing users with a convenient implementation of the methods developed in this study.

\section{Background and model set-up}
\label{sec::gaussian}

The objective of this paper is to learn the structure of undirected graphical models applicable to a wide range of mixed and high-dimensional data. To achieve this, we extend the Gaussian copula model \cite{Liu09, Liu12, Xue12}, enabling the incorporation of both discrete and continuous data of any nature.

\begin{definition}[The nonparanormal model]
    A random vector of continuous variables \(\mathbf{X} = (X_1, \dots, X_d)\) follows a $d$-dimensional nonparanormal distribution if there exists a set of monotone and differentiable univariate functions $f = \{f_1,\dots, f_d\}$ such that the transformed vector \(f(\mathbf{X}) = (f_1(X)_1, \dots, f_d(X)_d)\) is multivariate Gaussian with mean $0$ and covariance matrix \(\Sigma\), i.e. \(f(\mathbf{X})\sim N(0,\Sigma)\). We write
    \begin{equation}\label{def::nonparanormal}
        \mathbf{X} \sim \text{NPN}(0, \Sigma, f),
    \end{equation}
    where without loss of generality, the diagonal entries in \(\Sigma\) are equal to one.
\end{definition}

As demonstrated by \citet{Liu09}, the model in Eq. \eqref{def::nonparanormal} is a semiparametric Gaussian copula model. The following definition indicates how to extend this model to the presence of general mixed data.

\begin{definition}[latent Gaussian copula model for general mixed data]\label{latent_gaussian_cm}
    Let $\mathbf{X} = (\mathbf{X}_1,\mathbf{X}_2)$ be a $d$-dimensional random vector with \(\mathbf{X}_1\) a $d_1$-dimensional vector of possibly ordered discrete variables, and \(\mathbf{X}_2\) a $d_2$-dimensional vector of continuous variables with \(d = d_1 + d_2\). Suppose there exists a $d_1$-dimensional random vector of latent continuous variables $\mathbf{Z}_1 = (Z_1, \dots, Z_{d_1})^T$ such that the following relation holds:
    \begin{equation}\label{latent_ordered}
        X_j = x_j^{r} \quad if \quad \gamma_j^{r-1} \leq Z_j < \gamma_j^r \quad \text{for all } j = 1, \dots d_1 \ \text{and } r = 1, \dots, l_j+1,
    \end{equation}
    where $\gamma^r_j$ represents some unknown thresholds with $\gamma_j^0 = -\infty$ and $\gamma_j^{l_j +1} = +\infty$, $x^r_j \in \mathbb{N}_0$ and $l_{j} +1$ the number of discrete levels of $X_j$ for all $j \in 1, \dots, d_1$.

    Then, $\mathbf{X}$ satisfies the latent Gaussian copula model if $\mathbf{Z} \coloneqq (\mathbf{Z}_1, \mathbf{X}_2) \sim \text{NPN}(0, \mathbf{\Sigma}, f)$. We write
    \begin{equation}
        \mathbf{X} \sim \text{LNPN}(0, \mathbf{\Sigma}, f, \mathbf{\gamma}),
    \end{equation}
    where $\mathbf{\gamma} = \cup_{j=1}^{d_1} \{\gamma_j^r, r = 0, \dots, l_j+1\}$.
\end{definition}

Note that Definition \ref{latent_gaussian_cm} entails the class of Gaussian copula models if no discrete variables are present and the class of latent Gaussian models if $\mathbf{Z} = (\mathbf{Z}_1, \mathbf{X}_2) \sim \text{N}(0, \mathbf{\Sigma})$. As shown by \citet{Fan17}, the latent Gaussian copula model (LGCM) is invariant concerning any re-ordering of the discrete variables.

We denote \([d] = \{1,\dots,d\}\), \([d_1] = \{1,\dots,d_1\}\), and \([d_2] = \{d_1 +1, \dots, d_2\}\), respectively. Several identifiability issues arise in the latent Gaussian copula class.
First, the mean and the variances are not identifiable unless the monotone transformations \(f\) were restricted to preserve them. Note that this only affects the diagonal entries in \(\mathbf\Sigma\), not the full covariance matrix. Therefore, without loss of generality, we assume the mean to be the zero vector and \(\Sigma_{jj} = 1\) for all \(j \in [d]\). Another identifiability issue relates to the unknown threshold parameters. To ease notation, let \(\Gamma_j^r \equiv f_j(\gamma_j^r)\) and \(\Gamma_j \equiv \{f_j(\gamma_j^r)\}_{r=0}^{l_j+1}\). In the LGCM, only the transformed thresholds \(\Gamma_j\) rather than the original thresholds are identifiable from the discrete variables. We assume, without loss of generality, that the transformed thresholds retain the limiting behavior of the original thresholds, i.e., \(\Gamma_{j}^{0} = -\infty\) and \(\Gamma_j^{l_j+1} = \infty\).

Let $\mathbf{\Omega}= \mathbf{\Sigma}^{-1}$ denote the latent precision matrix. Then, the zero-pattern of $\mathbf{\Omega}$ under the LGCM still encodes the conditional independencies of the latent continuous variables \cite{Liu09}. Thus, the underlying undirected graph is represented by $\mathbf{\Omega}$ just as for the parametric normal. Note that the LGCM for general mixed data in Definition \ref{latent_gaussian_cm} agrees with that of \citet{Quan18} and of \citet{Feng19}. The problem phrased by \citet{Fan17} is a special case of  Definition \ref{latent_gaussian_cm}. A more detailed comparison between both approaches can be found in Section \ref{sec::nonparanormal}. Nominal discrete variables need to be transformed into a dummy system.

For the remainder of the paper, assume we observe an independent $n$-sample of the $d$-dimensional vector $\mathbf{X}$ which is assumed to follow an LGCM of the form \(\text{LNPN}(0, \mathbf{\Sigma}, f, \Gamma)\), where \(\Gamma = \cup_{j=1}^{d_1}\Gamma_j\). We estimate $\mathbf{\Sigma}$ by considering the corresponding entries separately i.e. the couples $(X_j, X_k)$ for \(j,k \in [d]\). Consequently, we have to keep in view three possible cases depending on the couple's variable types, respectively:

\begin{description}[labelwidth=4em,leftmargin =\dimexpr\labelwidth+\labelsep\relax, font=\mdseries]
    \item[\textit{Case I}:] Both $X_j$ and $X_k$ are continuous, i.e. $j,k \in [d_2]$.
    \item[\textit{Case II}:] $X_j$ is discrete and $X_k$ is continuous, i.e. \(j\in [d_1], k\in [d_2]\) and vice versa.
    \item[\textit{Case III}:] Both $X_j$ and $X_k$ are discrete, i.e. $j,k \in [d_1]$.
\end{description}

\subsection{Maximum-likelihood estimation under the latent Gaussian model}\label{sec::latent_gaussian}

At the outset, we examine each of the three cases under the latent Gaussian model, a special case of the LGCM where all transformations are identity functions. Consider \textit{Case I}, where both $X_j$ and $X_k$ are continuous. This corresponds to the regular Gaussian graphical model set-up discussed thoroughly, for instance, in \cite{Ravikumar11}. Hence, the estimator for $\mathbf\Sigma$ when both $X_j$ and $X_k$ are continuous is:
\begin{definition}[MLE $\hat{\mathbf{\Sigma}}^{(n)}$ of $\mathbf{\Sigma}$; \textit{Case I}]\label{def1}
    Let $\Bar{x}_j$ denote the sample mean of $X_j$. The estimator $\hat{\mathbf{\Sigma}}^{(n)} = (\hat{\Sigma}_{jk}^{(n)})_{d_1 < j < k\leq d_2}$ of the correlation matrix $\mathbf{\Sigma}$ is defined by:
    \begin{equation}
        \hat{\Sigma}_{jk}^{(n)} = \frac{\sum_{i=1}^n(x_{ij}- \Bar{x}_j)(x_{ik}- \Bar{x}_k)}{\sqrt{\sum_{i=1}^n(x_{ij}- \Bar{x}_j)^2} \sqrt{\sum_{i=1}^n(x_{ik}- \Bar{x}_k)^2}}
    \end{equation}
    for all $d_1 < j < k \leq d_2$.
\end{definition}
This is the Pearson product-moment correlation coefficient, which, of course, coincides with the maximum likelihood estimator (MLE) for the bivariate normal couple $\{(X_j, X_k)\}_{i=1}^n$.

Turning to \textit{Case II}, let $X_j$ be ordinal and $X_k$ be continuous. We are interested in the product-moment correlation $\Sigma_{jk}$ between two jointly Gaussian variables, where $X_j$ is not directly observed but only the ordered categories (see Eq. \eqref{latent_ordered}). This is called the \textit{polyserial} correlation \cite{Olsson82}. The likelihood and log-likelihood of the $n$-sample are defined by:
\begin{equation}\label{polyserial_likelihood}
    \begin{split}
        L_{jk}^{(n)}(\Sigma_{jk}, x_j^r,x_k) &= \prod_{i=1}^n p(x^r_{ij},x_{ik}, \Sigma_{jk}) = \prod_{i=1}^n p(x_{ik})p(x^r_{ij} \mid x_{ik}, \Sigma_{jk}) \\
        \ell_{jk}^{(n)}(\Sigma_{jk}, x^r_j,x_k) &= \sum_{i=1}^n \big[\log(p(x_{ik})) + \log(p(x^r_{ij} \mid x_{ik}, \Sigma_{jk}))\big],
    \end{split}
\end{equation}
where $p(x_{ij}^{r},x_{ik}, \Sigma_{jk})$ denotes the joint probability of  $X_j$ and $X_k$ and $p(x_{ik})$ the marginal density of the Gaussian variable $X_k$. MLEs are obtained by differentiating the log of the likelihood in Eq. \eqref{polyserial_likelihood} with respect to the unknown parameters, setting the partial derivatives to zero, and solving the system of equations for $\Sigma_{jk}, \mu, \sigma^2$, and $\Gamma_j^r$ for $r \in [l_j]$. Under the latent Gaussian model, we have the special case that the thresholds are identifiable from the observed data as \(\Gamma_j^r = \gamma_j^r\).

\begin{definition}[MLE $\hat{\mathbf{\Sigma}}^{(n)}$ of $\mathbf{\Sigma}$; \textit{Case II}]\label{definition_case2}
    Recall the log-likelihood in Eq. \eqref{polyserial_likelihood}. The estimator $\hat{\mathbf{\Sigma}}^{(n)} = (\hat{\Sigma}_{jk}^{(n)})_{1 < j \leq d_1 < k \leq d_2}$ is defined by:
    \begin{equation}
        \begin{split}
            \hat{\Sigma}_{jk}^{(n)} &= \argmax_{\Abs{\Sigma_{jk}} \leq 1} \ell_{jk}^{(n)}(\Sigma_{jk}, x^r_j,x_k) \\
            &= \argmax_{\Abs{\Sigma_{jk}} \leq 1} \frac{1}{n} \ell_{jk}^{(n)}(\Sigma_{jk}, x^r_j,x_k)
        \end{split}
    \end{equation}
    for all $1 < j \leq d_1 < k \leq d_2$.
\end{definition}
\noindent Regularity conditions ensuring consistency and asymptotic efficiency, as well as asymptotic normality, can be verified to hold here \cite{Cox74}.

Lastly, consider \textit{Case III}, where both $X_j$ and $X_k$ are ordinal. The probability of an observation with $X_j = x^r_j$ and $X_k = x^s_k$ is given by
\begin{equation}\label{cell_probabilities}
    \begin{split}
        \pi_{rs} &\coloneqq p(X_j = x^r_j, X_k = x^s_k) \\
        &= p(\Gamma_j^{r-1} \leq Z_j < \Gamma_j^r, \Gamma_k^{s-1} \leq Z_k < \Gamma_k^s) \\
        &= \int_{\Gamma_j^{r-1}}^{\Gamma_j^{r}} \int_{\Gamma_k^{s-1}}^{\Gamma_k^{s}} \phi(z_j,z_k,\Sigma_{jk}) dz_j dz_k,
    \end{split}
\end{equation}
where $r = 1, \dots, l_j$ and $s = 1, \dots, l_k$ and $\phi(x,y,\rho)$ denotes the standard bivariate density with correlation $\rho$. Then, as outlined by \citet{Olsson79} the likelihood and log-likelihood of the $n$-sample are defined as:
\begin{equation}\label{polychoric_likelihood}
    \begin{split}
        L_{jk}^{(n)}(\Sigma_{jk}, x_j^r,x_k^s) &= C \prod_{r=1}^{l_{{j}}} \prod_{s=1}^{l_{{k}}} \pi_{rs}^{n_{rs}}, \\
        \ell_{jk}^{(n)}(\Sigma_{jk}, x_j^r,x_k^s) &= \log(C) + \sum_{r=1}^{l_{{j}}}\sum_{s=1}^{l_{k}} n_{rs} \log(\pi_{rs}),
    \end{split}
\end{equation}
where $C$ is a constant and $n_{rs}$ denotes the observed frequency of $X_j = x^r_j$ and $X_k = x^s_k$ in a sample of size $n= \sum_{r=1}^{l_{{j}}}\sum_{s=1}^{l_{{k}}} n_{rs}$. Differentiating the log-likelihood, setting it to zero, and solving for the unknown parameters yields the estimator for $\Sigma$ for \textit{Case III}:
\begin{definition}[MLE $\hat{\mathbf{\Sigma}}^{(n)}$ of $\mathbf{\Sigma}$; \textit{Case III}]\label{definition_case3}
    Recall the log-likelihood in Eq. \eqref{polychoric_likelihood}. The estimator $\hat{\mathbf{\Sigma}}^{(n)} = (\hat{\Sigma}_{jk}^{(n)})_{1\leq j < k\leq d_1}$ of $\mathbf{\Sigma}$ is defined by:
    \begin{equation}
        \begin{split}
            \hat{\Sigma}_{jk}^{(n)} &= \argmax_{\Abs{\Sigma_{jk}} \leq 1} \ell_{jk}^{(n)}(\Sigma_{jk}, x_j^r,x_k^s) \\
            &= \argmax_{\Abs{\Sigma_{jk}} \leq 1} \frac{1}{n} \ell_{jk}^{(n)}(\Sigma_{jk}, x_j^r,x_k^s),
        \end{split}
    \end{equation}
    for all $1 < j < k \leq d_1 $.
\end{definition}
\noindent Regularity conditions ensuring consistency and asymptotic efficiency, as well as asymptotic normality, can again be verified to hold here \cite{Wallentin17}.

Summing up, under the latent Gaussian model, a special case of the LGCM, $\hat{\mathbf{\Sigma}}^{(n)}$ is a consistent and asymptotically efficient estimator for the underlying latent correlation matrix $\mathbf{\Sigma}$. Corresponding concentration results are derived in Section \ref{sec::convergence_results}.

\section{Latent Gaussian Copula Models}
\label{sec::nonparanormal}

\citet{Fan17} propose the binary LGCM, a special case of the LGCM allowing for the presence of binary and continuous variables. Following the approach of the nonparanormal SKEPTIC \citep{Liu12}, they circumvent the direct estimation of monotone transformation functions $\{f_j\}_{j=1}^d$ by employing rank correlation measures, such as Kendall's tau or Spearman's rho. These measures remain invariant under monotone transformations. Notably, for \textit{Case I}, a well-known mapping exists between Kendall's tau, Spearman's rho, and the underlying Pearson correlation coefficient $\Sigma_{jk}$. As a result, the primary contribution of \citet{Fan17} lies in deriving corresponding bridge functions for cases II and III. To reduce computational burden \citet{Yoon21} propose a hybrid multilinear interpolation and optimization scheme of the underlying latent correlation.

When considering the general mixed case, \citet{Fan17} advocate for binarizing all ordinal variables. This concept has been embraced by \citet{Feng19}, who suggest an initial step of binarizing all ordinal variables to create preliminary estimators. Subsequently, these estimators are meaningfully combined using a weighted aggregate. To extend the binary latent Gaussian copula model and explore generalizations regarding bridge functions, \citet{Quan18} ventured into scenarios where a combination of continuous, binary, and ternary variables is present. However, a notable drawback of this approach becomes evident. Dealing with a mix of binary and continuous variables requires three bridge functions -- one for each case. The complexity grows as discrete variables introduce distinct state spaces. In fact, a combination of continuous variables and discrete variables with $k$ different state spaces necessitates $\binom{k+2}{2}$ bridge functions.

For this reason, we adopt an alternative approach to the latent Gaussian copula model when dealing with general mixed data, allowing discrete variables to possess any number of states. In this strategy, the number of cases to be considered remains consistent at three, as already introduced in the preceding section.

\subsection{Nonparanormal Case I}\label{sec::nonparanormal_case1}

For \textit{Case I}, the mapping between $\Sigma_{jk}$ and the population versions of Spearman's rho and Kendall's tau is well known \citep{Liu09}. Here we make use of Spearman's rho $\rho^{\text{\textit{Sp}}}_{jk} = corr(F_j(X_j),F_k(X_k))$ with $F_j$ and $F_k$ denoting the cumulative distribution functions (CDFs) of $X_j$ and $X_k$, respectively. Then $\Sigma_{jk} = 2\sin{\frac{\pi}{6} \rho^{\text{\textit{Sp}}}_{jk}} \quad \text{for } d_1  < j < k \leq d_2$. In practice, we use the sample estimate
\begin{equation*}
    \hat{\rho}^{\text{\textit{Sp}}}_{jk} = \frac{\sum_{i=1}^n (R_{ji} - \Bar{R}_{j}) (R_{ki} - \Bar{R}_{k})}{\sqrt{\sum_{i=1}^n(R_{ji} - \Bar{R}_{j})^2\sum_{i=1}^n(R_{ki} - \Bar{R}_{k})^2}},
\end{equation*}
with $R_{ji}$ corresponding to the rank of $X_{ji}$ among $X_{j1}, \dots, X_{jn}$ and $\Bar{R}_{j} = 1/n \sum_{i=1}^n R_{ji} = (n+1)/2$; compare \cite{Liu12}. From this, we obtain the following estimator:
\begin{definition}[Nonparanormal estimator $\hat{\mathbf{\Sigma}}^{(n)}$ of $\mathbf{\Sigma}$; \textit{Case I}]\label{case1_nonpara}
    The estimator $\hat{\mathbf{\Sigma}}^{(n)} = (\hat{\Sigma}_{jk}^{(n)})_{d_1 < j< k\leq d_2}$ of the correlation matrix $\mathbf{\Sigma}$ is defined by:
    \begin{equation}\label{spearman_mapping}
        \hat{\Sigma}_{jk}^{(n)} = 2\sin{\frac{\pi}{6} \hat{\rho}^{\text{\textit{Sp}}}_{jk}},
    \end{equation}
    for all $d_1 < j < k \leq d_2$.
\end{definition}

\subsection{Nonparanormal Case II}\label{sec::nonparanormal_case2}

In \textit{Case II}, the complexity increases. Employing a rank-based approach for the nonparanormal model makes direct application of the ML procedure unfeasible, given that the continuous variable is not observed in its Gaussian form. Nevertheless, a two-step approach remains viable. First, an estimate of \(f_j\) must be formulated and subsequently employed in Definition \ref{definition_case2}. Yet, scrutinizing convergence rates for this procedure poses challenges, as the estimated transformation appears in multiple instances within the first-order condition of the MLE. Section 2 in the Supplementary Materials provides further details and compares the two-step likelihood approach to the one we propose below.

Instead, we will proceed by suitably modifying other approaches that address the Gaussian case through a more direct \textit{ad hoc} examination of the relationship between $\Sigma_{jk}$ and the point polyserial correlation \citep{Bedrick92, Bedrick96}.
Section 2 of the Supplementary Materials compares the nonparanormal \textit{Case II} estimation strategies.

In what follows, in the interest of readability, we omit the index in the monotone transformation functions but explicitly allow them to vary among the $\mathbf{Z}$. According to Definition \ref{def1}, we have the following Gaussian conditional expectation
\begin{equation}
    E[f(X_k) \mid f(Z_j)] = \mu_{f(X_k)} + \Sigma_{jk}\sigma_{f(X_k)} f(Z_j), \quad \text{for } 1 \leq j \leq d_1 < k \leq d_2,
\end{equation}
where we can assume w.l.o.g. that $\mu_{f(X_k)} = 0$. After multiplying both sides with the discrete variable $X_j$, we move it into the expectation on the left-hand side of the equation. This is permissible as $X_j$ is a function of $f(Z_j)$, i.e.
\begin{equation*}
    E[f(X_k)X_j \mid f(Z_j)] = \Sigma_{jk}\sigma_{f(X_k)} f(Z_j)X_j.
\end{equation*}
Now let us take again the expectation on both sides, rearrange and expand by $\sigma_{X_j}$, yielding
\begin{equation}\label{population_polyserial_nonpara}
    \Sigma_{jk} = \frac{E[f(X_k)X_j]}{\sigma_{f(X_k)} E[f(Z_j)X_j]} = \frac{r_{f(X_k)X_j}\sigma_{X_j}}{E[f(Z_j)X_j]},
\end{equation}
where $r_{f(X_k)X_j}$ is the product-moment correlation between the Gaussian (unobserved) variable $f(X_k)$ and the observed discretized variable $X_j$.

All that remains is to find sample versions of each of the three components in Eq. \eqref{population_polyserial_nonpara}. Let us start with the expectation in the denominator $E[f(Z_j)X_j]$. By assumption $f(\mathbf{Z}) \sim \text{N}(\mathbf{0},\mathbf{\Sigma})$ and therefore w.l.o.g. $f(Z_j) \sim \text{N}(0,1)$ for all $j \in 1, \dots, d_1$. Consequently, we have:
\begin{equation}
    \begin{split}
        E[f(Z_j)X_j] &= \sum_{r=1}^{l_{j+1}} x^r_j \int_{\Gamma_j^{r-1}}^{\Gamma_j^{r}} f(z_j) d F(f(z_j)) = \sum_{r=1}^{l_{j+1}} x^r_j \int_{\Gamma_j^{r-1}}^{\Gamma_j^{r}} f(z_j) \phi(f(z_j)) dz_j \\
        &= \sum_{r=1}^{l_{j+1}} x^r_j \bigg(\phi(\Gamma_j^r) - \phi(\Gamma_j^{r-1}) \bigg) = \sum_{r=1}^{l_{j}} (x^{r+1}_j - x^r_j)\phi(\Gamma_j^r),
    \end{split}
\end{equation}
where $\phi(t)$ denotes the standard normal density. Whenever the ordinal states are consecutive integers we have $\sum_{r=1}^{l_{j}} (x^{r+1}_j - x^r_j)\phi(\Gamma_j^r) = \sum_{r=1}^{l_{j}}\phi(\Gamma_j^r)$. Based on this derivation, it is straightforward to give an estimate of $E[f(Z_j)X_j]$ once estimates of the thresholds $\Gamma_j$ have been formed (see Section \ref{sec::thresholds} for more details). Let us turn to the numerator of Eq. \eqref{population_polyserial_nonpara}. The standard deviation of $X_j$ does not require any special treatment, and we simply use $\sigma^{(n)}_{X_j} = \sqrt{1/n \sum_{i=1}^n (X_{ij} - \bar{X}_{j})^2}$ to be able to treat discrete variables with a general number of states. However, the product-moment correlation $r_{f(X_k), X_j}$ is inherently more challenging as it involves the (unobserved) transformed version of the continuous variables. Therefore, we proceed to estimate the transformation.

To this end, consider the marginal distribution function of $X_k$, namely \[F_{X_k}(x)=P(X_k \leq x) = P(f(X_k) \leq f(x)) = \Phi(f(x)),\] such that $f(x) = \Phi^{-1}(F_{X_k}(x))$. In this setting, \citet{Liu09} propose to evaluate the quantile function of the standard normal at a Winsorized version of the empirical distribution function. This is necessary as the standard Gaussian quantile function $\Phi^{-1}(\cdot)$ diverges when evaluated at the boundaries of the $[0,1]$ interval. More precisely, consider $\hat{f}(u) = \Phi^{-1}(W_{\delta_n}[\hat{F}_{X_k}(u)])$,where $W_{\delta_n}$ is a Winsorization operator, i.e. \[W_{\delta_n}(u) \equiv \delta_n I(u < \delta_n) + u I(\delta_n \leq u \leq (1-\delta_n)) + (1-\delta_n) I(u > (1-\delta_n)).\] The truncation constant $\delta_n$ can be chosen in several ways.
\citet{Liu09} propose to use $\delta_n = 1/(4n^{1/4}\sqrt{\pi\log n})$ in order to control the bias-variance trade-off.
Thus, equipped with an estimator for the transformation functions, the product-moment correlation is obtained the usual way, i.e.
\begin{equation*}
    r^{(n)}_{\hat{f}(X_k),X_j} = \frac{\sum_{i=1}^n (\hat{f}(X_{ik}) - \mu(\hat{f}))(X_{ij} - \mu(X_j)}{\sqrt{\sum_{i=1}^n \Big(\hat{f}(X_{ik}) - \mu(\hat{f})\Big)^2}\sqrt{\sum_{i=1}^n \Big(X_{ij} - \mu(X_j)\Big)^2}},
\end{equation*}
where $\mu(\hat{f}) \equiv 1/n\sum_{i=1}^n \hat{f}(X_{ik})$ and $\mu(X_j) \equiv 1/n\sum_{i=1}^n X_{ij}$. The resulting estimator is a double-two-step estimator of the mixed couple $X_j$ and $X_k$.
\begin{definition}
    [Estimator $\hat{\mathbf{\Sigma}}^{(n)}$ of $\mathbf{\Sigma}$; \textit{Case II} nonparanormal]
    The estimator $\hat{\mathbf{\Sigma}}^{(n)} = (\hat{\Sigma}_{jk}^{(n)})_{1 < j \leq d_1 < k \leq d_2}$ of the correlation matrix $\mathbf{\Sigma}$ is defined by:
    \begin{equation}
        \hat{\Sigma}_{jk}^{(n)} = \frac{r^{(n)}_{\hat{f}(X_k),X_j} \sigma^{(n)}_{X_j}}{\sum_{r=1}^{l_{j}} \phi(\hat{\Gamma}_j^r)(x_j^{r+1} - x_j^r)}
    \end{equation}
    for all $1 < j \leq d_1 < k \leq d_2$.
\end{definition}

\subsection{Nonparanormal Case III}\label{sec::nonparanormal_case3}

Lastly, let us turn to \textit{Case III} where both $X_j$ and $X_k$ are discrete, but they might differ in their respective state spaces. In the previous section, the ML procedure could no longer be applied directly because we do not observe the continuous variable in its Gaussian form. In \textit{Case III}, however, we only observe the discrete variables generated by the latent scheme outlined in Definition \ref{def1}. Due to the monotonicity of the transformation functions, the ML procedure for \textit{Case III} from Section \ref{sec::latent_gaussian} can still be applied, i.e.

\begin{definition}[Nonparanormal estimator $\hat{\mathbf{\Sigma}}^{(n)}$ of $\mathbf{\Sigma}$; \textit{Case III} ]
    The estimator $\hat{\mathbf{\Sigma}}^{(n)} = (\hat{\Sigma}_{jk}^{(n)})_{1\leq j < k\leq d_1}$ of the correlation matrix $\mathbf{\Sigma}$ is defined by:
    \begin{equation}
        \hat{\Sigma}_{jk}^{(n)} = \argmax_{\Abs{\Sigma_{jk}} \leq 1} \frac{1}{n} \ell_{jk}^{(n)}(\Sigma_{jk}, x_j^r,x_k^s)
    \end{equation}
    for all $1 < j < k \leq d_1 $.
\end{definition}

In summary, the estimator $\hat{\mathbf{\Sigma}}^{(n)}$ under the latent Gaussian copula model is a simple but important tool for flexible mixed graph learning. By using ideas from polyserial and polychoric correlation measures, we not only have an easy-to-calculate estimator but also overcome the issue of finding bridge functions between all different kinds of discrete variables.

\subsection{Threshold estimation}\label{sec::thresholds}

The unknown threshold parameters $\Gamma_j$ for \(j \in [d_1]\) play a key role in linking the observed discrete to the latent continuous variables. Therefore, being able to form accurate estimates of the $\Gamma_j$ is crucial for both the likelihood-based procedures and the nonparanormal estimators outlined above.

We start by highlighting that we set the LGCM model up such that for each $\Gamma_j$, there exists a constant $G$ such that $\Abs{\Gamma^r_j} \leq G$ for all $r \in [l_{j}]$, i.e., the estimable thresholds are bounded away from infinity. Let us define the cumulative probability vector $\pi_j = (\pi^1_j, \dots, \pi^{l_{j}}_j)$. Then, by Eq. \eqref{latent_ordered}, it is easy to see that
\begin{equation}\label{thresholds_identity}
    \begin{split}
        \pi^r_j &= \sum_{i=1}^r P(X_j = x^i_j) = P(X_j \leq x^r_j) \\
        &= P(Z_j \leq \gamma_j^r) = P(f_j(Z_j) \leq f_j(\gamma_j^r)) = \Phi(\Gamma^r_j).
    \end{split}
\end{equation}
From this equation, it is immediately clear that the thresholds satisfy $\Gamma^{r}_{j} = \Phi^{-1}( \pi^r_j )$.
Consequently, when forming sample estimates of the unknown thresholds, we replace the cumulative probability vector with its sample equivalent, namely
\begin{equation}
    \hat\pi^r_j = \sum_{k=1}^r \Big[\frac{1}{n} \sum_{i=1}^n \mathbbm{1}{(X_{ij} = x^k_j)}\Big] = \frac{1}{n} \sum_{i=1}^n \mathbbm{1}{(X_{ij} \leq x^r_j)},
\end{equation}
and plug it into the identity, i.e. $\hat\Gamma^r_j = \Phi^{-1}\big( \hat\pi^r_j \big)$ for $j \in [d_1]$. The following lemma assures that these threshold estimates can be formed with high accuracy.
\begin{lemma}\label{lemma::thresholds}
    Suppose the estimated thresholds are bounded away from infinity, i.e., \(\Abs{\hat\Gamma^r_j} \leq G\) for all $j \in [d_1]$ and $r = 1, \dots, l_j$ and some \(G\). The following bound holds for all $t > 0$ with Lipschitz constant \(L_1 = 1/(\sqrt{\frac{2}{\pi}} \min\{\hat\pi^r_j, 1- \hat\pi^r_j\})\):
    \begin{equation*}
        P\Big(\Abs{\hat\Gamma_j^r - \Gamma_j^r} \geq t \Big) \leq 2\exp{\Big(- \frac{2t^2n}{L_1^2}\Big)}.
    \end{equation*}
\end{lemma}

The proof of Lemma \ref{lemma::thresholds} is given in Section 5 
of the Supplementary Materials. The requirement that the estimated thresholds are bounded away from infinity typically does not pose any restriction in finite samples. All herein-developed methods are applied in a two-step fashion. In the ensuing theoretical results, we stress this by denoting the estimated thresholds as $\bar{\Gamma}_j^r$.

\subsection{Concentration results}\label{sec::convergence_results}

Define \(\mathbf{\Sigma}^*\) and \(\mathbf{\Omega}^*\) as the true covariance matrix and its inverse, respectively. We start by stating the following assumptions:

\begin{assumption}\label{ass1}
    For all $1 \leq j < k \leq d$, $\Abs{\Sigma_{jk}^*} \neq 1$. In other words, there exists a constant $\delta > 0$ such that $\Abs{\Sigma_{jk}^*} \leq 1 - \delta$.
\end{assumption}

\begin{assumption}\label{ass2}
    For any $\Gamma_j^r$ with $j \in [d_1]$ and $r \in [l_{j}]$ there exists a constant $G$ such that $\Abs{\Gamma_j^r} \leq G$.
\end{assumption}

\begin{assumption}\label{ass3}
    Let $j < k$ and consider the log-likelihood functions in Definition \ref{definition_case2} and in Definition \ref{definition_case3}. We assume that with probability one,
    \begin{itemize}
        \item $\{-1+\delta, 1 - \delta\}$ are not critical points of the respective log-likelihood functions.
        \item The log-likelihood functions have a finite number of critical points.
        \item Every critical point that is different from $\Sigma_{jk}^*$ is non-degenerate.
        \item All joint and conditional states of the discrete variables have positive probability.
    \end{itemize}
\end{assumption}


Assumptions \ref{ass1} and \ref{ass2} ensure that $f(X_j)$ and $f(X_k)$ are not perfectly linearly dependent and that the thresholds are bounded away from infinity, respectively. Importantly, these constraints impose minimal restrictions in practice. Assumption \ref{ass3} guarantees that the likelihood functions in Section \ref{latent_gaussian_cm} exhibit a ``nice'' behavior, representing a mild technical requirement.

\paragraph{Convergence results for latent Gaussian models}
The subsequent theorem, drawing on \citet{Mei18}, hinges on four conditions, all substantiated in Section 3 of the Supplementary Materials. This concentration result specifically pertains to the MLEs introduced in Section \ref{sec::latent_gaussian} within the framework of the latent Gaussian model. We remark that related methodology has been applied by \citet{Anne19} in addressing zero-inflated Gaussian data under double truncation.

\begin{theorem}\label{uniform_convergence}
    Suppose that Assumptions \ref{ass1}--\ref{ass3} hold, and let $j \in [d_1]$ and $k \in [d_2]$ for  \textit{Case II} and $j,k \in [d_1]$ for \textit{Case III}. Let $\alpha \in (0,1)$, and let \(n \geq 4 C \log(n) \log\Big(\frac{B}{\alpha}\Big)\) with some known constants $B$, $C$, and $D$ depending on cases II and III but independent of $(n,d)$. Then, it holds that
    \begin{equation}
        P\Bigg(\max_{j,k}\abs{\hat{\Sigma}_{jk}^{(n)} - \Sigma_{jk}^{*}} \geq D\sqrt{\frac{\log(n)}{n} \log\bigg(\frac{B}{\alpha}\bigg)}\Bigg) \leq \frac{d(d-1)}{2}\alpha.
    \end{equation}
\end{theorem}
\textit{Case I} of the latent Gaussian model addresses the well-understood scenario involving observed Gaussian variables, with concentration results and rates of convergence readily available -see, for example, Lemma 1 in \citet{Ravikumar11}. Consequently, the MLEs converge to \(\mathbf{\Sigma}^{*}\) at the optimal rate of \(n^{-1/2}\), mirroring the convergence rate as if the underlying latent variables were directly observed.

\paragraph{Convergence of nonparanormal estimators.} Recall the three cases, which, in principle, will have to be considered again.

\begin{description}[labelwidth=4em, leftmargin =\dimexpr\labelwidth+\labelsep\relax, font=\mdseries]
    \item[\textit{Case I}:] When both random variables are continuous, concentration results follow immediately from \citet{Liu12} who make use of Hoeffding's inequalities for $U$-statistics.
    \item[\textit{Case II}:] For the case where one variable is discrete and the other one continuous, we present concentration results below.
    \item[\textit{Case III}:] When both variables are discrete, we make an important observation that Theorem \ref{uniform_convergence} above still applies and needs not to be altered. We do not observe the continuous variables directly but only their discretized versions. Consequently, the threshold estimates remain valid under the monotone transformation functions, and so does the polychoric correlation.
\end{description}

\noindent The following theorem provides concentration properties for \textit{Case II} under the LGCM.
\begin{theorem}\label{concentration_caseII}
    Suppose that Assumptions \ref{ass1} and \ref{ass2} hold and $j \in [d_1]$ and $k \in [d_2]$. Then for any $\epsilon \in \Big[C_M\sqrt{\frac{\log d \log^2 n}{\sqrt{n}}},8(1+4c^2)\Big]$, with sub-Gaussian parameter $c$, generic constants $k_i, i = 1,2,3$ and constant $C_M = \frac{48}{\sqrt{\pi}} \big(\sqrt{2M} - 1\big)(M+2)$ for some $M \geq 2\big(\frac{\log d_2}{\log n} +1\big)$ with $C_\Gamma = \sum_{r=1}^{l_j} \phi(\bar{\Gamma}_j^r)(x_j^{r+1} - x_j^r)$ and Lipschitz constant $L$ the following probability bound holds
    \begin{multline*}
        P\left(\max_{jk}\abs{\hat{\Sigma}_{jk}^{(n)} -  \Sigma_{jk}^{*}} \geq \epsilon \right) \\
        \begin{aligned}
             & \leq 8\exp\Bigg(2\log d - \frac{\sqrt{n}\epsilon^2}{(64 \ L \ C_\gamma \ l_{\max} \ \pi)^2 \log n}\Bigg)                                                 \\
             & + 8\exp\left( 2\log d - \frac{n \epsilon^2}{(4L \ C_\gamma)^2 \ 128(1+4c^2)^2} \right)                                                                   \\
             & + 8\exp\Big(2\log d - \frac{\sqrt{n}}{8\pi\log n}\Big) + 4\exp\Big(- \frac{k_1 n^{3/4} \sqrt{\log n}}{k_2+ k_3} \Big) + \frac{2}{\sqrt{\pi \log(nd_2)}}.
        \end{aligned}
    \end{multline*}

\end{theorem}
The proof of the theorem is given in Section 6
of the Supplementary Materials. The first four terms in the probability bound stem from finding bounds to different regions of the support of the transformed continuous variable. The last term is a consequence of the fact that we estimate the transform directly.

Regarding the scaling of the dimension in terms of sample size, the ensuing corollary follows immediately.
\begin{corollary}
    For some known constant $K_{\Sigma}$ independent of $d$ and $n$ we have
    \begin{equation}
        P\Bigg(\max_{j,k}\abs{\hat{\Sigma}_{jk}^{(n)} - \Sigma_{jk}^*} > K_{\Sigma}\sqrt{\frac{\log d \log n}{\sqrt{n}}} \Bigg) = o(1).
    \end{equation}
\end{corollary}

The nonparanormal estimator for \textit{Case II} converges to \(\Sigma_{jk}^*\) at rate \(n^{-1/4}\), which is slower than the optimal parametric rate of \(n^{-1/2}\). This stems not from the presence of the discrete variable but from the direct estimation of the transformation function \({f}_j\) and the corresponding truncation constant \(\delta_n\). Both \citet{Xue12} and \citet{Liu12} discuss room for improvement of the estimator for \(f_j\) to get a rate closer to the optimal one. Improvements depend strongly on the choice of truncation constant, striking a bias-variance balance in high-dimensional settings. In all of the numerical experiments below, we find that Theorem \ref{concentration_caseII} gives a worst-case rate that does not appear to negatively impact performance compared to estimators that attain the optimal rate.

\subsection{Estimating the precision matrix}\label{sec:precision_matrix}

Similar to \citet{Fan17}, we plug our estimate of the sample correlation matrix into existing routines for estimating $\mathbf{\Omega}^*$. In particular, we employ the graphical lasso (glasso) estimator \citep{Friedman08}, i.e.
\begin{equation}\label{glasso}
    \hat{\mathbf\Omega} = \argmin_{\mathbf\Omega \succeq 0} \big[\text{tr}(\hat{\mathbf\Sigma}^{(n)}\mathbf\Omega) - \log\Abs{\mathbf\Omega} + \lambda \sum_{j\neq k}\Abs{\Omega_{jk}}\big],
\end{equation}
where $\lambda > 0$ is a regularization parameter. As $\hat{\mathbf\Sigma}^{(n)}$ exhibits at worst the same theoretical properties as established in 
\citet{Liu09}, convergence rate and graph selection results follow immediately.

We do not penalize diagonal entries of $\mathbf\Omega$ and therefore have to make sure that $\hat{\mathbf\Sigma}^{(n)}$ is at least positive semidefinite to establish convergence in Eq. \eqref{glasso}. Hence, we need to project $\hat{\mathbf\Sigma}^{(n)}$ into the cone of positive semidefinite matrices; see also \citep{Liu12, Fan17}. In practice, we use an efficient implementation of the alternating projections method proposed by \citet{Higham88}.

To select the tuning parameter in Eq. \eqref{glasso}
\citet{Foygel10} introduce an extended BIC (eBIC) in particular for Gaussian graphical models establishing consistency in higher dimensions under mild asymptotic assumptions. We consider
\begin{equation}\label{EBIC}
    eBIC_\theta = -2 \ell^{(n)}(\hat{\mathbf\Omega}(E)) + \Abs{E} \log(n) + 4 \Abs{E}\theta \log(d),
\end{equation}
where $\theta \in [0,1]$ governs penalization of large graphs. Furthermore, $\abs{E}$ represents the cardinality of the edge set of a candidate graph on $d$ nodes and $\ell^{(n)}(\hat{\mathbf\Omega}(E))$ denotes the corresponding maximized log-likelihood which in turn depends on $\lambda$ from Eq. \eqref{glasso}. In practice, first, one retrieves a small set of models over a range of penalty parameters $\lambda > 0$ (called \textit{glasso path}). Then, we calculate the eBIC for each model in the path and select the one with the minimal value.

\section{Numerical results}\label{sec::numerical_results}


To numerically assess the accuracy of our mixed graph estimation approach, we commence with a simulation study in which the estimators are rigorously evaluated in a gold-standard fashion and compared against oracles.

\subsection{Simulation setup}\label{sec::setup}




We start by constructing the underlying precision matrix $\mathbf{\Omega}^*$ whose zero pattern encodes the undirected graph. We set $\mathbf{\Omega}^*_{jj} = 1$ and $\mathbf{\Omega}_{jk}^* = s \cdot b_{jk}$ if $j\neq k$, where $s$ is a constant signal strength chosen to assure positive definiteness. Furthermore, $b_{jk}$ are realizations of a Bernoulli random variable with corresponding success probability $p_{jk} = (2\pi)^{-1/2} \exp\big[\Norm{v_j - v_k}_2 / (2c) \big]$. In particular, $v_j = (v_j^{(1)}, v_j^{(2)})$ are independent realizations of a bivariate uniform $[0,1]$ distribution and $c$ controls the sparsity of the graph.

Throughout the simulation, we set $s = 0.15$ and incrementally increase the dimensionality s.t. $d \in \{50,250,750\}$, representing a transition from small to large-scale graphs. We let $\mathbf{\Sigma}^* = (\mathbf{\Omega}^*)^{-1}$ be rescaled such that all diagonal elements are equal to $1$. Given $\mathbf\Sigma^*$, we first obtain the partially latent continuous data \(\mathbf{Z} = (\mathbf{Z}_1, \mathbf{X}_2)\) where $\mathbf{Z} \sim \text{NPN}_d(\mathbf{0}, \mathbf{\Sigma}^*, f)$. In practice, we draw \(n\) i.i.d. samples from \(\text{N}_d(\mathbf{0}, \mathbf{\Sigma}^*)\) and apply the back-transform \(f^{-1}\) to each individual variable.

To generate general mixed data \(\mathbf{X} = (\mathbf{X}_1,\mathbf{X}_2)\) according to the LGCM we need to appropriately threshold \(\mathbf{Z}_1\). Let $\mathbf{X}_1$ be partitioned into equally sized collections of binary, ordinal, and Poisson distributed random variables, i.e., $\mathbf{X}_1 = (\mathbf{X}_1^{\text{bin}}, \mathbf{X}_1^{\text{ord}},\mathbf{X}_1^{\text{pois}})$. We use the inverse probability integral transform (IPT) to generate random samples from the respective cumulative distribution functions, corresponding to the relationship described in Eq. \eqref{latent_ordered}. For $\mathbf{X}_1^{\text{bin}}$ IPT is employed with success probability drawn from Uniform$[0.4,0.6]$ for $80\%$ of $\mathbf{X}_1^{\text{bin}}$. We assign unbalanced classes to the remaining $20\%$, where the success probability is drawn from Uniform$[0.05,0.1]$. Regarding $\mathbf{X}_1^{\text{ord}}$, IPT is used to generate samples from the multinomial distribution. To that end, we draw the number of categories from Uniform$[3,7]$ and round it to the nearest integer. We set the probability of falling into one of these categories to be proportional to their number. Lastly, $\mathbf{X}_1^{\text{pois}}$ is generated using IPT with the rate parameter set to $6$. In case we only need a mix of binary and continuous data, we set \(\mathbf{X}_1 = \mathbf{X}_1^{\text{bin}}\).

Throughout the experiments, $\hat{\Omega}$ is chosen by minimizing the eBIC according to the procedure outlined in Section \ref{sec:precision_matrix} with $\theta = 0.1$ for the low and medium, $\theta = 0.5$ for the high dimensional graphs.
The sample size $n$ is set to \(200\) for \(d\in \{50, 250\}\) and \(300\) for \(d = 750\). We set the number of simulation runs to \(100\). Lastly, we choose the sparsity parameter $c$ such that the number of edges aligns roughly with the dimension -- except for $d = 50$, where we allow for $200$ edges following \citet{Fan17}.

\paragraph{Performance metrics.}
To evaluate performance, we report the mean estimation error $\Norm{\hat{\mathbf{\Omega}} - \mathbf\Omega^*}_F$ using the Frobenius norm. Additionally, we employ graph recovery metrics. For this purpose, we calculate the number of true positives $\text{TP}(\lambda)$ and false positives $\text{FP}(\lambda)$ based on the \textit{glasso path}. $\text{TP}(\lambda)$ represents the count of non-zero lower off-diagonal elements that are consistent both in $\mathbf\Omega^*$ and $\hat{\mathbf\Omega}$, while $\text{FP}(\lambda)$ denotes the count of non-zero lower off-diagonal elements in $\hat{\mathbf\Omega}$ that are zero in $\mathbf\Omega^*$.

The true positive rate $\text{TPR}(\lambda)$ and false positive rate $\text{FPR}(\lambda)$ are defined as $\text{TPR} = \frac{\text{TP}(\lambda)}{\abs{E}}$ and $\text{FPR} = \frac{\text{FP}(\lambda)}{d(d-1)/2 - \abs{E}}$, respectively. Finally, we consider the area under the curve (AUC), where a value of $0.5$ corresponds to random guessing of edge presence and a value of $1$ indicates perfect error-free recovery of the underlying latent graph (in the rank sense of ROC analysis).


\subsection{Simulation results}

\paragraph{Binary-continuous data}

We start by considering a mix of binary and continuous variables generated as outlined in Section \ref{sec::setup} to compare our methods against the bridge function approach of \citet{Fan17}. For this purpose, \figref{fig:bench_binary} depicts the mean estimation error $\Norm{\hat{\mathbf\Omega} - \mathbf\Omega^*}_F$ and the AUC for the different estimators under the different $(d,n)$ regimes. We include the following estimators for \(\mathbf\Omega^*\): (1) An oracle estimator (\texttt{oracle}) that corresponds to estimating $\hat{\mathbf\Sigma}^{(n)}$ using the mapping between Spearman's rho and \(\hat\Sigma_{jk}^*\) (Eq. \eqref{spearman_mapping}) based on realization of the (partially) latent continuous data $(\mathbf{Z_1},\mathbf{X_2})$. (2) The bridge function based estimator (\texttt{bridge}) proposed by \citet{Fan17}. (3) The polychhoric and polyserial MLE estimator (\texttt{mle}) proposed in Section \ref{sec::latent_gaussian}. (4) The general mixed estimator (\texttt{poly}) proposed in Section \ref{sec::nonparanormal}. In the left column, we set $f_j(x) = x$ for all $j$, i.e., we recover the latent Gaussian model. In the right column, we set $f_j(x) = x^{1/3}$ for all $j$ to recover the LGCM.

\begin{figure}
    \centering
    \includegraphics[width=\textwidth]{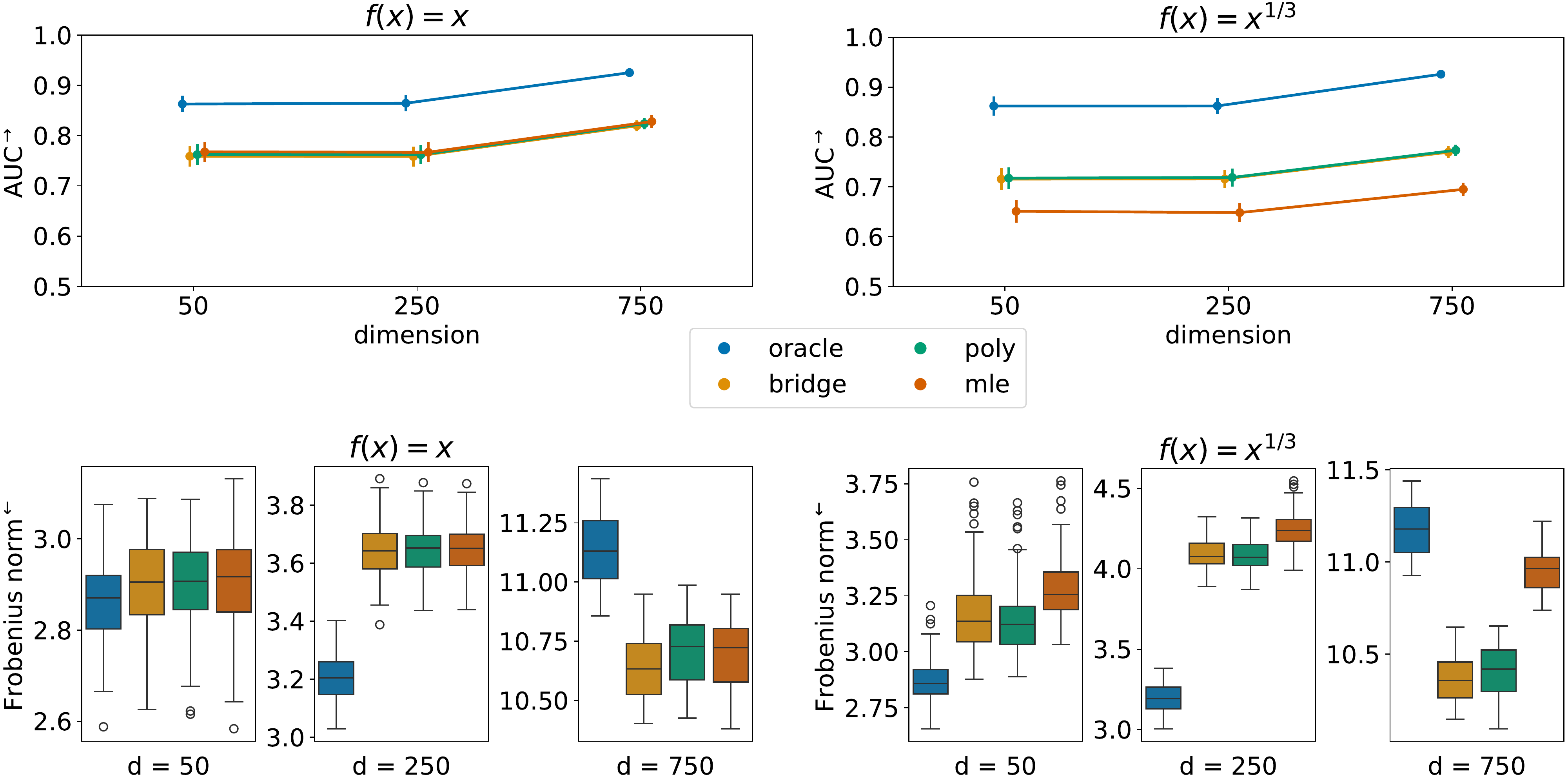}
    \caption{Simulation results for the binary-continuous data setting based on \(100\) simulation runs. The left column corresponds to the latent Gaussian model, where the transformation function is the identity. The right colum depicts results for the LGCM with \(f_j(x) = x^{1/3}\) for all \(j\). The top row reports mean and standard deviation of the AUC along simulation runs, and the bottom row depicts boxplots of the estimation error $\Norm{\hat{\mathbf\Omega} - \mathbf\Omega^*}_F$. The y-axis labels have superscript arrows attached to indicate the direction of improvement: \(\rightarrow\) implies that larger values are better, and \(\leftarrow\) implies that smaller values are better.}
    \label{fig:bench_binary}
\end{figure}

\figref{fig:bench_binary} suggests that under the latent Gaussian model (left column), there are virtually no differences between all non-oracle estimators in graph recovery or estimation error. As expected, the \texttt{oracle} has the highest AUC and lowest estimation error across scenarios. The only exception is the estimation error when the dimension is \(d = 750\). This surprising result stems from an increased FPR for \texttt{oracle} (see Figure 3 in the Supplementary Materials) when minimizing the eBIC with the additional penalty set to $\theta = 0.5$. A higher penalty seems appropriate in this case. Meanwhile, the non-oracle estimators are more conservative, and the additional penalty appears to be chosen correctly in these cases.

In the right column of \figref{fig:bench_binary}, binary-continuous mixed data is generated from the LGCM. The \textit{Case I} and \textit{Case II} MLEs are misspecified in this case, which translates to lower AUC and higher estimation error.
The remaining estimators are unaffected by the transformation. Encouragingly, we find no substantial performance difference between the bridge function approach and our procedure in any of the metrics considered, including the TPR and FPR results in Figure 3 in the Supplementary Materials.

\paragraph{General mixed data}
Let us turn to the general mixed setting. While the bridge function approach by \citet{Fan17} does not extend beyond the binary-continuous mix, we can still compare our approach to the ensemble method developed by \citet{Feng19}. Due to its close connection to the original method, we continue to denote the proposed ensemble estimator \texttt{bridge}.
\begin{figure}
    \centering
    \includegraphics[width=\textwidth]{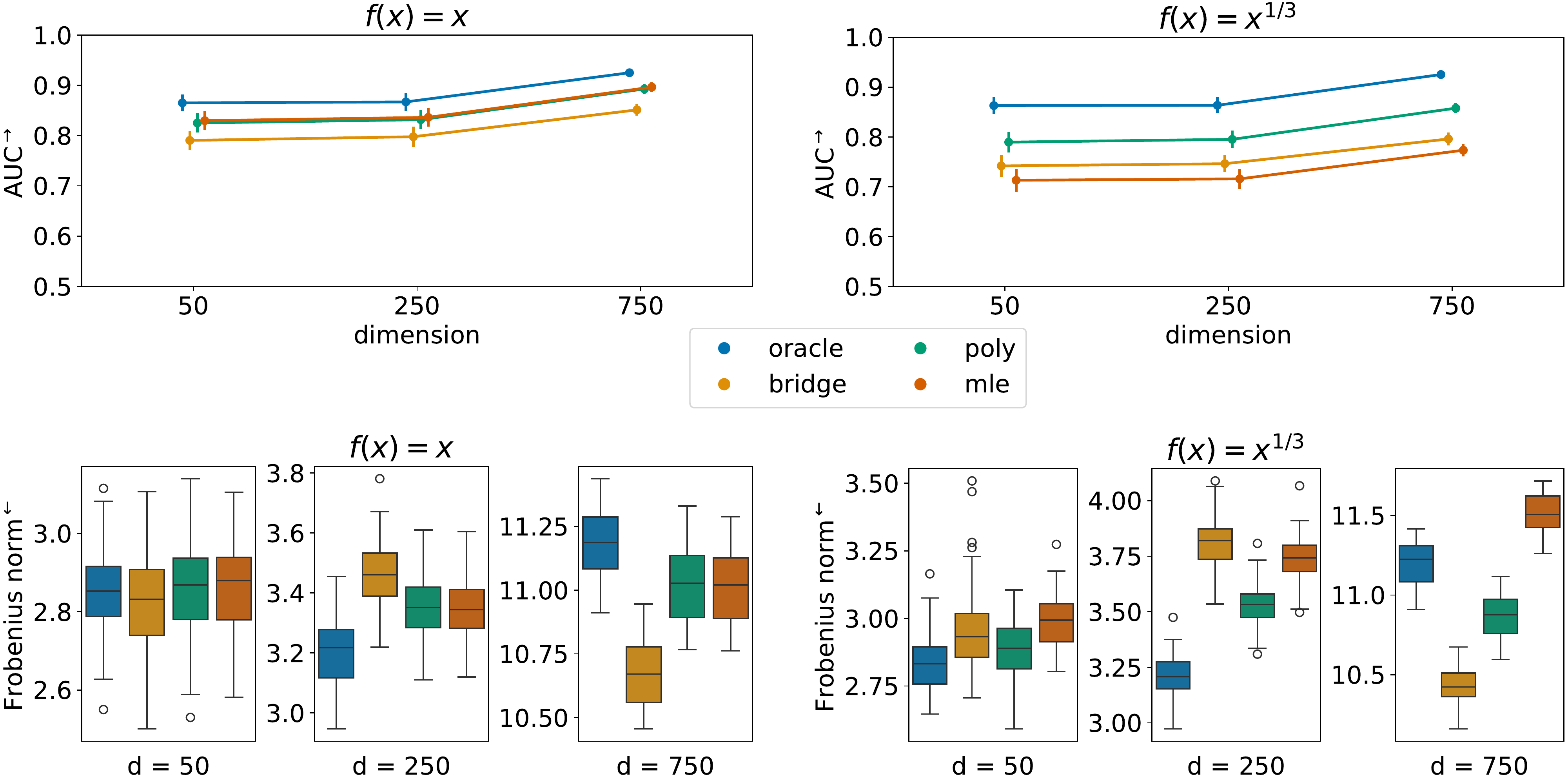}
    \caption{Simulation results for the general mixed data setting based on \(100\) simulation runs. The left column corresponds to the latent Gaussian model, where the transformation function is the identity. The right colum depicts results for the LGCM with \(f_j(x) = x^{1/3}\) for all \(j\). The top row reports mean and standard deviation of the AUC along simulation runs, and the bottom row depicts boxplots of the estimation error $\Norm{\hat{\mathbf\Omega} - \mathbf\Omega^*}_F$. The y-axis labels have superscript arrows attached to indicate the direction of improvement: \(\rightarrow\) implies that larger values are better, and \(\leftarrow\) implies that smaller values are better.}
    \label{fig:bench_genral}
\end{figure}

Similar to above, \figref{fig:bench_genral} depicts the mean estimation error $\Norm{\hat{\mathbf\Omega} - \mathbf\Omega^*}_F$ and the AUC and left and right columns correspond to latent Gaussian und LGCM settings, respectively. This time, given general mixed data, differences in terms of AUC between the estimators are noticeable. When the transformations are the identity, the MLE is correctly specified, and it is tied with \texttt{poly} in terms of AUC. The \texttt{bridge} estimator performs worse than the other two estimators. When the transformations are $f_j(x) = x^{1/3}$, the MLE is misspecified, and our \texttt{poly} estimator performs best among non-oracle estimators in terms of AUC. The \texttt{bridge} estimator performs only marginally better than the misspecified \texttt{mle}.

Turning to estimation error results, when \(f_j(x) = x\), the \texttt{poly} and \texttt{mle} estimators perform similarly across dimensions. As the \texttt{oracle} estimator is formed on the latent continuous data, it is unaffected by any discretization and behaves the same as in the binary-continuous case above. The \texttt{bridge} ensemble estimator accounts for a slightly higher estimation error when \(d = 250\) and a lower one when \(d=750\). This pattern can be explained by the FPR of the \texttt{bridge} estimator as illustrated in Figure 4 in the Supplementary Materials. While the FPR is slightly higher in the \(d=250\) case, it is lower in the \(d=750\) case. Similar to the \texttt{oracle} results, this appears to be a consequence of the additional penalty term in the eBIC. Considering the estimation error when \(f_j(x) = x^{1/3}\), the \texttt{poly} and \texttt{bridge} estimators retain their performance. As before, the \texttt{mle} estimator is misspecified and performs worse than the other two estimators.

Overall, the simulation results suggest that our proposed \texttt{poly} estimator performs similarly (binary-continuous data setting) or better (general mixed setting) than the current state-of-the-art. In particular, the \texttt{poly} estimator achieves good performance scores regarding recovery of the graph structure in the general mixed setting. Estimation error results are more sensitive to the choice of the additional high-dimensional penalty. These empirical results suggest that despite the theoretically slower convergence rate, the \texttt{poly} estimator is competitive regarding graph recovery and estimation error.



\section{Conclusion}\label{sec::conclusions}

Estimating high-dimensional undirected graphs from general mixed data is a challenging task. We propose an innovative approach that blends classical generalized correlation measures, specifically polychoric and polyserial correlations, with recent concepts from high-dimensional graphical modeling and copulas.

A pivotal insight guiding our approach is recognizing that polychoric and polyserial correlations can be effectively modeled through a latent Gaussian copula. Although adapting polyserial correlation to the nonparanormal case demands careful consideration, the polychoric correlation requires no adjustments. The resulting estimators exhibit favorable theoretical properties, even in high dimensions, and demonstrate robust empirical performance in our simulation study.

Our advocated framework builds on prior work extending the graphical lasso for Gaussian observations to nonparanormal models and subsequently to mixed data, as seen in the contributions of \citet{Fan17}, followed by \citet{Quan18} and \citet{Feng19}. A key distinction in our approach is the absence of the need to specify bridge functions. Indeed, our method seamlessly handles various types of mixed data without requiring additional user effort. 

\section{Software}\label{sec5}

Software in the form of the \texttt{R} package \pkg{hume} is available on the corresponding author's GitHub page (\url{https://github.com/konstantingoe/hume}). The \texttt{R}-code to reproduce the simulation study conducted in the paper is available under \url{https://github.com/konstantingoe/mixed_hidim_graphs}.

\section*{Supplementary Materials}\label{sec::sm}

We refer the reader to the Supplementary Materials for technical appendices, including proofs of the theorems and lemmas in the main manuscript. Additionally, we present further simulation results and an analysis of real phenotyping data concerning severe COVID-19 outcomes from the UK Biobank.

\section*{Funding}

This work was supported by the Helmholtz AI project
``Scalable and Interpretable Models for Complex And Structured Data'' (SIMCARD),
the Medical Research Council [programme number
        MC UU 00002/17] and the National Institute for Health Research (NIHR) Cambridge Biomedical Research Centre.
\\

\noindent This project also received funding from the European Research Council (ERC) under the European Union’s Horizon 2020 research and innovation programme [grant agreement No 883818].



\begin{acks}[Acknowledgments]
    We thank the anonymous reviewers whose insightful comments and constructive feedback significantly enhanced the quality and clarity of this paper. We further want to thank Hongjian Shi for his helpful insights on the subject.\\

    \paragraph{Conflict of Interest:} None declared.
\end{acks}

\clearpage
\newpage
\renewcommand{\thesection}{SM} 
\section*{Supplementary Materials}

\section{Methodology}

In the following sections, we derive the MLEs for the latent Gaussian model. Then, we provide the proof for Theorem 3.2, followed by the proof of Theorem 3.3.

\subsection{\textit{Case II} MLE derivation}

In \textit{Case II}, we consider the instance where $X_j$ is ordinal and $X_k$ is continuous. Recall that in the latent Gaussian model, we take all \(\{f_k\}_{k=1}^d\) to be the identity. Consequently, \(X_k\) is Gaussian and \(\Gamma_j^r = f_j(\gamma_j^r) = \gamma_j^r\) for all \(j \in [d_1]\) and \(r \in [l_j]\). Recall that we are interested in the product-moment correlation $\Sigma_{jk}$ between two jointly Gaussian variables, where $X_j$ is not directly observed, but only the ordered categories (Eq. (2) in the Manuscript) are given. The likelihood of the $n$-sample is defined by:
\begin{equation}\label{polyserial_likelihood_appendix}
    \begin{split}
        L_{jk}^{(n)}(\Sigma_{jk}, x_j^r,x_k) &= \prod_{i=1}^n p(x_{ij}^{r},x_{ik}, \Sigma_{jk}) \\
        &= \prod_{i=1}^n p(x_{ik})p(x_{ij}^{r} \mid x_{ik}, \Sigma_{jk}),
    \end{split}
\end{equation}
where $p(x_{ij}^{r},x_{ik}, \Sigma_{jk})$ denotes the joint probability of  $X_j$ and $X_k$ and $p(x_{ik})$ the marginal density of the Gaussian variable $X_k$, i.e.
\begin{equation*}\label{marginal_normal}
    p(x_{ik}) = \big(2\pi\sigma\big)^{-\frac{1}{2}} \exp\Bigg[-\frac{1}{2}\bigg(\frac{x_{ik} - \mu}{\sigma_{ik}}\bigg)^2\Bigg].
\end{equation*}
Furthermore, the conditional probability of $X_j$ in Eq. 
(5) in the Manuscript can be written as:
\begin{equation}\label{threshold_conditionalprob}
    \begin{split}
        p(X_j = x_j^r \mid X_k, \Sigma_{jk}) = p(\Gamma_j^{r-1} \leq Z_j < \Gamma_j^r \mid X_k, \Sigma_{jk}) \\
        = p(Z_j \leq \Gamma_j^{r} \mid X_k, \Sigma_{jk}) - p(Z_j \leq \Gamma_j^{r-1} \mid X_k, \Sigma_{jk}) \\
        \Phi(\tilde{\Gamma}_j^{r}) - \Phi(\tilde{\Gamma}_j^{r-1}), \quad r = 1, \dots, l_{j},
    \end{split}
\end{equation}
where
\begin{equation*}
    \tilde{\Gamma}_j^{r} = \frac{\Gamma_j^{r} - \Sigma_{jk}\tilde{X_k} }{\sqrt{1-(\Sigma_{jk})^2}},
\end{equation*}
with $\tilde{X_k} = \frac{X_k-\mu_k}{\sigma_k}$ and $\Phi(t)$ denoting the standard normal distribution function. This follows straight from the fact that the conditional distribution of $Z_j$ is Gaussian with mean $\Sigma_{jk}\tilde{X_k}$ and variance $(1-(\Sigma_{jk})^2)$. The log-likelihood is then $\ell_{jk}^{(n)}(\Sigma_{jk}, x_j^r,x_k)$ with
\begin{equation}\label{polyserial_log-likelihood}
    \ell_{jk}^{(n)}(\Sigma_{jk}, x_j^r,x_k) = \sum_{i=1}^n \big[\log(p(x_{ik})) + \log(p(x_{j}^{r} \mid x_{ik}, \Sigma_{jk}))\big].
\end{equation}

Due to the heavy computational burden involved when estimating all parameters simultaneously, a \textit{two-step estimator} has been proposed \cite{Olsson82}. That is, in a first step $\mu_k, \sigma_k^2$ are estimated by $\bar{X_k}$ and $s_k^2$, respectively. Moreover, the thresholds $\Gamma_j^r, r = 1, \dots, l_{j}$ are estimated by the quantile function of the standard normal distribution evaluated at the cumulative marginal proportions of $x_j^r$ just as described in Section 3.4. 

In a second step, all that remains is obtaining the MLE for $\Sigma_{jk}$ now with the readily computed estimates from the first step:
\begin{equation}\label{FOC}
    \frac{\partial \ell_{jk}^{(n)}(\Sigma_{jk}, x_j^r,x_k)}{\partial \Sigma_{jk}} = \sum_{i=1}^n \frac{1}{p(x_{ij}^{r} \mid x_{ik}, \Sigma_{jk})} \frac{\partial p(x_{ij}^{r} \mid x_{ik}, \Sigma_{jk})}{\partial \Sigma_{jk}}.
\end{equation}
Let us take a closer look at the partial derivative of the conditional probability in Eq. \eqref{FOC}:
\begin{equation}\label{case2_firstdiff}
    \begin{split}
        \lefteqn{\frac{\partial p(x_{ij}^{r} \mid x_{ik}, \Sigma_{jk})}{\partial \Sigma_{jk}}} \\
        &= \frac{\partial \Phi({\tilde{\Gamma}}_j^{r})}{\partial \Sigma_{jk}} - \frac{\partial\Phi({\tilde{\Gamma}}_j^{r-1})}{\partial \Sigma_{jk}} \\
        &= \phi({\tilde{\Gamma}}_j^{r})\frac{\partial {\tilde{\Gamma}}_j^{r}}{\partial\Sigma_{jk}} - \phi({\tilde{\Gamma}}_j^{r-1})\frac{\partial {\tilde{\Gamma}}_j^{r}}{\partial\Sigma_{jk}} \\
        &= (1-(\Sigma_{jk})^2)^{-\frac{3}{2}}\Big[\phi({\tilde{\Gamma}}_j^{r})({\Gamma}_j^r\Sigma_{jk} - {\tilde{x}}_{ik}) - \phi({\tilde{\Gamma}}_j^{r-1})({\Gamma}_j^{r-1}\Sigma_{jk} - {\tilde{x}}_{ik})\Big],
    \end{split}
\end{equation}
where
\[{\tilde{X}}_{k} = \frac{X_{k} - \bar{X}_k}{\sqrt{s_k^2}} \quad \text{and} \quad {\tilde{\Gamma}}_j^{r} = \frac{{\Gamma}_j^{r} - \Sigma_{jk}{\tilde{X}}_k}{\sqrt{1-(\Sigma_{jk})^2}}.\]
The last equality in Eq. \eqref{case2_firstdiff} follows from taking the derivative and applying the chain rule. Putting all the pieces together, we obtain
\begin{equation}\label{MLE_polyserial}
    \begin{split}
        \frac{\partial\ell_{jk}^{(n)}(\Sigma_{jk}, x_j^r,x_k)}{\partial \Sigma_{jk}} = \sum_{i=1}^n &\Bigg[\frac{1}{p(x_{ij}^{r} \mid x_{ik}, \Sigma_{jk})} (1-(\Sigma_{jk})^2)^{-\frac{3}{2}} \\
        &\Big[\phi({\tilde{\Gamma}}_j^{r})({\Gamma}_j^r\Sigma_{jk} - {\tilde{x}}_{ik}) - \phi({\tilde{\Gamma}}_j^{r-1})({\Gamma}_j^{r-1}\Sigma_{jk} - {\tilde{x}}_{ik})\Big]\Bigg].
    \end{split}
\end{equation}

Setting Eq. \eqref{MLE_polyserial} to zero and solving for $\Sigma_{jk}$ yields the \textit{Case II} two-step MLE for $\Sigma_{jk}$. Note that this is a nonlinear optimization problem, which can efficiently be solved utilizing some quasi-Newton method.

\subsection{\textit{Case III} MLE derivation}

Turning to \textit{Case III}, where both $X_j$ and $X_k$ are ordinal variables, we argue in Section 3.3 in the manuscript that the MLE for the polychoric correlation remains valid for the LGCM. The probability of an observation taking values $X_j = x^r_j$ and $X_k = x^s_k$ is
\begin{equation}\label{cell_probabilities_appendix}
    \begin{split}
        \pi_{rs} &\coloneqq p(X_j = x^r_j, X_k = x^s_k) \\
        &= p(\Gamma_j^{r-1} \leq Z_j < \Gamma_j^r, \Gamma_k^{s-1} \leq Z_k < \Gamma_k^s) \\
        &= p(\Gamma_j^{r-1} \leq f_j(Z_j) < \Gamma_j^r, \Gamma_k^{s-1} \leq f_k(Z_k) < \Gamma_k^s) \\
        &= \int_{\Gamma_j^{r-1}}^{\Gamma_j^{r}} \int_{\Gamma_k^{s-1}}^{\Gamma^k_{s}} \phi_2(z_j,z_k,\Sigma_{jk}) dz_j dz_k,
    \end{split}
\end{equation}
where $r \in [l_j]$ and $s \in [l_k]$ and $\phi_2(x,y,\rho)$ denotes the standard bivariate density with correlation $\rho$. Then, as in the manuscript, the likelihood and log-likelihood of the $n$-sample are defined as:
\begin{equation}\label{polychoric_likelihood_appendix}
    \begin{split}
        L_{jk}^{(n)}(\Sigma_{jk}, x_j^r,x_k^s) &= C \prod_{r=1}^{l_{{j}}} \prod_{s=1}^{l_{{k}}} \pi_{rs}^{n_{rs}}, \\
        \ell_{jk}^{(n)}(\Sigma_{jk}, x_j^r,x_k^s) &= \log(C) + \sum_{r=1}^{l_{{j}}}\sum_{s=1}^{l_{{k}}} n_{rs} \log(\pi_{rs}).
    \end{split}
\end{equation}
where $C$ is a constant and $n_{rs}$ denotes the observed frequency of $X_j = x^r_j$ and $X_k = x^s_k$ in a sample of size $n= \sum_{r=1}^{l_{{j}}}\sum_{s=1}^{l_{{k}}} n_{rs}$.
Similar to \textit{Case II} above, we employ the \textit{two-step estimator} for the polychoric correlation. Given the threshold estimates from the first step, let us state the derivative of $\ell_{jk}^{(n)}(\Sigma_{jk}, x_j^r,x_k^s)$ with respect to $\Sigma_{jk}$ explicitly. First, recall that from Eq. \eqref{cell_probabilities_appendix}
\begin{equation}
    \begin{split}
        \pi_{rs} &= \int_{{\Gamma}_j^{r-1}}^{{\Gamma}_j^{r}} \int_{{\Gamma}_k^{s-1}}^{{\Gamma}^k_{s}} \phi_2(z_j,z_k,\Sigma_{jk}) dz_j dz_k \\
        &= \Phi_2({\Gamma}_j^r, {\Gamma}_k^s, \Sigma_{jk}) - \Phi_2({\Gamma}_j^{r-1}, {\Gamma}_k^s, \Sigma_{jk}) \\
        &- \Phi_2({\Gamma}_j^r, {\Gamma}_k^{s-1}, \Sigma_{jk}) + \Phi_2({\Gamma}_j^{r-1}, {\Gamma}_k^{s-1}, \Sigma_{jk}),
    \end{split}
\end{equation}
where $\Phi_2(u,v,\rho)$ is the standard bivariate normal distribution function with correlation parameter $\rho$. Note also that we have $\frac{\partial \Phi_2(u,v, \rho)}{\partial \rho} = \phi_2(u,v, \rho)$; see \citep{Tallis62}. Taking the derivative of $\ell_{jk}^{(n)}(\Sigma_{jk}, x_j^r,x_k^s)$ with respect to $\Sigma_{jk}$ yields
\begin{multline*}
    \frac{\partial \ell^{(n)}(\Sigma_{jk}, x_j^r,x_k^s)}{\partial \Sigma_{jk}} = \sum_{r=1}^{l_{X_{j}}}\sum_{s=1}^{l_{X_{k}}} \frac{n_{rs}}{\pi_{rs}} \frac{\partial \pi_{rs}}{\partial \Sigma_{jk}} \\
    \begin{aligned}
        = \sum_{r=1}^{l_{X_{j}}}\sum_{s=1}^{l_{X_{k}}} \frac{n_{rs}}{\pi_{rs}} \Big[ & \phi_2({\Gamma}_j^r, {\Gamma}_k^s, \Sigma_{jk}) - \phi_2({\Gamma}_j^{r-1}, {\Gamma}_k^s, \Sigma_{jk}) -             \\
                                                                                     & \phi_2({\Gamma}_j^r, {\Gamma}_k^{s-1}, \Sigma_{jk}) + \phi_2({\Gamma}_j^{r-1}, {\Gamma}_k^{s-1}, \Sigma_{jk})\Big].
    \end{aligned}
\end{multline*}

Again, setting the derivative to zero and solving for $\Sigma_{jk}$ using some quasi-Newton method yields the \textit{Case III} two-step MLE for $\Sigma_{jk}$.

\section{Comparison of \textit{Case II} estimators under the LGCM}\label{sec::case2_comparison}

In this section, we conduct an empirical comparison of \textit{Case II} estimators within the framework of the LGCM. Specifically, we examine the \textit{Case II} MLE derived under the latent Gaussian model, as discussed in Section 2.1 of the Manuscript, which involves incorporating the estimated transformations \(\hat{f}\) at appropriate locations. Furthermore, we investigate the ad hoc estimator presented in Section 3.2 of the Manuscript in more detail.

We start by rewriting Eq. \eqref{MLE_polyserial}, where we replace occurrences of the $X_k$ with the corresponding transformation $\hat{f}_k(X_k)$. The resulting transformation-based first-order condition (FOC) for the \textit{Case II} MLE under the LGCM becomes:

\begin{align}\label{eq:case2_foc_transformed}
    \MoveEqLeft \frac{\partial\ell_{jk}^{(n)}(\Sigma_{jk}, x_j^r,\hat{f_k}(x_k))}{\partial \Sigma_{jk}}                                                                                                                                 \\
    = \sum_{i=1}^n & \Bigg[\frac{1}{\Phi(\tilde{\Gamma}_j^{r}(\hat{f_k})) - \Phi(\tilde{\Gamma}_j^{r-1}(\hat{f_k}))}(1-(\Sigma_{jk})^2)^{-\frac{3}{2}}                                                                                  \\
                   & \Big[\phi({\tilde{\Gamma}}_j^{r}(\hat{f_k}))({\Gamma}_j^r\Sigma_{jk} - \hat{f_k}(\tilde{x}_{ik})) - \phi({\tilde{\Gamma}}_j^{r-1}(\hat{f_k}))({\Gamma}_j^{r-1}\Sigma_{jk} - \hat{f_k}(\tilde{x}_{ik}))\Big]\Bigg],
\end{align}
where
\[\hat{f}_k(\tilde{x}_{ik}) = \frac{\hat{f}_k(x_{ik}) - \hat{f}_k(\bar{x}_k)}{\sqrt{\hat{f}_k(s_k)^2}} \quad \text{and} \quad {\tilde{\Gamma}}_j^{r}(\hat{f}_k) = \frac{{\Gamma}_j^{r} - \Sigma_{jk}\hat{f}_k(\tilde{x}_{ik})}{\sqrt{1-(\Sigma_{jk})^2}}.\]

In the ensuing empirical evaluation, the data generation scheme is as follows. First, we generate \(n\) data points \((z_{ij}, x_{ik})_{i=1}^n\) from a standard bivariate normal with correlation \(\Sigma^*_{jk}\). Second, we apply the same transformation \(f_t^{-1}(x) = 5x^5\) for \(t \in \{j,k\}\) to all the data points. Third, we generate binary data \(x_{ij}^r\) by randomly choosing \(f^{-1}_j(z_{ij})\)-thresholds (guaranteeing relatively balanced classes) and then applying inversion sampling.

Computing the transformation-based MLE for \textit{Case II} can be achieved in several ways. Consider the plugged-in log-likelihood function, i.e.
\begin{equation}\label{eq:plugged_in_objective}
    \begin{split}
        \ell_{jk}^{(n)}(\Sigma_{jk}, x_j^r,\hat{f}_k(x_k)) & = \sum_{i=1}^n \big[\log(p(\hat{f}_k(x_{ik}))) + \log(p(x_{j}^{r} \mid \hat{f}_k(x_{ik}), \Sigma_{jk}))\big]                      \\
        & = \sum_{i=1}^n \big[\log(p(\hat{f}_k(x_{ik}))) + \Phi(\tilde{\Gamma}_j^{r}(\hat{f_k})) - \Phi(\tilde{\Gamma}_j^{r-1}(\hat{f_k}))\big].
    \end{split}
\end{equation}
One strategy to optimize Eq. \eqref{eq:plugged_in_objective} is direct maximization with a quasi-Newton optimization procedure to determine the optimal values for $\hat\Sigma_{jk}$. This strategy is used, for instance, in the R package \texttt{polycor} \citep{polycor2022}. Alternatively, another approach involves utilizing the Eq. \eqref{eq:case2_foc_transformed} and solving for $\hat\Sigma_{jk}$ through a nonlinear root-finding method. To do this, we employ Broyden's method \citep{Broyden1965}.

In \figref{fig:case2_comparison}, we generate data according to the scheme above for \(n=1000\) and a grid of true correlation values \(\Sigma^*_{jk} \in [0, 0.98]\) with a step size of \(s=0.02\). Due to symmetry, taking only positive correlations is sufficient for comparison purposes. For each correlation value along the grid, we generate \(100\) mixed binary-continuous data pairs and compute the MLE (using the abovementioned strategies) and the ad hoc estimator from Section 3.2 in the main text. We plot the true correlation values against the absolute difference between estimates and true correlation and the corresponding Monte-Carlo standard error for the MLE and the ad hoc estimator.

\begin{figure}\label{fig:case2_comparison}
    \includegraphics[width=\textwidth]{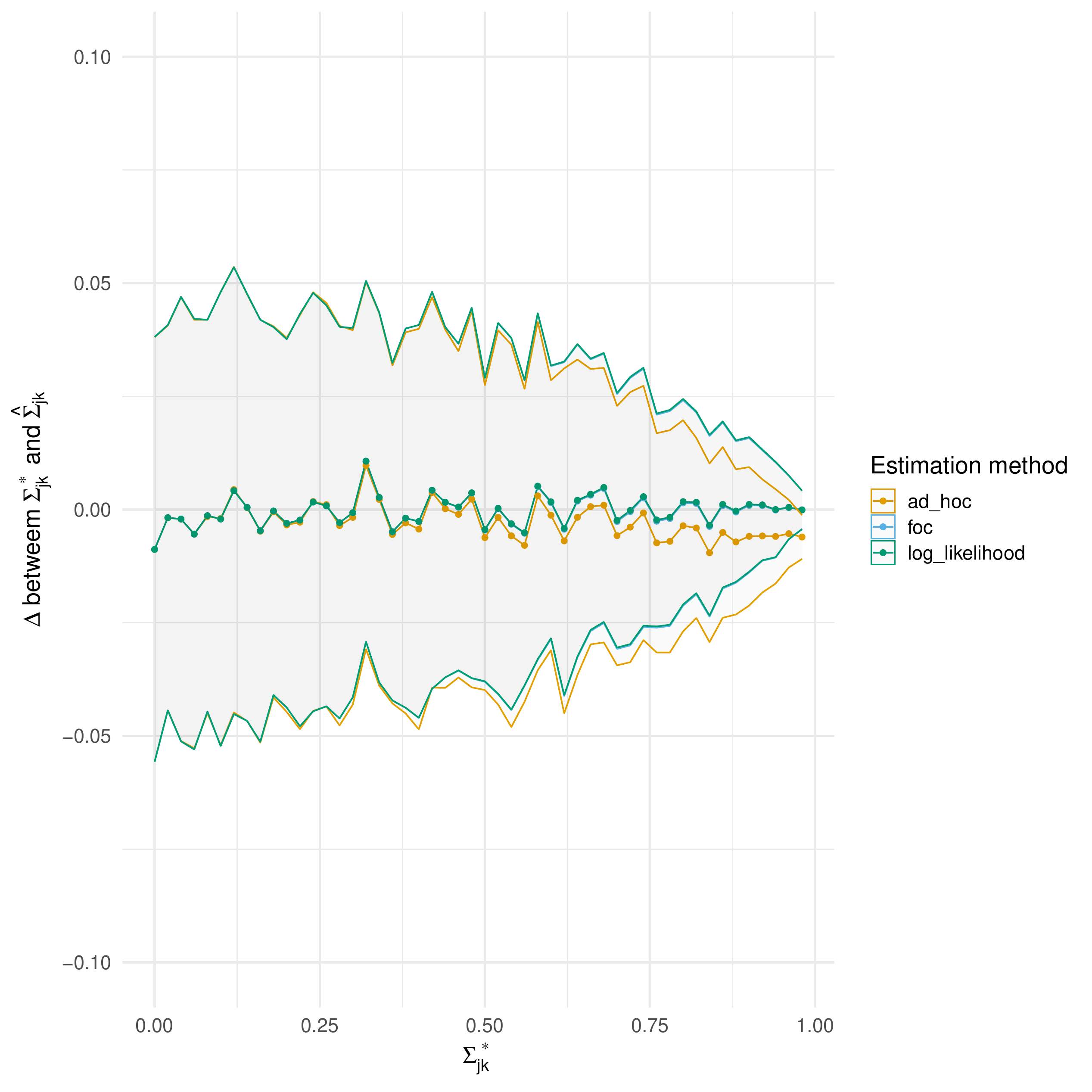}
    \caption{Comparison of the \textit{Case II} MLE under the LGCM and the ad hoc estimator.}
\end{figure}

As expected, both strategies for attaining the MLE yield the same results. The ad hoc estimator's bias becomes noticeable only when the underlying correlation \(\Sigma^*_{jk}\) exceeds \(0.75\) and it remains at such a mild level that we consider it negligible \citep[see][for a similar observation]{Olsson82}. The strength of the ad hoc estimator lies in its simplicity and computational efficiency. The left panel of \figref{fig:case2_speed} shows the median computation time surrounded by the first and third quartiles. We compare the two MLE optimization strategies and the ad hoc estimator for a grid of sample sizes \(n \in [50, 10000]\) with a step size of \(s_t = 50\). Here we fix \(\Sigma^*_{jk} = .87\) and repeat each calculation \(100\) times recording the time elapsed.

\begin{figure}\label{fig:case2_speed}
    \includegraphics[width=\textwidth]{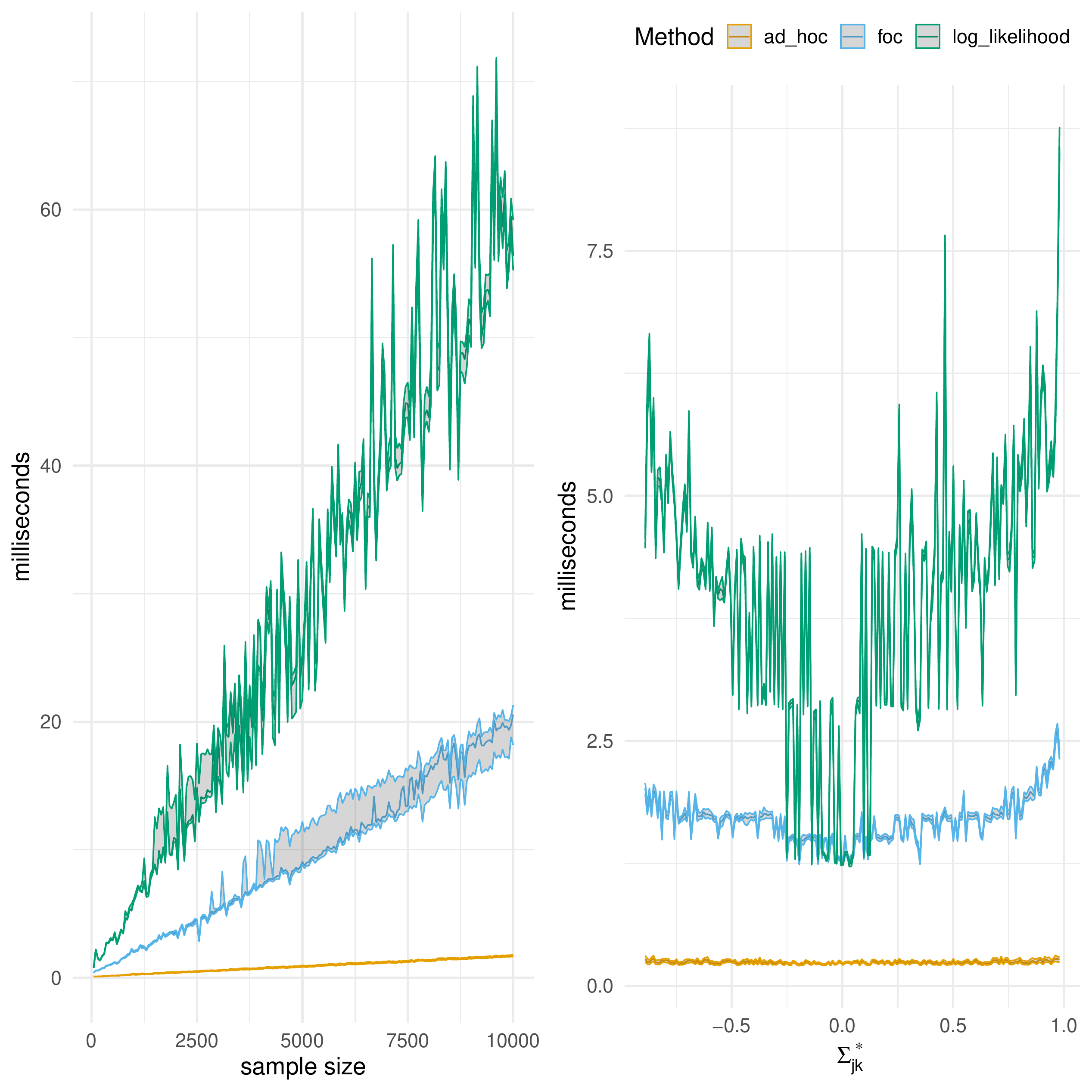}
    \caption{Computation time in milliseconds for the \textit{Case II} MLE and ad hoc estimators. We report the median (solid line) and the first and third quartile (shaded area) of recorded computation time. In the left panel, we compare computation time against a grid of sample sizes \(n \in [50, 10000]\) with a step size of \(s_t = 50\). In the right panel, we compare computation time against a grid of true correlation values \(\Sigma^*_{jk} \in [-.98, .98]\).}
\end{figure}

The right panel of \figref{fig:case2_speed} demonstrates computation time across a grid of length \(200\) of values for \(\Sigma_{jk}^* \in [-.98, 98]\). The sample size is, in this case, fixed at \(n=1000\). The ad hoc estimator is consistently and considerably faster than the MLE, regardless of the strategy used. The difference in computation time is especially pronounced for large sample sizes and correlation values approaching the endpoints of the $[-1,1]$-interval. Setting the FOC to zero and solving for \(\Sigma_{jk}\) is computationally more efficient than directly maximizing the log-likelihood function. The time difference in MLE strategies is more pronounced at the endpoints of the $[-1,1]$ interval. The ad hoc estimator is not affected by this issue. Therefore, in the high-dimensional setting we consider in this paper, the ad hoc estimator is preferable to the MLE due to (1) its computational efficiency, (2) its simplicity, which allows us to form concentration inequalities, and (3) its robustness to the underlying correlation value.

\section{Proof of Theorem 3.2
 }\label{proof_convergence}

\begin{condition}[Gradient statistical noise]\label{Gradient statistical noise}
    The gradient of the log-likelihood function is $\tau^2$-sub-Gaussian. That is, for any $\lambda \in \mathbb{R}$ and for all $\Sigma_{jk} \in [-1+\delta, 1-\delta]$ for $j,k$ according to \textit{Case II} or \textit{Case III}.
    \begin{equation}
        \mathbb{E}\Bigg[\exp\Bigg(\lambda \Big(\frac{\partial\ell_{jk}}{\partial \Sigma_{jk}} - \mathbb{E}\frac{\partial\ell_{jk}}{\partial \Sigma_{jk}} \Big) \Bigg)\Bigg] \leq \exp\Big(\frac{\tau^2\lambda^2}{2}\Big),
    \end{equation}
    where $\ell_{jk}$ corresponds to the log-likelihood functions in Definitions 2.4 and 2.5 of the main Manuscript, respectively.
    \paragraph{\textit{Case II}}: Recall that
    \begin{equation*}
        \frac{\partial\ell_{jk}(\Sigma_{jk}, x_j^r,x_k)}{\partial \Sigma_{jk}} = \frac{1}{p(x_{ij}^{r} \mid x_{ik}, \Sigma_{jk})} \frac{\partial p(x_{ij}^{r} \mid x_{ik}, \Sigma_{jk})}{\partial \Sigma_{jk}}.
    \end{equation*}
    Replacing these with the derivations made in Eq. \eqref{MLE_polyserial}, we write
    \begin{multline*}
        \frac{\partial\ell_{jk}(\Sigma_{jk}, x_j^r,x_k)}{\partial \Sigma_{jk}} = \\
        \frac{(1-(\Sigma_{jk})^2)^{-\frac{3}{2}}}{\Phi({\tilde{\Gamma}}_j^{r}) - \Phi({\tilde{\Gamma}}_j^{r-1})} \Bigg[\phi({\tilde{\Gamma}}_j^{r})({\Gamma}_j^r\Sigma_{jk} - {\tilde{x}}_{k}) - \phi({\tilde{\Gamma}}_j^{r-1})({\Gamma}_j^{r-1}\Sigma_{jk} - {\tilde{x}}_{k})\Bigg].
    \end{multline*}
    It is easy to see that $p(x_{ij}^{r} \mid x_{ik}, \Sigma_{jk}) \in (0,1)$ almost surely. Assumption 3.3
    makes sure that we exclude impossible events where $p(x_{ij}^{r} \mid x_{ik}, \Sigma_{jk}) = 0$. Moreover, we require that $\Gamma_j^r > \Gamma_j^{r-1}, \forall j \in 1, \dots, d_1$ this implies that $\Phi({\tilde{\Gamma}}_j^{r}) > \Phi({\tilde{\Gamma}}_j^{r-1})$. In other words, there exists a $\kappa >0$ such that $p(x_{ij}^{r} \mid x_{ik}, \Sigma_{jk}) \leq \frac{1}{\kappa}$.

    Let us now turn to $\partial p(x_{ij}^{r} \mid x_{ik}, \Sigma_{jk})/\partial \Sigma_{jk}$. First, for all $\Sigma_{jk} \in [-1+\delta, 1-\delta]$ we clearly have $1 \leq (1-(\Sigma_{jk})^2)^{-\frac{3}{2}} \leq \varpi$ for $\varpi > 1$. What's more, the density of the standard normal is bounded, i.e., $\abs{\phi(t)} \leq (2\pi)^{-\frac{1}{2}}$ for all $t \in \mathbb{R}$. Similarly,
    \begin{equation*}
        \abs{\phi({\tilde{\Gamma}}_j^{r})({\Gamma}_j^r\Sigma_{jk} - {\tilde{x}}_{k}) - \phi({\tilde{\Gamma}}_j^{r-1})({\Gamma}_j^{r-1}\Sigma_{jk} - {\tilde{x}}_{k})} \leq  \abs{\phi({\tilde{\Gamma}}_j^{r})({\Gamma}_j^r\Sigma_{jk} - {\tilde{x}}_{k})} \leq L_1,
    \end{equation*}
    due to Assumption 3.2. 
    Therefore,
    \begin{equation*}
        \abs{\frac{\partial\ell_{jk}(\Sigma_{jk}, x_j^r,x_k)}{\partial \Sigma_{jk}}} \leq \kappa L_1,
    \end{equation*}
    and $\Big(\frac{\partial\ell_{jk}}{\partial \Sigma_{jk}} - \mathbb{E}\frac{\partial\ell_{jk}}{\partial \Sigma_{jk}} \Big)$ is zero-mean and bounded. Then by  \citeauthor{Hoeffding63}'s \cite{Hoeffding63} lemma, the gradient of the log-likelihood function is $\tau^2$-sub-Gaussian with $\tau = 2\kappa L_1$

    \paragraph{\textit{Case III}}: Recall that we have
    \begin{equation*}
        \frac{\partial \ell_{jk}(\Sigma_{jk}, x_j^r,x_k^s)}{\partial \Sigma_{jk}} = \frac{1}{\pi_{rs}} \frac{\partial \pi_{rs}}{\partial \Sigma_{jk}}, \quad \text{for some }j<k \in [d_1].
    \end{equation*}
    Considering
    \begin{align*}
        \pi_{rs} = & \Phi_2({\Gamma}_j^r, {\Gamma}_k^s, \Sigma_{jk}) - \Phi_2({\Gamma}_j^{r-1}, {\Gamma}_k^s, \Sigma_{jk})          \\
        -          & \Phi_2({\Gamma}_j^r, {\Gamma}_k^{s-1}, \Sigma_{jk}) + \Phi_2({\Gamma}_j^{r-1}, {\Gamma}_k^{s-1}, \Sigma_{jk}),
    \end{align*}
    we note that this again has to be in $(0,1)$ due to Assumptions 3.1 and 3.2,
    such that $\pi_{rs} \leq \frac{1}{\xi}$. 

    Now let us show that
    \begin{align*}
        \frac{\partial \pi_{rs}}{\partial \Sigma_{jk}}
        = \Big[ & \phi_2({\Gamma}_j^r, {\Gamma}_k^s, \Sigma_{jk}) - \phi_2({\Gamma}_j^{r-1}, {\Gamma}_k^s, \Sigma_{jk})              \\
        -       & \phi_2({\Gamma}_j^r, {\Gamma}_k^{s-1}, \Sigma_{jk}) + \phi_2({\Gamma}_j^{r-1}, {\Gamma}_k^{s-1}, \Sigma_{jk})\Big]
    \end{align*}
    is bounded. Indeed, the density of the standard bivariate normal random variable is of the form $\phi_2(x,y) = c e^{-q(x,y)}$. Since $q(x,y)$ is a quadratic function of $x,y$ it follows that $\abs{\phi_2(x,y)} \leq c$. Therefore, every element in $\frac{\partial \pi_{rs}}{\partial \Sigma_{jk}}$ is bounded and thus $\abs{\frac{\partial \pi_{rs}}{\partial \Sigma_{jk}}} \leq K_1$. By the same argument as for \textit{Case II} $\Big(\frac{\partial\ell_{jk}}{\partial \Sigma_{jk}} - \mathbb{E}\frac{\partial\ell_{jk}}{\partial \Sigma_{jk}} \Big)$ is zero-mean and bounded and by Hoeffding's lemma the gradient of the log-likelihood function is $\tau^2$-sub-Gaussian with $\tau = 2\xi K_1$. Based on these arguments, we can conclude that the condition for gradient statistical noise is satisfied.
\end{condition}

\begin{condition}[Hessian statistical noise]\label{Hessian statistical noise}
    The Hessian of the log-likelihood function is $\tau^2$-sub-exponential, i.e. for all $\Sigma_{jk} \in [-1+\delta, 1-\delta]$ and for $j,k$ according to \textit{Case II} or \textit{Case III} we have
    \begin{equation}
        \norm{\frac{\partial^2\ell_{jk}}{\partial \Sigma_{jk}^{2}}}_{\psi_{1}} \leq \tau^2,
    \end{equation}
    where $\Norm{\cdot}_{\psi_{1}}$ denotes the \textit{Orlicz} $\psi_{1}$-norm, defined as
    \begin{equation*}
        \Norm{X}_{\psi_{1}} \coloneqq \sup_{p\geq 1} \frac{1}{p}\mathbb{E}\Big(\abs{X - \mathbb{E}(X)}^p\Big)^{\frac{1}{p}}.
    \end{equation*}
    Note that $\ell_{jk}$ corresponds to the respective log-likelihood functions in Definitions 2.4 and 2.5.

    \paragraph{\textit{Case II}}: We have
    \begin{equation}\label{hessian_case2}
        \begin{split}
            \frac{\partial^2\ell_{jk}}{\partial \Sigma_{jk}^{2}} &= \frac{\partial^2 p(x_j^{r} \mid x_{k}, \Sigma_{jk})/\partial \Sigma_{jk}^{2}}{p(x_j^{r} \mid x_{k}, \Sigma_{jk})} - \Bigg(\frac{\partial p(x_j^{r} \mid x_{k}, \Sigma_{jk}) / \partial \Sigma_{jk}}{p(x_j^{r} \mid x_{k}, \Sigma_{jk})}\Bigg)^2. \\
            \text{Clearly, } \abs{\frac{\partial^2\ell_{jk}}{\partial \Sigma_{jk}^{2}}} &\leq \frac{\abs{\partial^2 p(x_j^{r} \mid x_{k}, \Sigma_{jk})/\partial \Sigma_{jk}^{2}}}{\abs{p(x_j^{r} \mid x_{k}, \Sigma_{jk})}} + \Bigg( \frac{\abs{\partial p(x_j^{r} \mid x_{k}, \Sigma_{jk}) / \partial \Sigma_{jk}}}{\abs{p(x_j^{r} \mid x_{k}, \Sigma_{jk})}} \Bigg)^2 \\
            &\leq \kappa L_2 + \kappa^2 L_1^2,
        \end{split}
    \end{equation}
    where it remains to show that $\abs{\partial^2 p(x_{ij}^{r} \mid x_{ik}, \Sigma_{jk})/\partial \Sigma_{jk}^{2}} \leq L_2$. Indeed, we can rewrite our objective as follows:
    \begin{multline}\label{second_derivative_case2}
        \frac{\partial}{\partial \Sigma_{jk}} \Bigg(\frac{\partial\ell_{jk}(\Sigma_{jk}, x_j^r,x_k)}{\partial \Sigma_{jk}}\Bigg) \\
        \begin{aligned}
             & = \frac{\partial}{\partial \Sigma_{jk}}\Bigg( (1-(\Sigma_{jk})^2)^{-\frac{3}{2}} \Bigg[\phi({\tilde{\Gamma}}_j^{r})({\Gamma}_j^r\Sigma_{jk} - {\tilde{x}}_{k}) - \phi({\tilde{\Gamma}}_j^{r-1})({\Gamma}_j^{r-1}\Sigma_{jk} - {\tilde{x}}_{k})\Bigg] \Bigg)      \\
             & = \frac{3\Sigma_{jk}}{1-\Sigma_{jk}^2} (1-(\Sigma_{jk})^2)^{-\frac{3}{2}} \phi({\tilde{\Gamma}}_j^{r})({\Gamma}_j^r\Sigma_{jk} - {\tilde{x}}_{k}) + \frac{\phi^\prime({\tilde{\Gamma}}_j^{r})({\Gamma}_j^r\Sigma_{jk} - {\tilde{x}}_{k})^2}{(1-\Sigma_{jk}^2)^3} \\
             & + \frac{\phi({\tilde{\Gamma}}_j^{r}){\Gamma}_j^{r}}{(1-\Sigma_{jk}^2)^{-\frac{3}{2}}} - \frac{3\Sigma_{jk}}{1-\Sigma_{jk}^2} (1-(\Sigma_{jk})^2)^{-\frac{3}{2}} \phi({\tilde{\Gamma}}_j^{r-1})({\Gamma}_j^{r-1}\Sigma_{jk} - {\tilde{x}}_{k})                    \\
             & -  \frac{\phi^\prime({\tilde{\Gamma}}_j^{r-1})({\Gamma}_j^{r-1}\Sigma_{jk} - {\tilde{x}}_{k})^2}{(1-\Sigma_{jk}^2)^3} - \frac{\phi({\tilde{\Gamma}}_j^{r-1}){\Gamma}_j^{r-1}}{(1-\Sigma_{jk}^2)^{-\frac{3}{2}}}.
        \end{aligned}
    \end{multline}
    Thus,
    \begin{align*}
        \abs{\frac{\partial}{\partial \Sigma_{jk}} \Bigg(\frac{\partial\ell_{jk}(\Sigma_{jk}, x_j^r,x_k)}{\partial \Sigma_{jk}}\Bigg)}
         & \leq \abs{\frac{3\Sigma_{jk}}{1-\Sigma_{jk}^2} (1-(\Sigma_{jk})^2)^{-\frac{3}{2}} \phi({\tilde{\Gamma}}_j^{r})({\Gamma}_j^r\Sigma_{jk} - {\tilde{x}}_{k})}                                                 \\
         & + \abs{\frac{\phi^\prime({\tilde{\Gamma}}_j^{r})({\Gamma}_j^r\Sigma_{jk} - {\tilde{x}}_{k})^2}{(1-\Sigma_{jk}^2)^3} + \frac{\phi({\tilde{\Gamma}}_j^{r}){\Gamma}_j^{r}}{(1-\Sigma_{jk}^2)^{-\frac{3}{2}}}} \\
         & \leq L_2,
    \end{align*}
    due to Assumptions 3.1 and 3.2 
    and because both $\phi(t)$ and $\phi^{\prime}(t)$ are bounded for all $t\in\mathbb{R}$. Therefore, the inequality in Eq. \eqref{hessian_case2} follows and $\frac{\partial^2\ell_{jk}}{\partial \Sigma_{jk}^{2}} - \mathbb{E}\bigg(\frac{\partial^2\ell_{jk}}{\partial \Sigma_{jk}^{2}}\bigg)$ is bounded by $2(\kappa L_2 + \kappa^2 L_1^2)$. Hence, for all $p\geq1$
    \begin{equation}
        \frac{1}{p}\mathbb{E}\Bigg[\abs{\partial^2\ell_{jk}/\partial \Sigma_{jk}^{2} - \mathbb{E}\big(\partial^2\ell_{jk}/\partial \Sigma_{jk}^{2}\big)}^p\Bigg]^{\frac{1}{p}} \leq \frac{2}{p}\big(\kappa L_2 + \kappa^2 L_1^2\big).
    \end{equation}
    Finally, for $\tau = 2\kappa L_1$ we can choose $L_1$ and $\kappa$ such that \[2(\kappa L_2 + \kappa^2 L_1^2) \leq \tau^2 = 4\kappa^2 L_1^2.\] Thus, the \textit{Hessian statistical noise}-condition for \textit{Case II} is satisfied.

    \paragraph{\textit{Case III}:} Let us start with the Hessian of \(\ell_{jk}\) in the polychoric case:
    \begin{equation}\label{hessian_case3}
        \begin{split}
            \frac{\partial^2 \ell_{jk}(\Sigma_{jk}, x_j^r,x_k^s)}{\partial \Sigma_{jk}^{2}} &= \frac{\partial^2 \pi_{rs}/\partial \Sigma_{jk}^{2}}{\pi_{rs}} - \Bigg(\frac{\partial \pi_{rs}/\partial \Sigma_{jk}}{\pi_{rs}}\Bigg)^2 \\
            \text{Thus: } \abs{\frac{\partial^2 \ell_{jk}(\Sigma_{jk}, x_j^r,x_k^s)}{\partial \Sigma_{jk}^{2}}} &\leq \frac{\abs{\partial^2 \pi_{rs}/\partial \Sigma_{jk}^{2}}}{\abs{\pi_{rs}}} + \Bigg(\frac{\abs{\partial \pi_{rs}/\partial \Sigma_{jk}}}{\abs{\pi_{rs}}} \Bigg)^2 \\
            &\leq \xi K_2 + \xi^2K_1^2.
        \end{split}
    \end{equation}
    Again it remains to show that $\partial^2 \pi_{rs}/\partial \Sigma_{jk}^{2} \leq K_2$. Consider
    \begin{equation*}
        \begin{split}
            \lefteqn{\abs{\frac{\partial}{\partial \Sigma_{jk}}\bigg(\frac{\partial \pi_{rs}}{\partial \Sigma_{jk}}\bigg)}} \\
            &= \Big\lvert \frac{\partial}{\partial \Sigma_{jk}}\phi_2({\Gamma}_j^r, {\Gamma}_k^s, \Sigma_{jk}) - \phi_2({\Gamma}_j^{r-1}, {\Gamma}_k^s, \Sigma_{jk}) \\
            &- \phi_2({\Gamma}_j^r, {\Gamma}_k^{s-1}, \Sigma_{jk}) + \phi_2({\Gamma}_j^{r-1}, {\Gamma}_k^{s-1}, \Sigma_{jk})\Big\rvert \\
            &\leq \abs{\frac{\partial}{\partial \Sigma_{jk}}\phi_2({\Gamma}_j^r, {\Gamma}_k^s, \Sigma_{jk})} + \abs{\frac{\partial}{\partial \Sigma_{jk}}\phi_2({\Gamma}_j^{r-1}, {\Gamma}_k^s, \Sigma_{jk})} \\
            &+ \abs{\frac{\partial}{\partial \Sigma_{jk}}\phi_2({\Gamma}_j^r, {\Gamma}_k^{s-1}, \Sigma_{jk})} + \abs{\frac{\partial}{\partial \Sigma_{jk}}\phi_2({\Gamma}_j^{r-1}, {\Gamma}_k^{s-1}, \Sigma_{jk})} \\
            &\leq K_2,
        \end{split}
    \end{equation*}
    where each of the derivatives of the bivariate density is bounded, since we assume that the correlation is bounded away from one and minus one, i.e., \(\Sigma_{jk} \in [-1 + \delta, 1 - \delta]\).

    Similar to \textit{Case II}, for $\tau = 2\xi K_1$ we can choose $K_1$ and $\xi$ such that \[2(\xi k_2 + \xi^2 k_1^2) \leq \tau^2 = 4\xi^2 K_1^2.\]
    Consequently, the \textit{Hessian statistical noise}-condition for \textit{Case III} is satisfied, which concludes the proof of the \textit{Hessian statistical noise}-condition.
\end{condition}

Concerning the third condition, we introduce some additional notation. Denote the sample risk by $\hat{R}_{jk}^{(n)}(\Sigma_{jk})$ for \(j,k\) according to \textit{Case II} and \textit{Case III}, i.e.,
\begin{equation*}
    \text{\textit{Case II}:} \quad \hat{R}_{jk}^{(n)}(\Sigma_{jk}) = \frac{1}{n} \sum_{i=1}^n \big[\log(p(x_{ik})) + \log(p(x_{ij}^{r} \mid x_{ik}, \Sigma_{jk}))\big],
\end{equation*}
and
\begin{equation*}
    \text{\textit{Case III}:} \quad \hat{R}_{jk}^{(n)}(\Sigma_{jk}) = \frac{1}{n} \sum_{r=1}^{l_{X_{j}}}\sum_{s=1}^{l_{X_{k}}} n_{rs} \log(\pi_{rs}).
\end{equation*}
Lastly, we define $R_{jk}(\Sigma_{jk}) = \mathbb{E}_{\Sigma_{jk}^*}\hat{R}_{jk}^{(n)}(\Sigma_{jk})$ to be the population risk for each of the respective cases.
\begin{condition}[Hessian regularity]
    The Hessian regularity condition consists of three parts:
    \begin{enumerate}
        \item The second derivative of the population risk $R_{jk}(\Sigma_{jk})$ is bounded at one point. That is, there exists one $\Abs{\Bar{\Sigma}_{jk}} \leq 1- \delta$ and $H > 0$ such that $\Abs{R_{jk}^{\prime\prime}(\Bar{\Sigma}_{jk})} \leq H$.
        \item The second derivative of the log-likelihood with respect to $\Sigma_{jk}$ is Lipschitz continuous with integrable Lipschitz constant, i.e. there exists a $M^* > 0$ such that $\mathbb{E}[M] \leq M^*$, where
              \begin{equation*}
                  M = \sup_{\substack{\Abs{\Sigma_{jk}^{(1)}}, \  \Abs{\Sigma_{jk}^{(2)}} \leq 1- \delta, \\ \Sigma_{jk}^{(1)} \neq \Sigma_{jk}^{(2)}}} \frac{\abs{\ell_{jk}^{\prime\prime}(\Sigma_{jk}^{(1)}) - \ell_{jk}^{\prime\prime}(\Sigma_{jk}^{(2)})}}{\abs{\Sigma_{jk}^{(1)} - \Sigma_{jk}^{(2)}}}.
              \end{equation*}
        \item The constants $H$ and $M^*$ are such that $H \leq \tau^2$ and $M^*\leq \tau^3$.
    \end{enumerate}
    We need some intermediate results that make dealing with $R_{jk}(\Sigma_{jk})$ easier. First, note that $\mathbb{E}_{\Sigma_{jk}^*}\hat{R}_{jk}^{(n)}(\Sigma_{jk}) = \mathbb{E}_{\Sigma_{jk}^*}\ell_{jk}(\Sigma_{jk})$. Second, for all $\Sigma_{jk} \in [-1+\delta, 1-\delta]$ and $m \in \{1,2\}$
    \begin{equation*}
        R_{jk}^m(\Sigma_{jk}) = \frac{\partial^m}{\partial\Sigma_{jk}^m}\mathbb{E}_{\Sigma_{jk}^*}\ell_{jk}(\Sigma_{jk}) = \mathbb{E}_{\Sigma_{jk}^*}\frac{\partial^m}{\partial\Sigma_{jk}^m}\ell_{jk}(\Sigma_{jk}),
    \end{equation*}
    by Lemma \ref{expectation_commutes_case2} and Corollary \ref{expectation_commutes_case3}.

    \begin{enumerate}
        \item Recall the first part of the Hessian regularity condition, whereby Eq. \eqref{hessian_case2} and Eq. \eqref{hessian_case3} for all $\Sigma_{jk} \in [-1+\delta, 1-\delta]$ we have
              \[\text{\textit{Case II}:} \quad \Abs{\frac{\partial^2}{\partial\Sigma_{jk}^2}\ell_{jk}(\Sigma_{jk})} \leq \kappa L_2 + \kappa^2L_1^2\]
              and
              \[\text{\textit{Case III}:} \quad \Abs{\frac{\partial^2}{\partial\Sigma_{jk}^2}\ell_{jk}(\Sigma_{jk})} \leq \xi K_2 + \xi^2K_1^2\] for cases II and III, respectively. Consequently, the claim in the first part of the Hessian regularity condition holds for \textit{Case II} and \textit{Case III} for any $\Abs{\Bar{\Sigma}_{jk}} \leq 1-\delta$ with $H_{1} = \kappa L_2 + \kappa^2L_1^2$ and $H_{2} = \xi K_2 + \xi^2K_1^2$. Moreover, we also have $H_1 \leq \tau^2 = 4\kappa^2L_1^2$ and $H_2 \leq \tau^2 = 4\xi^2K_1^2$.

        \item The second part of the Hessian regularity condition requires that the second derivative of the log-likelihood with respect to $\Sigma_{jk}$ is Lipschitz continuous with integrable Lipschitz constant. By the mean-value-theorem, all we need to show is that we can find a bound on the third derivative of the log-likelihood function.

              \paragraph{\textit{Case II}:} Note that we have
              \begin{align*}
                  \MoveEqLeft \frac{\partial^3}{\partial\Sigma_{jk}^3}\ell_{jk}(\Sigma_{jk})                                                                                                                                                                                                                           \\
                   & = \frac{\partial}{\partial\Sigma_{jk}}\Bigg[\frac{\partial^2\ell_{jk}}{\partial \Sigma_{jk}^{2}}\Bigg]                                                                                                                                                                                            \\
                   & = \frac{\partial}{\partial\Sigma_{jk}}\Bigg[\frac{\partial^2 p(x_j^{r} \mid x_{k}, \Sigma_{jk})/\partial \Sigma_{jk}^{2}}{p(x_j^{r} \mid x_{k}, \Sigma_{jk})} - \Bigg(\frac{\partial p(x_j^{r} \mid x_{k}, \Sigma_{jk}) / \partial \Sigma_{jk}}{p(x_j^{r} \mid x_{k}, \Sigma_{jk})}\Bigg)^2\Bigg] \\
                   & = \frac{\partial^3 p(x_j^{r} \mid x_{k}, \Sigma_{jk}) / \partial \Sigma_{jk}^3}{p(x_j^{r} \mid x_{k}, \Sigma_{jk})}                                                                                                                                                                               \\
                   & - 3\frac{\big(\partial p(x_j^{r} \mid x_{k}, \Sigma_{jk}) / \partial \Sigma_{jk}\big)\big( \partial^2 p(x_j^{r} \mid x_{k}, \Sigma_{jk}) / \partial \Sigma_{jk}^2\big)}{(p(x_j^{r} \mid x_{k}, \Sigma_{jk}))^2}                                                                                   \\
                   & +  2\Bigg(\frac{\partial p(x_j^{r} \mid x_{k}, \Sigma_{jk}) / \partial \Sigma_{jk}}{p(x_j^{r} \mid x_{k}, \Sigma_{jk})}\Bigg)^3.
              \end{align*}
              Hence \[\abs{\frac{\partial^3}{\partial\Sigma_{jk}^3}\ell_{jk}(\Sigma_{jk})}
                  \leq \kappa L_3 + 3\kappa^2L_2L_1 + 2 \kappa^3L_1^3.\]
              It remains to show therefore, that \[\abs{\partial^3 p(x_j^{r} \mid x_{k}, \Sigma_{jk}) / \partial \Sigma_{jk}^3} \leq L_3.\]
              When taking the derivative of Eq. \eqref{second_derivative_case2}, it is obvious that the resulting statement is bounded due to Assumptions 3.1 and 3.2
              and the fact that $\phi(t), \phi^\prime(t), \phi^{\prime\prime}(t)$ are all bounded for all $t \in \mathbb{R}$.

              Therefore, by applying the mean-value-theorem we get \[M_1 \leq \kappa L_3 + 3\kappa^2L_2L_1 + 2 \kappa^3L_1^3,\]
              and the natural choice for \[M_1^* = \kappa L_3 + 3\kappa^2L_2L_1 + 2 \kappa^3L_1^3,\] where it follows that \[M_1^* \leq \tau^3 = 8\kappa^3L_1^3.\]

              \paragraph{\textit{Case III}:} We proceed similarly and first consider
              \begin{align}
                  \begin{split}
                      \MoveEqLeft \frac{\partial^3}{\partial\Sigma_{jk}^3}\ell_{jk}(\Sigma_{jk}) \\
                      & = \frac{\partial}{\partial\Sigma_{jk}}\Bigg[\frac{\partial^2 \ell_{jk}(\Sigma_{jk}, x_j^r,x_k^s)}{\partial \Sigma_{jk}^{2}}\Bigg] \\
                      &= \frac{\partial}{\partial\Sigma_{jk}}\Bigg[\frac{\partial^2 \pi_{rs}/\partial \Sigma_{jk}^{2}}{\pi_{rs}} - \Bigg(\frac{\partial \pi_{rs}/\partial \Sigma_{jk}}{\pi_{rs}}\Bigg)^2\Bigg] \\
                      &= \frac{\partial^3 \pi_{rs} / \partial \Sigma_{jk}^3}{\pi_{rs}} - 3\frac{\big(\partial \pi_{rs} / \partial \Sigma_{jk}\big)\big( \partial^2 \pi_{rs} / \partial \Sigma_{jk}^2\big)}{(\pi_{rs})^2} \\
                      &+ 2\Bigg(\frac{\partial \pi_{rs} / \partial \Sigma_{jk}}{\pi_{rs}}\Bigg)^3.
                  \end{split}
              \end{align}
              Hence \[\abs{\frac{\partial^3}{\partial\Sigma_{jk}^3}\ell_{jk}(\Sigma_{jk})} \leq \xi K_3 + 3\xi^2K_2K_1 + 2 \xi^3K_1^3.\]
              Taking a closer look at $\abs{\partial^3 \pi_{rs} / \partial \Sigma_{jk}^3}$, boundedness again follows from the fact that the quadratic function in the exponential of the bivariate normal density does not vanish. Thus, we have \[M_2 \leq \xi K_3 + 3\xi^2K_2K_1 + 2 \xi^3K_1^3,\] and the natural choice for \[M_2^* = \xi K_3 + 3\xi^2K_2K_1 + 2 \xi^3K_1^3,\] where we have \[M_2^* \leq \tau^3 = 8\xi^3K_1^3.\] These steps validate the \textit{Hessian regularity} condition.
    \end{enumerate}

    \begin{condition}[Population risk is strongly Morse]
        There exist $\epsilon > 0$ and $\eta > 0$ such that $R_{jk}(\Sigma_{jk})$ is $(\epsilon,\eta)$\textit{-strongly Morse}, i.e.
        \begin{enumerate}
            \item For all $\Sigma_{jk}$ such that $\Abs{\Sigma_{jk}} = 1-\delta$ we have that $\Abs{R_{jk}^\prime(\Sigma_{jk})} > \epsilon$.
            \item For all $\Sigma_{jk}$ such that $\Abs{\Sigma_{jk}} \leq 1-\delta$: \[ \Abs{R_{jk}^\prime(\Sigma_{jk})} \leq \epsilon \implies \Abs{R_{jk}^{\prime\prime}(\Sigma_{jk})} \geq \eta.\]
        \end{enumerate}
        Put differently, $R_{jk}(\Sigma_{jk})$ is $(\epsilon,\eta)$\textit{-strongly Morse} if the boundaries $\{-1+ \delta, 1-\delta\}$ are not critical points of $R_{jk}(\Sigma_{jk})$ and if $R_{jk}(\Sigma_{jk})$ only has finitely many critical points that are all non-degenerate.

        Let us verify that $R^{\prime\prime}(\Sigma_{jk}) \neq 0$ for cases II and III. Indeed by Lemma \ref{expectation_commutes_case2} and Corollary \ref{expectation_commutes_case3} we can rewrite $R_{jk}^{\prime\prime}(\Sigma_{jk})$ and obtain
        \begin{align*}
            \MoveEqLeft R^{\prime\prime}(\Sigma_{jk})                                                                                                                                                                                                                                                                                                     \\
             & = \mathbb{E}_{\Sigma_{jk}}\Bigg[\frac{\partial^2 \ell_{jk}(\Sigma_{jk})}{\partial\Sigma_{jk}^2}\Bigg]                                                                                                                                                                                                                                      \\
             & = \mathbb{E}_{\Sigma_{jk}^*} \begin{dcases*}
                                                \frac{\partial^2 p(x_j^{r} \mid x_{k}, \Sigma_{jk})/\partial \Sigma_{jk}^{2}}{p(x_j^{r} \mid x_{k}, \Sigma_{jk})} - \Bigg(\frac{\partial p(x_j^{r} \mid x_{k}, \Sigma_{jk}) / \partial \Sigma_{jk}}{p(x_j^{r} \mid x_{k}, \Sigma_{jk})}\Bigg)^2 & \text{\textit{Case II}},  \\
                                                \frac{\partial^2 \pi(\Sigma_{jk})_{rs}/\partial \Sigma_{jk}^{2}}{\pi(\Sigma_{jk})_{rs}} - \Bigg(\frac{\partial \pi(\Sigma_{jk})_{rs}/\partial \Sigma_{jk}}{\pi(\Sigma_{jk})_{rs}}\Bigg)^2                                                       & \text{\textit{Case III}},
                                            \end{dcases*}
        \end{align*}
        where in $\pi(\Sigma_{jk})_{rs}$ we made the dependence on $\Sigma_{jk}$ explicit.
        \paragraph{\textit{Case II}:} Recall that we have
        \begin{align*}
            \MoveEqLeft\mathbb{E}_{\Sigma_{jk}^*}\Bigg[\frac{\partial^2 p(x_j^{r} \mid x_{k}, \Sigma_{jk}^*)/\partial \Sigma_{jk}^{2}}{p(x_j^{r} \mid x_{k}, \Sigma_{jk}^*)}\Bigg]                                                  \\
             & = \int_{-\infty}^{\infty} \sum_{r=1}^{l_j+1} \frac{\partial^2 p(x_j^{r} \mid x_{k}, \Sigma_{jk}^*)/\partial \Sigma_{jk}^{2}}{p(x_j^{r} \mid x_{k}, \Sigma_{jk}^*)} p(x_j^{r}, x_{k}; \Sigma_{jk}^*) dx_k             \\
             & = \int_{-\infty}^{\infty} \sum_{r=1}^{l_j+1} \frac{\partial^2 p(x_j^{r} \mid x_{k}, \Sigma_{jk}^*)/\partial \Sigma_{jk}^{2}}{p(x_j^{r} \mid x_{k}, \Sigma_{jk}^*)} p(x_j^{r} \mid x_{k}, \Sigma_{jk}^*)p(x_{k}) dx_k \\
             & = \sum_{r=1}^{l_j+1} \partial^2 p(x_j^{r} \mid x_{k}, \Sigma_{jk}^*)/\partial \Sigma_{jk}^{2},
        \end{align*}
        with
        \begin{align*}
            \MoveEqLeft \sum_{r=1}^{l_j+1} \partial^2 p(x_j^{r} \mid x_{k}, \Sigma_{jk}^*)/\partial \Sigma_{jk}^{2}                                                                                                                 \\
             & =\sum_{r=1}^{l_j+1} \Bigg[ \frac{3\Sigma_{jk}}{1-\Sigma_{jk}^2} (1-(\Sigma_{jk})^2)^{-\frac{3}{2}} \phi({\tilde{\Gamma}}_j^{r})({\Gamma}_j^r\Sigma_{jk} - {\tilde{x}}_{k})                                           \\
             & + \frac{\phi^\prime({\tilde{\Gamma}}_j^{r})({\Gamma}_j^r\Sigma_{jk} - {\tilde{x}}_{k})^2}{(1-\Sigma_{jk}^2)^3} + \frac{\phi({\tilde{\Gamma}}_j^{r}){\Gamma}_j^{r}}{(1-\Sigma_{jk}^2)^{-\frac{3}{2}}}                 \\
             & - \frac{3\Sigma_{jk}}{1-\Sigma_{jk}^2} (1-(\Sigma_{jk})^2)^{-\frac{3}{2}} \phi({\tilde{\Gamma}}_j^{r-1})({\Gamma}_j^{r-1}\Sigma_{jk} - {\tilde{x}}_{k})                                                              \\
             & - \frac{\phi^\prime({\tilde{\Gamma}}_j^{r-1})({\Gamma}_j^{r-1}\Sigma_{jk} - {\tilde{x}}_{k})^2}{(1-\Sigma_{jk}^2)^3} - \frac{\phi({\tilde{\Gamma}}_j^{r-1}){\Gamma}_j^{r-1}}{(1-\Sigma_{jk}^2)^{-\frac{3}{2}}}\Bigg] \\
             & = 0,
        \end{align*}
        since all terms except the ones involving $\phi({\tilde{\Gamma}}_j^{0})$ and $\phi({\tilde{\Gamma}}^j_{l_j+1})$ cancel and furthermore $ \lim\limits_{t \to \pm \infty} \phi(t) = \lim\limits_{t \to \pm \infty} \phi^\prime(t) = 0$ .

        \paragraph{\textit{Case III}:} Similarly, consider
        \begin{align*}
            \MoveEqLeft \mathbb{E}_{\Sigma_{jk}^*}\Bigg[\frac{\partial^2 \pi(x_j^{r},x^k_{s}; \Sigma_{jk}^*)/\partial \Sigma_{jk}^{2}}{\pi(x_j^{r},x^k_{s}; \Sigma_{jk}^*)}\Bigg]          \\
             & = \sum_r \sum_s \Bigg[\frac{\partial^2 \pi(x_j^{r},x^k_{s}; \Sigma_{jk}^*)/\partial \Sigma_{jk}^{2}}{\pi(x_j^{r},x^k_{s}; \Sigma_{jk}^*)} P(X_j = x^r_j, X_k = x^s_k)\Bigg] \\
             & = \sum_r\sum_s \Big[\partial^2 \pi(x_j^{r},x^k_{s}; \Sigma_{jk}^*)/\partial \Sigma_{jk}^{2}\Big]                                                                            \\
             & = \sum_r\sum_s \Big[q(\Gamma_j^r, \Gamma_k^s, \Sigma_{jk}^*)\phi_2(\Gamma_j^r, \Gamma_k^s, \Sigma_{jk}^*)                                                                   \\
             & - q(\Gamma_j^{r-1}, \Gamma_k^s, \Sigma_{jk}^*)\phi_2(\Gamma_j^{r-1}, \Gamma_k^s, \Sigma_{jk}^*)                                                                             \\
             & - q(\Gamma_j^r, \Gamma_k^{s-1}, \Sigma_{jk}^*)\phi_2(\Gamma_j^r, \Gamma_k^{s-1}, \Sigma_{jk}^*)                                                                             \\
             & + q(\Gamma_j^{r-1}, \Gamma_k^{s-1}, \Sigma_{jk}^*)\phi_2(\Gamma_j^{r-1}, \Gamma_k^{s-1}, \Sigma_{jk}^*)\Big]                                                                \\
             & = q(\Gamma_j^{l_j +1}, \Gamma_k^{l_k +1}, \Sigma_{jk}^*)\phi_2(\Gamma_j^{l_j +1}, \Gamma_k^{l_k +1}, \Sigma_{jk}^*)                                                         \\
             & - q(\Gamma_j^{l_j +1}, \Gamma_k^0, \Sigma_{jk}^*)\phi_2(\Gamma_j^{l_j +1}, \Gamma_k^0, \Sigma_{jk}^*)                                                                       \\
             & - q(\Gamma_j^0, \Gamma_k^{l_k +1}, \Sigma_{jk}^*)\phi_2(\Gamma_j^0, \Gamma_k^{l_k +1}, \Sigma_{jk}^*)                                                                       \\
             & + q(\Gamma_j^0, \Gamma_k^0, \Sigma_{jk}^*)\phi_2(\Gamma_j^0, \Gamma_k^0, \Sigma_{jk}^*)                                                                                     \\
             & = 0,
        \end{align*}
        with $q(s,t,\Sigma_{jk}^*))$ denoting the corresponding quadratic function from the derivative of the bivariate normal density. As above, $\Gamma_k^{l_k+1} = \infty$ and $\Gamma_k^0 = -\infty$ for all $k \in 1, \dots d_1$. This, together with the fact that all other terms cancel when summing over $r,s$, $\phi(\cdot)$ is zero in all points containing $\Gamma_k^{l_k+1}, \Gamma_k^0$ and so the last equality follows.

        From this it follows, that $R_{jk}^{\prime\prime}(\Sigma_{jk}^*)$ can only be zero if $\partial p(x_j^{r} \mid x_{k}, \Sigma_{jk}^*) / \partial \Sigma_{jk}$ for \textit{Case II} and $\partial \pi(\Sigma_{jk}^*)_{rs}/\partial \Sigma_{jk}$ for \textit{Case III} are zero. However, this is not possible due to Assumptions 3.2 and 3.3.
        To see this note that in Eq. \eqref{case2_firstdiff} $\partial p(x_j^{r} \mid x_{k}, \Sigma_{jk}^*) / \partial \Sigma_{jk}$ can only be zero if either ${\Gamma}^r_j = {\Gamma}^{r-1}_j$ which we ruled out in the definition of the LGCM, or if $\Abs{{\Gamma}^r_j} = \Abs{{\Gamma}^{r-1}_j} = \infty$ which is ruled out by Assumption 3.2.
        We would not observe any discrete states in the first place if we had $r=\{0,l_{j}+1\}$. Assumption 3.3
        rules this case out.
        Consequently, there exist $\epsilon > 0$ and $\eta > 0$ such that $R_{jk}(\Sigma_{jk})$ is $(\epsilon,\eta)$\textit{-strongly Morse}
    \end{condition}
\end{condition}

With these considerations, we have verified the required four conditions to hold such that Theorem 2 in \citet{Mei18} 
holds for each couple $(j,k)$ according to cases II and III. More precisely, let $\alpha \in (0,1)$. Now, letting  $n \geq 4 C \log(n) \log(\frac{B}{\alpha})$
where $C = C_0 \Big(\frac{\tau^2}{\epsilon^2} \vee \frac{\tau^4}{\eta^2} \vee \frac{\tau^2L^2}{\eta^4} \Big)$ and $B = \tau(1-\delta)$
with $\tau = 2 [\kappa L_1 \vee \xi K_1]$  and $C_0$ denoting a universal constant. Letting further $L = \sup_{\Sigma_{jk}: \Abs{\Sigma_{jk}} \leq 1- \delta} \Abs{R_{jk}^{\prime\prime\prime}(\Sigma_{jk})}$ we obtain
\begin{equation}\label{individual_bound}
    \mathbb{P}\Bigg( \abs{\hat{\Sigma}^{(n)}_{jk} - \Sigma_{jk}^*} \geq \frac{2\tau}{\eta}\sqrt{C_0\frac{\log(n)}{n}\Big[\log\bigg(\frac{B}{\alpha}\bigg) \vee 1\Big] } \Bigg) \leq \alpha,
\end{equation}
and consequently, the result in Theorem 3.2
follows.



\section{Proof of Lemmas \ref{interchange_caseII_firstdif} to \ref{expectation_commutes_case3}}

\begin{lemma}\label{interchange_caseII_firstdif}
    For all $\Abs{\Sigma_{jk}} \leq 1-\delta$ and all $j \in [d_1], k\in [d_2]$ we have
    \begin{equation*}
        \int_S \frac{\partial }{\partial \Sigma_{jk}} p(x_j^{r} \mid x_{k}, \Sigma_{jk}) d\mu(x^r_j)= \frac{\partial }{\partial \Sigma_{jk}} \int_S p(x_j^{r} \mid x_{k}, \Sigma_{jk}) d\mu(x^r_j),
    \end{equation*}
    where $\mu$ is the counting measure on $S$, the corresponding discrete space.

    \begin{proof}
        Clearly, from Eq. \eqref{case2_firstdiff} we have
        \begin{multline*}
            \frac{\partial }{\partial \Sigma_{jk}} p(x_j^{r} \mid x_{k}, \Sigma_{jk}) \\
            =  (1-(\Sigma_{jk})^2)^{-\frac{3}{2}}\Big[\phi({\tilde{\Gamma}}_j^{r})({\Gamma}_j^r\Sigma_{jk} - {\tilde{x}}_{k}) - \phi({\tilde{\Gamma}}_j^{r-1})({\Gamma}_j^{r-1}\Sigma_{jk} - {\tilde{x}}_{k})\Big],
        \end{multline*}
        and therefore
        \begin{align*}
            \MoveEqLeft \int_S (1-(\Sigma_{jk})^2)^{-\frac{3}{2}}\Big[\phi({\tilde{\Gamma}}_j^{r})({\Gamma}_j^r\Sigma_{jk} - {\tilde{x}}_{k}) - \phi({\tilde{\Gamma}}_j^{r-1})({\Gamma}_j^{r-1}\Sigma_{jk} - {\tilde{x}}_{k})\Big] d\mu(x^r_j) \\
             & = (1-(\Sigma_{jk})^2)^{-\frac{3}{2}} \sum_{r=1}^{l_{j}} \Big[\phi({\tilde{\Gamma}}_j^{r})({\Gamma}_j^r\Sigma_{jk} - {\tilde{x}}_{k}) - \phi({\tilde{\Gamma}}_j^{r-1})({\Gamma}_j^{r-1}\Sigma_{jk} - {\tilde{x}}_{k})\Big]       \\
             & = 0                                                                                                                                                                                                                             \\
             & = \frac{\partial }{\partial \Sigma_{jk}} \sum_{r=1}^{l_{j}} p(x_j^{r} \mid x_{k}, \Sigma_{jk})                                                                                                                                  \\
             & = \frac{\partial }{\partial \Sigma_{jk}} 1,
        \end{align*}
        since all terms except the ones involving $\phi({\tilde{\Gamma}}_j^{0})$ and $\phi({\tilde{\Gamma}}_j^{l_j +1})$ cancel and
        \begin{equation*}
            \lim\limits_{t \to \pm \infty} \phi(t) = \lim\limits_{t \to \pm \infty} \phi^\prime(t) = 0,
        \end{equation*}
        and probabilities associated with all possible values must sum up to one.
    \end{proof}
\end{lemma}

\begin{corollary}
    For all $\Abs{\Sigma_{jk}} \leq 1-\delta$ and all $j \in [d_1], k\in [d_2]$ we have
    \begin{equation*}
        \int_S \frac{\partial^2 }{\partial \Sigma_{jk}^2} p(x_j^{r} \mid x_{k}, \Sigma_{jk}) d\mu(x^r_j)= \frac{\partial^2}{\partial \Sigma_{jk}^2} \int_S p(x_j^{r} \mid x_{k}, \Sigma_{jk}) d\mu(x^r_j),
    \end{equation*}
    where $\mu$ is the counting measure on $S$, the corresponding discrete space.

    \begin{proof}
        From Eq. \eqref{second_derivative_case2} we obtain
        \begin{align*}
            \MoveEqLeft \frac{\partial^2 }{\partial \Sigma_{jk}^2} p(x_j^{r} \mid x_{k}, \Sigma_{jk})                                                                                                                          \\
             & = \frac{3\Sigma_{jk}}{1-\Sigma_{jk}^2} (1-(\Sigma_{jk})^2)^{-\frac{3}{2}} \phi({\tilde{\Gamma}}_j^{r})({\Gamma}_j^r\Sigma_{jk} - {\tilde{x}}_{k})                                                               \\
             & +\frac{\phi^\prime({\tilde{\Gamma}}_j^{r})({\Gamma}_j^r\Sigma_{jk} - {\tilde{x}}_{k})^2}{(1-\Sigma_{jk}^2)^3} + \frac{\phi({\tilde{\Gamma}}_j^{r}){\Gamma}_j^{r}}{(1-\Sigma_{jk}^2)^{-\frac{3}{2}}}             \\
             & - \frac{3\Sigma_{jk}}{1-\Sigma_{jk}^2} (1-(\Sigma_{jk})^2)^{-\frac{3}{2}} \phi({\tilde{\Gamma}}_j^{r-1})({\Gamma}_j^{r-1}\Sigma_{jk} - {\tilde{x}}_{k})                                                         \\
             & - \frac{\phi^\prime({\tilde{\Gamma}}_j^{r-1})({\Gamma}_j^{r-1}\Sigma_{jk} - {\tilde{x}}_{k})^2}{(1-\Sigma_{jk}^2)^3} - \frac{\phi({\tilde{\Gamma}}_j^{r-1}){\Gamma}_j^{r-1}}{(1-\Sigma_{jk}^2)^{-\frac{3}{2}}}.
        \end{align*}
        By similar arguments to Lemma \ref{interchange_caseII_firstdif}, when taking the sum over all possible states, all terms except the ones involving $\phi({\tilde{\Gamma}}_j^{0})$ and $\phi({\tilde{\Gamma}}_j^{l_j +1})$ still cancel as they appear in every additive term in the above equation -- recall that $\phi^\prime(t) = -t\phi(t)$ -- and equality then follows immediately.
    \end{proof}
\end{corollary}

\begin{lemma}\label{expectation_commutes_case2}
    For all $\Abs{\Sigma_{jk}} \leq 1-\delta$ we have
    \begin{enumerate}
        \item \begin{equation*}
                  \frac{\partial }{\partial \Sigma_{jk}} \mathbb{E}_{\Sigma_{jk}^*} \big[\ell_{jk}(\Sigma_{jk}, x^r_j,x_k)\big] = \mathbb{E}_{\Sigma_{jk}^*} \Bigg[\frac{\partial }{\partial \Sigma_{jk}} \ell_{jk}(\Sigma_{jk}, x^r_j,x_k) \Bigg],
              \end{equation*}
              i.e.
              \begin{align*}
                  \frac{\partial }{\partial \Sigma_{jk}} & \int_{S\times \mathbb{R}} \log L_{jk}(\Sigma_{jk}, x^r_j,x_k) L_{jk}(\Sigma_{jk}^*, x^r_j,x_k) d\varepsilon(x^r_j,x_k)                                        \\
                  =                                      & \int_{S\times \mathbb{R}} \frac{\partial }{\partial \Sigma_{jk}} \log L_{jk}(\Sigma_{jk}, x^r_j,x_k) L_{jk}(\Sigma_{jk}^*, x^r_j,x_k) d\varepsilon(x^r_j,x_k)
              \end{align*}
              where $\varepsilon$ is the product measure on $S \times \mathbb{R}$ defined by
              \begin{equation*}
                  \varepsilon \coloneqq \mu \otimes \lambda
              \end{equation*}
              with $\mu$ denoting the counting measure on the corresponding discrete space $S$ and $\lambda$ the Lebesgue measure on the corresponding Euclidean space.
        \item \begin{equation*}
                  \frac{\partial }{\partial \Sigma_{jk}} \mathbb{E}_{\Sigma_{jk}^*} \big[\ell_{jk}(\Sigma_{jk}, x_j^r,x_k^s)\big] = \mathbb{E}_{\Sigma_{jk}^*} \Bigg[\frac{\partial }{\partial \Sigma_{jk}} \ell_{jk}(\Sigma_{jk}, x_j^r,x_k^s) \Bigg],
              \end{equation*}
              i.e.
              \begin{align*}
                  \frac{\partial }{\partial \Sigma_{jk}} & \int_{S\times S^\prime} \log L_{jk}(\Sigma_{jk}, x_j^r,x_k^s) L_{jk}(\Sigma_{jk}^*, x_j^r,x_k^s) d\varpi(x_j^r,x_k^s)                                        \\
                  =                                      & \int_{S\times S^\prime} \frac{\partial }{\partial \Sigma_{jk}} \log L_{jk}(\Sigma_{jk}, x_j^r,x_k^s) L_{jk}(\Sigma_{jk}^*, x_j^r,x_k^s) d\varpi(x_j^r,x_k^s)
              \end{align*}
              where $\varpi$ is the product measure on $S \times S^\prime$ defined by
              \begin{equation*}
                  \varpi \coloneqq \mu \otimes \mu^\prime
              \end{equation*}
              with $\mu$ and $\mu^\prime$ denoting the counting measure on the corresponding discrete space $S$ and $S^\prime$, respectively.
    \end{enumerate}

    \begin{proof}
        Let us start with 1. and rewrite the right-hand side:
        \begin{align*}
            \int_{S\times \mathbb{R}} \frac{\partial }{\partial \Sigma_{jk}} & \log L_{jk}(\Sigma_{jk}, x^r_j,x_k) L_{jk}(\Sigma_{jk}^*, x^r_j,x_k) d\varepsilon(x^r_j,x_k)                                                         \\
                                                                             & = \int_\mathbb{R} \sum_{r=1}^{l_{j}} \frac{\partial }{\partial \Sigma_{jk}} \log p(x_j^{r},x_{k}, \Sigma_{jk}) p(x_j^{r},x_{k}, \Sigma^*_{jk}) dx_k.
        \end{align*}
        The left-hand side corresponds to
        \begin{align*}
            \frac{\partial }{\partial \Sigma_{jk}} \int_{S\times \mathbb{R}} & \log L_{jk}(\Sigma_{jk}, x^r_j,x_k) L_{jk}(\Sigma_{jk}^*, x^r_j,x_k) d\varepsilon(x^r_j,x_k)                                                          \\
                                                                             & = \frac{\partial }{\partial \Sigma_{jk}} \int_\mathbb{R} \sum_{r=1}^{l_{j}} \log p(x_j^{r}, x_{k}, \Sigma_{jk}) p(x_j^{r},x_{k}, \Sigma_{jk}^*) dx_k.
        \end{align*}
        By Lebesgue's Dominated Convergence Theorem, we can interchange integration and differentiation as $\log p(x_j^{r}, x_{k}, \Sigma_{jk})$ is absolutely continuous s.t. its derivative exists almost everywhere and
        \begin{equation*}
            \abs{\frac{\partial \log p(x_j^{r}, x_{k}, \Sigma_{jk})}{\partial \Sigma_{jk}}}
        \end{equation*} is upper bounded by some integrable function. Indeed, the latter requirement has already been shown in Condition \ref{Gradient statistical noise}. The second point follows by the same arguments where $\log p(x_j^{r}, x^s_{k}, \Sigma_{jk}) = \log(C) + \log(\pi_{rs})$ is absolutely continuous and bounded as shown in Condition \ref{Gradient statistical noise}. This concludes the proof.
    \end{proof}
\end{lemma}

\begin{corollary}\label{expectation_commutes_case3}
    For all $\abs{\Sigma_{jk}} \leq 1-\delta$ we have \\
    $\frac{\partial^2 }{\partial \Sigma_{jk}^2} \mathbb{E}_{\Sigma_{jk}^*} \big[\ell_{jk}(\Sigma_{jk}, x^r_j,x_k)\big] = \mathbb{E}_{\Sigma_{jk}^*} \Bigg[\frac{\partial^2 }{\partial \Sigma_{jk}^2} \ell_{jk}(\Sigma_{jk}, x^r_j,x_k) \Bigg]$, for \textit{Case II} and \\
    $\frac{\partial^2 }{\partial \Sigma_{jk}^2} \mathbb{E}_{\Sigma_{jk}^*} \big[\ell_{jk}(\Sigma_{jk}, x_j^r,x_k^s)\big] = \mathbb{E}_{\Sigma_{jk}^*} \Bigg[\frac{\partial^2 }{\partial \Sigma_{jk}^2} \ell_{jk}(\Sigma_{jk}, x_j^r,x_k^s) \Bigg]$, for case III.

    \begin{proof}
        The claim follows immediately by the same arguments as in Lemma \ref{expectation_commutes_case2} and the bound on the second derivative of the log-likelihood functions in Condition \ref{Hessian statistical noise}, respectively.
    \end{proof}
\end{corollary}

\section{Proof of Lemma 3.1 of the Manuscript 
 }\label{lemma_threshold_proof}

First, note that $\Phi^{-1}(\cdot)$ is Lipschitz on $[\Phi(-G),\Phi(G)]$ with a Lipschitz constant \(L_1 = \nabla \Phi^{-1}(\hat\pi_j^r) =  1/(\sqrt{\frac{2}{\pi}} \min\{\hat\pi^r_j, 1- \hat\pi^r_j\})\). Then we have
\begin{align*}
    \Abs{\hat{\Gamma}^r_j - \Gamma^r_j} & = \Abs[\Big]{\Phi^{-1}\big(\frac{1}{n} \sum_{i=1}^n \mathbbm{1}{(X_{ij} \leq x^r_j)}\big) - \Phi^{-1}\big(\Phi(\Gamma^r_j)\big)} \\
                                        & \leq L_1\Abs[\Big]{\frac{1}{n} \sum_{i=1}^n \mathbbm{1}{(X_{ij} \leq x^r_j)} - \Phi(\Gamma^r_j)}.
\end{align*}

Applying Hoeffding's inequality, we obtain for some $t > 0$

\begin{align*}
    P\Big(\Abs{\hat\Gamma_j^r - \Gamma_j^r} \geq t \Big) & \leq P\Big(L_1\Abs[\Big]{\frac{1}{n} \sum_{i=1}^n \mathbbm{1}{(X_{ij} \leq x^r_j)} - \Phi(\Gamma^r_j)} \geq t  \Big) \\
                                                         & \leq 2\exp{\Big(- \frac{2t^2n}{L_1^2}\Big)}.
\end{align*}
This concludes the proof.

\section{Proof of Theorem 3.3
 }\label{proof_concentration2}

In what follows, the proof of Theorem 3.3
revolves largely around the Winsorized estimator introduced in Section 3.2. 
Recall that \[\hat{f}(x) = \Phi^{-1}(W_{\delta_n}[\hat{F}_{X_k}(x)])\] where \[W_{\delta_n}(u) \equiv \delta_n I(u < \delta_n) + u I(\delta_n \leq u \leq (1-\delta_n)) + (1-\delta_n) I(u > (1-\delta_n)),\] with truncation constant $\delta_n = 1/(4n^{1/4}\sqrt{\pi\log n})$. Recall $f(x) = \Phi^{-1}(F_{X_k}(x))$, and let $g = f^{-1}$.

Assume w.l.o.g. that we have consecutive integer scoring in our discrete variable $X_j$ such that the polyserial ad hoc estimator simplifies as
\begin{equation}
    \hat{\Sigma}_{jk}^{(n)} = \frac{r^{(n)}_{\hat{f}(X_k),X_j} \sigma^{(n)}_{X_j}}{\sum_{r=1}^{l_{j}} \phi(\bar{\Gamma}_j^r)(x_j^{r+1} - x_j^r)} = \frac{r^{(n)}_{\hat{f}(X_k),X_j} \sigma^{(n)}_{X_j}}{\sum_{r=1}^{l_{j}} \phi(\bar{\Gamma}_j^r)} = \frac{S_{\hat{f}(X_k)X_j}}{\sigma^{(n)}_{\hat{f}(X_k)}\sum_{r=1}^{l_{j}} \phi(\bar{\Gamma}_j^r)},
\end{equation}
for all $j \in [d_1], k \in [d_2]$. $S_{\hat{f}(X_k)X_j}$ denotes the sample covariance between the $\hat{f}(X_k)$ and the $X_j$, i.e.
\begin{equation*}
    S_{\hat{f}(X_k)X_j} = \frac{1}{n}\sum_{i=1}^n \Big(\hat{f}(X_{ik}) - \mu_n(\hat{f})\Big)\Big(X_{ij} - \mu_n(X_j)\Big),
\end{equation*}
where $\mu_n(\hat{f}) = 1/n\sum_{i=1}^n\hat f(X_{ik})$ and $\mu_n(X_j) = 1/n\sum_{i=1}^n X_{ij}$. Moreover, $\sigma^{(n)}_{\hat{f}(X_k)}$ denotes the sample standard deviation of the Winsorized estimator, i.e.
\begin{equation*}
    \sigma^{(n)}_{\hat{f}(X_k)} \equiv \sqrt{\frac{1}{n}\sum_{i=1}^n \Big(\hat{f}(X_{ik}) - \mu_n(\hat{f})\Big)^2}.
\end{equation*}
Recall that we treat the threshold estimates as given. In particular, we have here $\phi(\bar{\Gamma}_j^r)$, therefore note further that $\phi(\cdot)$, namely the density function of the standard normal is Lipschitz with Lipschitz constant $L_0 = (2\pi)^{-1/2}\exp(-1/2)$, s.t.
\begin{equation*}
    \abs{\phi(\bar{\Gamma}^r_j) - \phi(\Gamma^r_j)} \leq L_0 \abs{\bar{\Gamma}^r_j - \Gamma^r_j} \leq \abs{\bar{\Gamma}^r_j - \Gamma^r_j},
\end{equation*}
as $L_0 < 1$. Consequently, the statements regarding the accuracy of the threshold estimates in Section 3.4 
still hold here.

The outline of the proof is as follows: We start by forming concentration bounds for the sample covariance and the sample standard deviation separately. Then, we argue that the quotient of the two will be accurate in terms of a Lipschitz condition on the corresponding compactum. Let us start with the sample covariance. To study the Winsorized estimator, we consider the interval $[g(-\sqrt{M\log n}), g(\sqrt{M\log n})]$ for a choice of $M>2$. As the behavior of the estimator is different for the endpoints, we further split this interval into a middle and an end part, respectively, i.e.,
\begin{equation*}
    \begin{split}
        \mathbb{M}_n &\equiv (g(-\sqrt{\beta\log n}), g(\sqrt{\beta\log n})) \\
        \mathbb{E}_n &\equiv [g(-\sqrt{M\log n}), g(-\sqrt{\beta\log n})) \cup (g(\sqrt{\beta\log n}), g(\sqrt{M\log n})].
    \end{split}
\end{equation*}
This is only necessary for $\hat{f}(X_k)$ since $X_j \in 1, \dots, l_{j}$ is discrete and can therefore only take finitely many values. Now consider the sample covariance, where we have for any $t > 0$ that
\begin{align*}
    \MoveEqLeft P\Bigg(\max_{j,k} \abs{S_{\hat{f}(X_k)X_j} - S_{f(X_k)X_j}} > 2t \Bigg)                              \\
     & = P\Bigg(\max_{j,k} \Big\lvert\frac{1}{n}\sum_{i=1}^n \Big[\hat{f}(X_{ik})X_{ij} - f(X_{ik})X_{ij}            \\
     & \quad - \mu_n(\hat{f})\mu_n(X_j) + \mu_n(f)\mu_n(X_j)\Big]\Big\rvert > 2t \Bigg)                              \\
     & \leq P\Bigg(\max_{j,k}\abs{\frac{1}{n}\sum_{i=1}^n \Big[(\hat{f}(X_{ik}) - f(X_{ik}))X_{ij})\Big]} > t \Bigg) \\
     & \quad + P\Bigg(\max_{j,k}\abs{(\mu_n(\hat{f}) - \mu_n(f))\mu_n(X_j)} > t \Bigg).
\end{align*}
Let us take a closer look at the second term
\begin{multline*}
    P\Bigg(\max_{j,k}\abs{(\mu_n(\hat{f}) - \mu_n(f))\mu_n(X_j)} > t \Bigg) \\
    \begin{aligned}
         & = P\Bigg(\max_{j,k}\abs{ \frac{1}{n}\sum_{i=1}^n \left(\hat{f}(X_{ik}) - f(X_{ik})\right) \frac{1}{n}\sum_{i=1}^n X_{ij} } > t \Bigg)              \\
         & = P\Bigg(\max_{k}\abs{ \frac{1}{n}\sum_{i=1}^n \left(\hat{f}(X_{ik}) - f(X_{ik})\right)} \max_{j} \abs{\frac{1}{n}\sum_{i=1}^n X_{ij}}  > t \Bigg) \\
         & \leq P\Bigg(\max_{k}\abs{ \frac{1}{n}\sum_{i=1}^n \left(\hat{f}(X_{ik}) - f(X_{ik})\right)} > \frac{t}{l_{\max}} \Bigg),
    \end{aligned}
\end{multline*}
where $X_j$ is a discrete random variable with finite level set and $l_{\max} \equiv \max_j l_{j} > 0$.

Now, define
\begin{equation*}
    \bigtriangleup_i(j,k) \equiv (\hat{f}(X_{ik}) - f(X_{ik}))X_{ij}
\end{equation*}
and
\begin{equation*}
    \tilde{\bigtriangleup}_{r,s} \equiv (\hat{f}(s) - f(s))r,
\end{equation*}
for $r = 1, \dots, l_{j}$. Furthermore, consider the event $\mathcal{A}_n$, where
\begin{equation*}
    \mathcal{A}_n \equiv \{g(-\sqrt{M\log n}) \leq X_{1k}, \dots, X_{nk} \leq g(\sqrt{M\log n}), k = d_1 + 1, \dots, d\}.
\end{equation*}
The bound for the complement arises from the Gaussian maximal inequality \citet[Lemma 13]{Liu09}, i.e.,
\begin{equation*}
    P(\mathcal{A}_n^c) \leq P\left(\max_{i,k\in \{1,\dots,n\}\times\{d_1+1, \dots, d\}} \abs{f(X_{ik})} > \sqrt{2\log(nd_2)}\right) \leq \frac{1}{2\sqrt{\pi\log(nd_2)}}.
\end{equation*}

The following lemma gives insight into the behavior of the Winsorized estimator along the end region.
\begin{lemma}\label{end_region}
    On the event $\mathcal{A}_n$, consider $\beta = \frac{1}{2}$, $t \geq C_M\sqrt{\frac{\log d_2 log^2 n}{n^{1/2}}}$ and $A = \sqrt{\frac{2}{\pi}}(\sqrt{M}- \sqrt{\beta})$, then
    \begin{equation*}
        P\Bigg(\max_{j,k} \frac{1}{n} \sum_{X_k \in \mathbb{E}_n} \abs{(\hat{f}(X_{ik}) - f(X_{ik}))X_{ij}} > \frac{t}{2}\Bigg) \leq \exp\Big(- \frac{k_1 n^{3/4} \sqrt{\log n}}{k_2+ k_3} \Big),
    \end{equation*}
    and
    \begin{equation*}
        P\Bigg(\max_{k} \frac{1}{n} \sum_{X_k \in \mathbb{E}_n} \abs{\hat{f}(X_{ik}) - f(X_{ik})} > \frac{t}{2}\Bigg) \leq \exp\Big(- \frac{k_1 n^{3/4} \sqrt{\log n}}{k_2+ k_3} \Big),
    \end{equation*}
    where $k_i, i \in \{1,2,3\}$ are generic constants independent of sample size and dimension.

    \begin{proof}
        Let $\theta_1 \equiv \frac{n^{\beta/2}t}{4A\sqrt{\log n}}$ and let us first consider the bound for the first inequality.
        \begin{multline*}
            P\Bigg(\max_{j,k} \frac{1}{n} \sum_{X_k \in \mathbb{E}_n} \abs{(\hat{f}(X_{ik}) - f(X_{ik}))X_{ij}} > \frac{t}{2}\Bigg) \\
            =  P\Big(\max_{j,k} \frac{1}{n} \sum_{i: X_{ik} \in \mathbb{E}_n} \abs{(\hat{f}(X_{ik}) - f(X_{ik}))X_{ij}} > \frac{t}{2} \\
            \cap \max_{jk} \sup_{r \in \{1,\dots,l_{j}\}, s \in \mathbb{E}_n} \abs{\hat{f}(t) - f(t)}\abs{r} > \theta_1 \Big) \\
            + P\Big(\max_{j,k} \frac{1}{n} \sum_{i: X_{ik} \in \mathbb{E}_n} \abs{(\hat{f}(X_{ik}) - f(X_{ik}))X_{ij}} > \frac{t}{2} \\
            \cap \max_{jk} \sup_{r \in \{1,\dots,l_{j}\}, s \in \mathbb{E}_n} \abs{\hat{f}(s) - f(s)}\abs{r} \leq \theta_1 \Big) \\
            \leq P\Big(\max_{jk} \sup_{r \in \{1,\dots,l_{j}\}, s \in \mathbb{E}_n} \abs{\hat{f}(s) - f(s)}\abs{r} > \theta_1 \Big) + P\Big(\frac{1}{n} \sum_{i = 1}^n 1_{\{X_{ik} \in \mathbb{E}_n\}} > \frac{t}{2\theta_1}\Big).
        \end{multline*}
        Similarly, for the bound of the second inequality, we have
        \begin{multline*}
            P\Big(\max_k \frac{1}{n} \sum_{i: X_{ik} \in \mathbb{E}_n} \abs{\hat{f}(X_{ik}) - f(X_{ik})} > \frac{t}{2}\Big) \\
            =  P\Big(\max_k \frac{1}{n} \sum_{i: X_{ik} \in \mathbb{E}_n} \abs{\hat{f}(X_{ik}) - f(X_{ik})} > \frac{t}{2} \cap \max_k \sup_{s \in \mathbb{E}_n} \abs{\hat{f}(s) - f(s)} > \theta_1 \Big) \\
            + P\Big(\max_k \frac{1}{n} \sum_{i: X_{ik} \in \mathbb{E}_n} \abs{\hat{f}(X_{ik}) - f(X_{ik})} > \frac{t}{2} \cap \max_k \sup_{s \in \mathbb{E}_n} \abs{\hat{f}(s) - f(s)} \leq \theta_1\Big) \\
            \leq P\Big(\max_k \sup_{s \in \mathbb{E}_n} \abs{\hat{f}(s) - f(s)} > \theta_1 \Big) + P\Big(\frac{1}{n} \sum_{i = 1}^n 1_{\{X_{ik} \in \mathbb{E}_n\}} > \frac{t}{2\theta_1}\Big).
        \end{multline*}
        Recall that $\sup \{1,\dots,l_{j}\} = l_{j} > 0$. Furthermore, Lemma 16 in \cite{Liu09} states that for all $n$
        \begin{equation}
            \sup_{s \in \mathbb{E}_n}\abs{\hat{f}(s) - f(s)} < \sqrt{2 (M + 2) \log n}
        \end{equation}
        With this in mind, we have
        \begin{equation*}
            P\Big(\max_k \sup_{s \in \mathbb{E}_n} \abs{\hat{f}(s) - f(s)} > \theta_1 \Big) \leq d_2 P\Big(\sup_{s \in \mathbb{E}_n} \abs{\hat{f}(s) - f(s)} > \theta_1 \Big).
        \end{equation*}
        Recall that $C_M = 8/\sqrt{\pi}(\sqrt{2M} -1)(M+2)$ and since $t \geq C_M\sqrt{\frac{\log d_2 log^2 n}{n^{1/2}}}$, we have
        \begin{equation*}
            \theta_1 = \frac{n^{\beta/2}t}{4A\sqrt{\log n}} \geq \frac{C_M \sqrt{\log d_2 log^2 n}}{4A\sqrt{\log n}} = 2(M + 2)\log n.
        \end{equation*}
        Consequently, we have
        \begin{equation*}
            \theta_1 \geq 2(M + 2)\log n \geq \sqrt{2(M + 2)\log n},
        \end{equation*}
        as well as
        \begin{equation*}
            \frac{\theta_1}{l_{j}} \geq \sqrt{2(M + 2)\log n},
        \end{equation*}
        such that
        \begin{equation*}
            P\Big(\sup_{t \in \mathbb{E}_n} \abs{\hat{f}(t) - f(t)} > \theta_1 \Big) = P\Big(\sup_{r \in \{1,\dots,l_{j}\}, s \in \mathbb{E}_n} \abs{\hat{f}(s) - f(s)}\abs{r} > \theta_1 \Big) = 0.
        \end{equation*}
        In both cases, the second term is equivalent. We have
        \begin{multline*}
            P\Big(\frac{1}{n} \sum_{i = 1}^n 1_{\{X_{ik} \in \mathbb{E}_n\}} > \frac{t}{2\theta_1}\Big) = P\Big(\sum_{i = 1}^n 1_{\{X_{ik} \in \mathbb{E}_n\}} > \frac{nt}{2\theta_1}\Big) \\
            = P\Big(\sum_{i = 1}^n\big(1_{\{X_{ik} \in \mathbb{E}_n\}} - P(X_{1k} \in \mathbb{E}_n)\big) > \frac{nt}{2\theta_1} - P(X_{1k} \in n\mathbb{E}_n)\Big) \\
            \leq P\Big(\sum_{i = 1}^n\big(1_{\{X_{ik} \in \mathbb{E}_n\}} - P(X_{1k} \in \mathbb{E}_n)\big) > \frac{nt}{2\theta_1} - nA\sqrt{\frac{\log n}{n^\beta}}\Big).
        \end{multline*}
        Choosing $\theta_1$ this way guarantees that
        \begin{equation*}
            \frac{nt}{2\theta_1} - nA\sqrt{\frac{\log n}{n^\beta}} = nA\sqrt{\frac{\log n}{n^\beta}} > 0.
        \end{equation*}
        Then, using Bernstein's inequality, we get
        \begin{multline*}
            P\Big(\frac{1}{n} \sum_{i = 1}^n 1_{\{X_{ik} \in \mathbb{E}_n\}} > \frac{t}{2\theta_1}\Big) \\
            \begin{aligned}
                 & \leq P\Big(\sum_{i = 1}^n\big(1_{\{X_{ik} \in \mathbb{E}_n\}} - P(X_{1k} \in \mathbb{E}_n)\big) > nA\sqrt{\frac{\log n}{n^\beta}}\Big) \\
                 & \leq \exp\Big(- \frac{k_1 n^{2-\beta} \log n}{k_2 n^{1-\beta/2} \sqrt{\log n} + k_3 n^{1-\beta/2} \sqrt{\log n}} \Big),
            \end{aligned}
        \end{multline*}
        where $k_1,k_2,k_3 > 0$ are generic constants independent of $n$ and $d_2$. Collecting terms finishes the proof.
    \end{proof}
\end{lemma}

Turning back to the first decomposition of the sample covariance, we have
\begin{multline*}
    P\Bigg(\max_{j,k}\abs{\frac{1}{n}\sum_{i=1}^n \Big[(\hat{f}(X_{ik}) - f(X_{ik}))X_{ij})\Big]} > t \Bigg) \\
    \begin{aligned}
         & \leq P\Bigg(\max_{j,k}\abs{\frac{1}{n}\sum_{i=1}^n \bigtriangleup_i(j,k)} > t, \mathcal{A}_n \Bigg) + P(\mathcal{A}_n^c)                   \\
         & \leq P\Bigg(\max_{j,k}\abs{\frac{1}{n}\sum_{i=1}^n \bigtriangleup_i(j,k)} > t \cap \mathcal{A}_n\Bigg) + \frac{1}{2\sqrt{\pi \log(nd_2)}}.
    \end{aligned}
\end{multline*}
Further, we have
\begin{multline*}
    P\Bigg(\max_{j,k}\abs{\frac{1}{n}\sum_{i=1}^n \bigtriangleup_i(j,k)} > t \cap \mathcal{A}_n\Bigg) \\
    \begin{aligned}
         & \leq P\Bigg(\max_{j,k} \frac{1}{n} \sum_{X_k \in \mathbb{M}_n} \abs{\bigtriangleup_i(j,k)} > \frac{t}{2}\Bigg) + P\Bigg(\max_{j,k} \frac{1}{n} \sum_{X_k \in \mathbb{E}_n} \abs{\bigtriangleup_i(j,k)} > \frac{t}{2}\Bigg) \\
         & + \frac{1}{2\sqrt{\pi \log(nd_2)}}                                                                                                                                                                                         \\
         & \leq P\Bigg(\max_{j,k} \frac{1}{n} \sum_{X_k \in \mathbb{M}_n} \abs{\bigtriangleup_i(j,k)} > \frac{t}{2}\Bigg) + \exp\Big(- \frac{k_1 n^{3/4} \sqrt{\log n}}{k_2+ k_3} \Big)                                               \\
         & + \frac{1}{2\sqrt{\pi \log(nd_2)}},
    \end{aligned}
\end{multline*}
where $X_k \in \mathbb{M}_n$ is shorthand notation for $i:X_{ik} \in \mathbb{M}_n$. We derive the bound of the second term in Lemma \ref{end_region}. Thus, let us continue with the first term, where
\begin{multline*}
    P\Bigg(\max_{j,k} \frac{1}{n} \sum_{X_k \in \mathbb{M}_n} \abs{\bigtriangleup_i(j,k)} > \frac{t}{2}\Bigg) \\
    \begin{aligned}
         & \leq d^2P\Bigg(\sup_{r \in \{1,\dots, l_{j}\},s \in \mathbb{M}_n} \abs{\tilde{\bigtriangleup}_{r,s}} > \frac{t}{2}\Bigg) \\
         & = d^2P\Bigg(\sup_{r \in \{1,\dots, l_{j}\},s \in \mathbb{M}_n} \abs{(\hat{f}(s) - f(s))}\abs{r} > \frac{t}{2}\Bigg)      \\
         & = d^2P\Bigg(\sup_{s \in \mathbb{M}_n} \abs{(\hat{f}(s) - f(s))} > \frac{t}{2l_{j}}\Bigg),
    \end{aligned}
\end{multline*}
where clearly $\sup\{1,\dots,l_{j}\} = l_{j} > 0$. Define the event
\begin{equation*}
    \mathbb{B}_n \equiv \{\delta_n \leq \hat{F}_{X_k}(g_j(s)) \leq 1-\delta_n, \quad k \in [d_2]\}.
\end{equation*}
Now, from the definition of the Winsorized estimator, we observe that
\begin{multline*}
    d^2P\Bigg(\sup_{s \in \mathbb{M}_n} \abs{(\hat{f}(s) - f(s))} > \frac{t}{2l_{j}}\Bigg) \\
    \begin{aligned}
         & \leq d^2P\Bigg(\sup_{s \in \mathbb{M}_n} \abs{\Phi^{-1}(W_{\delta_n}[\hat{F}_{X_k}(s)]) - \Phi^{-1}(F_{X_k}(s))} > \frac{t}{2l_{j}} \cap \mathbb{B}_n\Bigg) + P(\mathbb{B}_n^c) \\
         & \leq d^2P\Bigg(\sup_{s \in \mathbb{M}_n} \abs{\Phi^{-1}(\hat{F}_{X_k}(s)) - \Phi^{-1}(F_{X_k}(s))} > \frac{t}{2l_{j}}\Bigg)                                                     \\
         & + 2\exp\Big(2\log d - \frac{\sqrt{n}}{8\pi\log n}\Big),
    \end{aligned}
\end{multline*}
where the expression for $P(\mathbb{B}_n^c))$ follows directly from Lemma 19 in \cite{Liu09}. Now, define
\begin{align*}
    T_{1n} & \equiv \max\Big\{F_{X_k}(g(\sqrt{\beta\log n})), 1-\delta_n\Big\} \\ T_{2n} &\equiv 1 - \min\Big\{F_{X_k}(g(-\sqrt{\beta\log n})), \delta_n\Big\},
\end{align*}
where it follows directly that $T_{1n} = T_{1n} = 1 - \delta_n$. Consequently, we apply the mean value theorem and get
\begin{multline*}
    P\Bigg(\sup_{s \in \mathbb{M}_n} \abs{\Phi^{-1}(\hat{F}_{X_k}(s)) - \Phi^{-1}(F_{X_k}(s))} > \frac{t}{2l_{j}}\Bigg) \\
    \begin{aligned}
         & \leq P\Bigg((\Phi^{-1})^{'}(\max(T_{1n}, T_{2n})) \sup_{s \in \mathbb{M}_n} \abs{\hat{F}_{X_k}(s) - F_{X_k}(s)} > \frac{t}{2l_{j}}\Bigg) \\
         & = P\Bigg((\Phi^{-1})^{'}(1-\delta_n) \sup_{s \in \mathbb{M}_n} \abs{\hat{F}_{X_k}(s) - F_{X_k}(s)} > \frac{t}{2l_{j}}\Bigg)              \\
         & \leq P\Bigg(\sup_{s \in \mathbb{M}_n} \abs{\hat{F}_{X_k}(s) - F_{X_k}(s)} > \frac{t}{(\Phi^{-1})^{'}(1-\delta_n)2l_{j}}\Bigg)            \\
         & \leq 2\exp\Bigg(-2\frac{nt^2}{4l_{j}^2[(\Phi^{-1})^{'}(1-\delta_n)]^2}\Bigg),
    \end{aligned}
\end{multline*}
where the last inequality arises from applying the Dvoretzky-Kiefer-Wolfowitz inequality. Now, we have that
\begin{align*}
    (\Phi^{-1})^{'}(1-\delta_n) & = \frac{1}{\phi\big(\Phi^{-1}(1-\delta_n)\big)}                                                                                             \\
                                & \leq \frac{1}{\phi\Big(\sqrt{2\log\frac{1}{\delta_n}}\Big)} = \sqrt{2\pi}\Big(\frac{1}{\delta_n}\Big) = 8\pi n^{\beta/2}\sqrt{\beta\log n}.
\end{align*}
Therefore,
\begin{multline*}
    d^2P\Bigg(\sup_{s \in \mathbb{M}_n} \abs{\Phi^{-1}(\hat{F}_{X_k}(s)) - \Phi^{-1}(F_{X_k}(s))} > \frac{t}{2l_{j}}\Bigg) \\
    \leq 2\exp\Bigg(2\log d - \frac{\sqrt{n}t^2}{64 l_{j}^2 \pi^2 \log n}\Bigg).
\end{multline*}
Collecting the remaining terms we have
\begin{multline*}
    P\Bigg(\max_{j,k} \frac{1}{n} \sum_{X_k \in \mathbb{M}_n} \abs{\bigtriangleup_i(j,k)} > \frac{t}{2}\Bigg) \\%
    \leq 2\exp\Bigg(2\log d - \frac{\sqrt{n}t^2}{64 l_{j}^2 \pi^2 \log n}\Bigg)
    + 2\exp\Big(2\log d - \frac{\sqrt{n}}{8\pi\log n}\Big).
\end{multline*}
Thus, we have for the first term in the covariance matrix decomposition
\begin{multline*}
    P\Bigg(\max_{j,k}\abs{\frac{1}{n}\sum_{i=1}^n \Big[(\hat{f}(X_{ik}) - f(X_{ik}))X_{ij})\Big]} > t \Bigg) \\
    \begin{aligned}
         & \leq P\Bigg(\max_{j,k} \frac{1}{n} \sum_{X_k \in \mathbb{M}_n} \abs{\bigtriangleup_i(j,k)} > \frac{t}{2}\Bigg)                     \\
         & + \exp\Big(- \frac{k_1 n^{3/4} \sqrt{\log n}}{k_2+ k_3} \Big) + \frac{1}{2\sqrt{\pi \log(nd_2)}}                                   \\
         & \leq 2\exp\Bigg(2\log d - \frac{\sqrt{n}t^2}{64 l_{j}^2 \pi^2 \log n}\Bigg) + 2\exp\Big(2\log d - \frac{\sqrt{n}}{8\pi\log n}\Big) \\
         & + \exp\Big(- \frac{k_1 n^{3/4} \sqrt{\log n}}{k_2+ k_3} \Big) + \frac{1}{2\sqrt{\pi \log(nd_2)}}.
    \end{aligned}
\end{multline*}
Let us turn back to the second term of the first sample covariance decomposition, i.e.
\begin{multline*}
    P\Bigg(\max_{j,k}\abs{(\mu_n(\hat{f}) - \mu_n(f))\mu_n(X_j)} > t \Bigg) \\
    \leq P\Bigg(\max_{k}\abs{ \frac{1}{n}\sum_{i=1}^n \left(\hat{f}(X_{ik}) - f(X_{ik})\right)} > \frac{t}{l_{\max}} \cap \mathcal{A}_n \Bigg) + \frac{1}{2\sqrt{\pi \log(nd_2)}}.
\end{multline*}
Now, analogous to before, we find
\begin{multline}\label{eq::std}
    P\Big(\max_k \frac{1}{n} \sum_{i=1}^n \abs{\hat{f}(X_{ik}) - f(X_{ik})} > \frac{t}{l_{\max}} \cap \mathcal{A}_n \Big) \\
    \begin{aligned}
         & \leq P\Big(\max_k \frac{1}{n} \sum_{X_{k} \in \mathbb{M}_n} \abs{\hat{f}(X_{ik}) - f(X_{ik})} > \frac{t}{2l_{\max}}\Big)  \\
         & + P\Big(\max_k \frac{1}{n} \sum_{X_{k} \in \mathbb{E}_n} \abs{\hat{f}(X_{ik}) - f(X_{ik})} > \frac{t}{2l_{\max}}\Big)     \\
         & \leq P\Big(\max_k \frac{1}{n} \sum_{X_{k} \in \mathbb{M}_n} \abs{\hat{f}(X_{ik}) - f(X_{ik})} > \frac{t}{2l_{\max}} \Big) \\
         & + \exp\Big(- \frac{k_1 n^{3/4} \sqrt{\log n}}{k_2+ k_3} \Big)                                                             \\
         & \leq d_2P\Big(\sup_{t \in \mathbb{M}_n} \abs{\hat{f}(t) - f(t)} > \frac{t}{2l_{\max}} \Big)                               \\
         & + \exp\Big(- \frac{k_1 n^{3/4} \sqrt{\log n}}{k_2+ k_3} \Big).
    \end{aligned}
\end{multline}
Let us take a closer look at
\begin{multline*}
    d_2P\Big(\sup_{t \in \mathbb{M}_n} \abs{\hat{f}(t) - f(t)} > \frac{t}{2l_{\max}} \Big) \\
    \begin{aligned}
         & \leq d_2P\Bigg(\sup_{t \in \mathbb{M}_n} \abs{\Phi^{-1}(W_{\delta_n}[\hat{F}_{X_k}(t)]) - \Phi^{-1}(F_{X_k}(t))} > \frac{t}{2l_{\max}} \cap \mathbb{B}_n\Bigg) \\
         & + d_2P(\mathbb{B}_n^c)                                                                                                                                         \\
         & \leq d_2P\Bigg(\sup_{t \in \mathbb{M}_n} \abs{\Phi^{-1}(\hat{F}_{X_k}(t)) - \Phi^{-1}(F_{X_k}(t))} > \frac{t}{2l_{\max}} \Bigg)                                \\
         & + 2\exp\Big(\log d_2 - \frac{\sqrt{n}}{8\pi\log n}\Big).
    \end{aligned}
\end{multline*}
The definition of the event $\mathbb{B}_n$ is the same as above. Then applying once more the Dvoretzky–Kiefer–Wolfowitz inequality we end up with the following upper bound:
\begin{multline*}
    d_2P\Bigg(\sup_{t \in \mathbb{M}_n} \abs{\Phi^{-1}(\hat{F}_{X_k}(t)) - \Phi^{-1}(F_{X_k}(t))} > \frac{t}{2l_{\max}} \Bigg) \\
    \leq 2\exp\Big(\log d_2 -\frac{\sqrt{n}t^2}{64 \ l_{\max}^2 \ \pi^2\log n}\Big).
\end{multline*}

Collecting terms and simplifying yields
\begin{multline*}
    P\Bigg(\max_{j,k}\abs{\frac{1}{n}\sum_{i=1}^n \Big[(\hat{f}(X_{ik}) - f(X_{ik}))X_{ij})\Big]} > t \Bigg) \\
    \begin{aligned}
         & \quad + P\Bigg(\max_{j,k}\abs{(\mu_n(\hat{f}) - \mu_n(f))\mu_n(X_j)} > t \Bigg)                                                                \\
         & \leq 2\exp\Bigg(2\log d - \frac{\sqrt{n}t^2}{64 \ l_{j}^2 \ \pi^2 \log n}\Bigg) + 2\exp\Big(2\log d - \frac{\sqrt{n}}{8\pi\log n}\Big)         \\
         & \quad + \exp\Big(- \frac{k_1 n^{3/4} \sqrt{\log n}}{k_2+ k_3} \Big) + \frac{1}{2\sqrt{\pi \log(nd_2)}}                                         \\
         & \quad + 2\exp\Bigg(\log d_2 - \frac{\sqrt{n}t^2}{64 \ l_{\max}^2 \ \pi^2 \log n}\Bigg) + 2\exp\Big(\log d_2 - \frac{\sqrt{n}}{8\pi\log n}\Big) \\
         & \quad + \exp\Big(- \frac{k_1 n^{3/4} \sqrt{\log n}}{k_2+ k_3} \Big) + \frac{1}{2\sqrt{\pi \log(nd_2)}}                                         \\
         & \leq 4\exp\Bigg(2\log d - \frac{\sqrt{n}t^2}{64 \ l_{\max}^2 \ \pi^2 \log n}\Bigg) + 4\exp\Big(2\log d - \frac{\sqrt{n}}{8\pi\log n}\Big)      \\
         & \quad + 2\exp\Big(- \frac{k_1 n^{3/4} \sqrt{\log n}}{k_2+ k_3} \Big) + \frac{1}{\sqrt{\pi \log(nd_2)}}.
    \end{aligned}
\end{multline*}
Then
\begin{multline}
    P\Bigg(\max_{j,k} \abs{S_{\hat{f}(X_k)X_j} - S_{f(X_k)X_j}} > 2t \Bigg) \\
    \begin{aligned}
         & \leq 4\exp\Bigg(2\log d - \frac{\sqrt{n}t^2}{64 \ l_{\max}^2 \ \pi^2 \log n}\Bigg) + 4\exp\Big(2\log d - \frac{\sqrt{n}}{8\pi\log n}\Big) \\
         & \quad + 2\exp\Big(- \frac{k_1 n^{3/4} \sqrt{\log n}}{k_2+ k_3} \Big) + \frac{1}{\sqrt{\pi \log(nd_2)}},
    \end{aligned}
\end{multline}
which completes the considerations regarding the sample covariance. \\

\noindent As a next step, we need to bound the error of the sample standard deviation of the Winsorized estimator
\begin{equation*}
    \sigma^{(n)}_{\hat{f}(X_k)} \equiv \sqrt{\frac{1}{n}\sum_{i=1}^n \Big(\hat{f}(X_{ik}) - \mu_n(\hat{f})\Big)^2}.
\end{equation*}

Consider the following decomposition of the standard deviation of the Winsorized estimator,
\begin{multline*}
    \abs{\sigma^{(n)}_{\hat{f}(X_k)} - \sigma^{(n)}_{f(X_k)}} \\
    \begin{aligned}
         & = \abs{\sqrt{1/n\sum_{i=1}^n \Big(\hat{f}(X_{ik}) - \mu_n(\hat{f})\Big)^2} - \sqrt{1/n\sum_{i=1}^n \Big(f(X_{ik}) - \mu_n(f)\Big)^2}}   \\
         & = \frac{1}{\sqrt{n}}\abs{\Big( \Norm{\hat{f}(X_k) - \mu_n(\hat{f})}_2 - \Norm{f(X_k) - \mu_n(f)}_2 \Big)}                               \\
         & \leq \frac{1}{\sqrt{n}} \Norm{\hat{f}(X_k) - \mu_n(\hat{f}) - f(X_k) + \mu_n(f)}_2                                                      \\
         & \leq \frac{1}{\sqrt{n}} \sqrt{n}\Norm{\hat{f}(X_k) - \mu_n(\hat{f}) - f(X_k) + \mu_n(f)}_\infty                                         \\
         & = \sup_{i: X_{ik} \in \{1,\dots,n\}} \abs{\hat{f}(X_{ik}) - f(X_{ik}) + \mu_n(f) - \mu_n(\hat{f})}                                      \\
         & = \sup_{i: X_{ik} \in \{1,\dots,n\}} \abs{\hat{f}(X_{ik}) - f(X_{ik})} + \abs{\mu_n(\hat{f}) - \mu_n(f)}                                \\
         & \leq \sup_{i: X_{ik} \in \{1,\dots,n\}} \abs{\hat{f}(X_{ik}) - f(X_{ik})} + \frac{1}{n} \sum_{i=1}^n \abs{\hat{f}(X_{ik}) - f(X_{ik})}.
    \end{aligned}
\end{multline*}
The first inequality is due to the reverse triangle inequality. The ensuing inequalities arise from applying standard norm equivalences. As before, we analyze both terms separately since we have to take care of the behavior of the Winsorized estimator, taking values at the end and the middle interval. We have for any $t > 0$,
\begin{multline*}
    P\left(\max_k \abs{\sigma^{(n)}_{\hat{f}(X_k)} - \sigma^{(n)}_{f(X_k)}} > 2t \right) \\
    \begin{aligned}
         & \leq P\Big(\max_k \sup_{i: X_{ik} \in \{1,\dots,n\}} \abs{\hat{f}(X_{ik}) - f(X_{ik})} > t, \mathcal{A}_n\Big)          \\
         & + P\Big(\max_k \frac{1}{n} \sum_{i=1}^n \abs{\hat{f}(X_{ik}) - f(X_{ik})} > t, \mathcal{A}_n\Big) + P(\mathcal{A}_n^c).
    \end{aligned}
\end{multline*}
Note that the second term is, in effect, equivalent to Eq. \eqref{eq::std} above such that we can immediately conclude that
\begin{multline*}
    P\Big(\max_k \frac{1}{n} \sum_{i=1}^n \abs{\hat{f}(X_{ik}) - f(X_{ik})} > t \cap \mathcal{A}_n \Big) \\
    \begin{aligned}
         & \leq d_2P\Big(\sup_{t \in \mathbb{M}_n} \abs{\hat{f}(t) - f(t)} > \frac{t}{2}\Big) + \exp\Big(- \frac{k_1 n^{3/4} \sqrt{\log n}}{k_2+ k_3} \Big) \\
         & \leq 2\exp\Big(\log d_2 -\frac{\sqrt{n}t^2}{64 \pi^2\log n}\Big) + 2\exp\Big(\log d_2 - \frac{\sqrt{n}}{8\pi\log n}\Big)                         \\
         & + \exp\Big(- \frac{k_1 n^{3/4} \sqrt{\log n}}{k_2+ k_3} \Big).
    \end{aligned}
\end{multline*}
The bound for the end region follows again from Lemma \ref{end_region}.

Similarly, we find that
\begin{multline*}
    P\Big(\max_k \sup_{i: X_{ik} \in \{1,\dots,n\}} \abs{\hat{f}(X_{ik}) - f(X_{ik})} > t \cap \mathcal{A}_n \Big) \\
    \begin{aligned}
         & \leq d_2P\Big(\sup_{t \in \mathbb{M}_n} \abs{\hat{f}(t) - f(t)} > \frac{t}{2} \Big) + d_2 P\Big(\sup_{t \in \mathbb{E}_n} \abs{\hat{f}(t) - f(t)} > \frac{t}{2}\Big) \\
         & = d_2P\Big(\sup_{t \in \mathbb{M}_n} \abs{\hat{f}(t) - f(t)} > \frac{t}{2} \Big),
    \end{aligned}
\end{multline*}
where the bound over the end region has been shown in Lemma \ref{end_region}. Thus, we only have to take care of
\begin{multline*}
    P\Big(\sup_{t \in \mathbb{M}_n} \abs{\hat{f}(t) - f(t)} > \frac{t}{2}\Big) \\
    \begin{aligned}
         & \leq P\Bigg(\sup_{t \in \mathbb{M}_n} \abs{\Phi^{-1}(W_{\delta_n}[\hat{F}_{X_k}(t)]) - \Phi^{-1}(F_{X_k}(t))} > \frac{t}{2} \cap \mathbb{B}_n\Bigg) + P(\mathbb{B}_n^c)      \\
         & \leq P\Bigg(\sup_{t \in \mathbb{M}_n} \abs{\Phi^{-1}(\hat{F}_{X_k}(t)) - \Phi^{-1}(F_{X_k}(t))} > \frac{t}{2}\Bigg) + 2\exp\Big(\log d_2 - \frac{\sqrt{n}}{8\pi\log n}\Big).
    \end{aligned}
\end{multline*}
The definition of the event $\mathbb{B}_n$ is the same as above. Then again, by the Dvoretzky-Kiefer-Wolfowitz inequality, we end up with the following upper bound:
\begin{equation*}
    P\Bigg(\sup_{t \in \mathbb{M}_n} \abs{\Phi^{-1}(\hat{F}_{X_k}(t)) - \Phi^{-1}(F_{X_k}(t))} > \frac{t}{2}\Bigg) \leq 2\exp\Big(-\frac{\sqrt{n}t^2}{64 \pi^2\log n}\Big).
\end{equation*}
Collecting terms, we arrive at the concentration bound for the sample standard deviation:
\begin{multline*}
    P\left(\max_k \abs{\sigma^{(n)}_{\hat{f}(X_k)} - \sigma^{(n)}_{f(X_k)}} > 2t \right) \\
    \begin{aligned}
         & \leq 4\exp\Big(\log d_2 -\frac{\sqrt{n}t^2}{64 \pi^2\log n}\Big) +  4\exp\Big(\log d_2 - \frac{\sqrt{n}}{8\pi\log n}\Big) \\
         & +  \exp\Big(- \frac{k_1 n^{3/4} \sqrt{\log n}}{k_2+ k_3} \Big) + \frac{1}{2\sqrt{\pi \log(nd_2)}}.
    \end{aligned}
\end{multline*}

With these intermediate results, we have shown that the sample covariance (numerator) and the sample standard deviation (denominator) can be estimated accurately.

The following lemma provides us with the means to form a probability bound for
\begin{equation*}
    \max_{jk}\abs{\hat{\Sigma}_{jk}^{(n)} -  \Sigma_{jk}^{*}}.
\end{equation*}

\begin{lemma}
    Consider the polyserial ad hoc estimator $\hat{\Sigma}_{jk}^{(n)}$ for $1 \leq j \leq d_1 < k \leq d_2$ and let $\epsilon \in \Big[ C_M\sqrt{\frac{\log d_2 log^2 n}{n^{1/2}}}, 8(1+4c^2)\Big]$, where $c$ is the corresponding sub-Gaussian parameter of the discrete variable.
    Both the numerator and the denominator are bounded, i.e.
    \begin{equation*}
        S_{\hat{f}(X_k)X_j} \in [-(1 + \epsilon), 1 + \epsilon],
    \end{equation*}
    and
    \begin{equation*}
        \sigma_{\hat{f}(X_k)}^{(n)} \in [1-\epsilon, 1+\epsilon].
    \end{equation*}
    Consequently, $\hat{\Sigma}_{jk}^{(n)}$ is Lipschitz with constant $L$. The following decomposition holds
    \begin{align*}
        \max_{jk}\abs{\hat{\Sigma}_{jk}^{(n)} -  \Sigma_{jk}^{*}} & =  \max_{jk}\abs{\hat{\Sigma}_{jk}^{(n)} - {\Sigma}_{jk}^{(n)} + {\Sigma}_{jk}^{(n)} - \Sigma_{jk}^{*}}                                                  \\
                                                                  & \leq \max_{jk}\abs{\hat{\Sigma}_{jk}^{(n)} - {\Sigma}_{jk}^{(n)}} + \max_{jk}\abs{{\Sigma}_{jk}^{(n)} - \Sigma_{jk}^{*}}                                 \\
                                                                  & \leq L \Big( \max_{jk} \abs{S_{\hat{f}(X_k)X_j} - S_{{f}(X_k)X_j} } + C_\Gamma \max_{k} \abs{\sigma_{\hat{f}(X_k)}^{(n)} - \sigma_{\hat{f}(X_k)}^{(n)} } \\
                                                                  & +  \max_{jk} \abs{S_{{f}(X_k)X_j} - S^*_{{f}(X_k)X_j}}  + C_\Gamma \max_{k} \abs{\sigma^{(n)}_{{f}(X_k)} - 1} \Big),
    \end{align*}
    where $C_\Gamma \equiv \sum_{r=1}^{l_{j}} \phi(\bar{\Gamma}_j^r)(x_j^{r+1} - x_j^r)$.
\end{lemma}

\begin{proof}
    Let us assume w.l.o.g. that $X_j$ - the discrete variable - has mean zero and variance one. By the Cauchy-Schwarz inequality, the true covariance of the pair is bounded from above by $1$, i.e.
    \begin{equation*}
        \abs{S^*_{f(X_k)X_j}} \leq \sigma^{2}_{f(X_k)} \sigma^{2}_{X_j} = 1.
    \end{equation*}
    Earlier we have shown that for some $t \geq C_M\sqrt{\frac{\log d_2 log^2 n}{n^{1/2}}}$ we have
    \begin{multline}
        P\Bigg(\max_{j,k} \abs{S_{\hat{f}(X_k)X_j} - S_{f(X_k)X_j}} > 2t \Bigg) \\
        \begin{aligned}
             & \leq 4\exp\Bigg(2\log d - \frac{\sqrt{n}t^2}{64 \ l_{\max}^2 \ \pi^2 \log n}\Bigg) + 4\exp\Big(2\log d - \frac{\sqrt{n}}{8\pi\log n}\Big) \\
             & + 2\exp\Big(- \frac{k_1 n^{3/4} \sqrt{\log n}}{k_2+ k_3} \Big) + \frac{1}{\sqrt{\pi \log(nd_2)}}.
        \end{aligned}
    \end{multline}
    According to Lemma 1 in \cite{Ravikumar11} with sub-Gaussian parameter $c > 0$, $f(X_k)$ is standard Gaussian and thus sub-Gaussian and $X_j$ is discrete and bounded and therefore also sub-Gaussian) we have the following tail bound
    \begin{equation*}
        P\left(\max_{jk}\abs{S_{{f}(X_k)X_j} - S^*_{{f}(X_k)X_j}} \geq t \right) \leq 4d^2 \exp\left\{- \frac{nt^2}{128(1+4c^2)^2} \right\},
    \end{equation*}
    for all $t \in (0,8(1+4c^2))$.
    Therefore with high probability for \(j \in [d_1], k \in [d_2]\), we have
    \begin{equation*}
        S_{\hat{f}(X_k)X_j} \in [S_{{f}(X_k)X_j} - 2t , S_{{f}(X_k)X_j} + 2t],
    \end{equation*}
    and since
    \begin{equation*}
        S_{{f}(X_k)X_j} \in [-1 - t , 1 + t],
    \end{equation*}
    with high probability, we have
    \begin{equation*}
        S_{\hat{f}(X_k)X_j} \in [-(1 + 3t), 1 + 3t].
    \end{equation*}
    Similar considerations hold for the sample standard deviation. We have already shown that
    \begin{multline*}
        P\left(\max_k \abs{\sigma^{(n)}_{\hat{f}(X_k)} - \sigma^{(n)}_{f(X_k)}} > 2t \right) \\
        \begin{aligned}
             & \leq 4\exp\Big(\log d_2 -\frac{\sqrt{n}t^2}{64 \pi^2\log n}\Big) +  4\exp\Big(\log d_2 - \frac{\sqrt{n}}{8\pi\log n}\Big) \\
             & +  \exp\Big(- \frac{k_1 n^{3/4} \sqrt{\log n}}{k_2+ k_3} \Big) + \frac{1}{2\sqrt{\pi \log(nd_2)}}.
        \end{aligned}
    \end{multline*}
    Furthermore, we use Lemma 1 again in \cite{Ravikumar11} to form a bound for the variance. Since $f(X_k)$ is standard Gaussian and hence sub-Gaussian with parameter $c = 1$ we immediately get
    \begin{equation*}
        P\left(\max_k \abs{(\sigma^{(n)}_{{f}(X_k)})^2 - 1} \geq t \right) \leq 4d_2\exp \left\{- \frac{nt^2}{128(1+4)^2} \right\}.
    \end{equation*}
    Put differently, with high probability
    \begin{equation*}
        (\sigma^{(n)}_{{f}(X_k)})^2 \in [1 - t, 1+ t],
    \end{equation*}
    and consequently, we also have with high probability
    \begin{equation*}
        \sigma^{(n)}_{{f}(X_k)} \in [\sqrt{1 -t}, \sqrt{1+t} ].
    \end{equation*}
    Since the interval $[1 -t, 1+t ]$ is always as least as wide as $[\sqrt{1 -t}, \sqrt{1+t} ]$, for all $t >0$ with high probability we then also have
    \begin{equation*}
        \sigma^{(n)}_{{f}(X_k)} \in [1 -t, 1+t ].
    \end{equation*}
    Putting these things together, we obtain that with high probability
    \begin{equation*}
        \sigma^{(n)}_{\hat{f}(X_k)} \in [1- 3t, 1 +3t].
    \end{equation*}

    In order to finish the proof consider the following function $h: \mathbb{R}\times \mathbb{R}_+ \to \mathbb{R}$ defined by
    \begin{equation*}
        h(u,v) = \frac{u}{v},
    \end{equation*}
    with $\nabla h = (1/v, -u/v^{2})^{T}$.

    As we have just shown for the polyserial ad hoc estimator, $\nabla h = (1/v, -u/v^2)^T$ is bounded with
    \begin{equation*}
        \sup \norm{\nabla h}_2 = \sqrt{\left(\frac{1}{1- 3t}\right)^2 + \left(-\frac{(1 + 3t)}{(1 - 3t)^2}\right)^2} \coloneqq L.
    \end{equation*}

    Consequently, $h$ is Lipschitz, and we have the following decomposition

    \begin{align*}
        \abs{h(u,v) - h(u',v')} & = \abs{h(u,v) - h(\tilde{u},\tilde{v}) + h(\tilde{u},\tilde{v}) - h(u',v')}                                     \\
                                & \leq \abs{h(u,v) - h(\tilde{u},\tilde{v})} + \abs{h(\tilde{u},\tilde{v}) - h(u',v')}                            \\
                                & \leq L \big(\abs{u-\tilde{u}} + \abs{v-\tilde{v}}\big) + L \big(\abs{\tilde{u} - u'} + \abs{\tilde{v}-v'}\big).
    \end{align*}
    Finally, taking $\epsilon = 3t$ finishes the proof.
\end{proof}

At last, collecting terms, we find that for $j \in [d_1]$ and $k [d_2]$ and any \[\epsilon \in \Big[C_M\sqrt{\frac{\log d \log^2 n}{\sqrt{n}}},8(1+4c^2)\Big]\] the following bound holds
\begin{align*}
    \MoveEqLeft P\left(\max_{jk}\abs{\hat{\Sigma}_{jk}^{(n)} -  \Sigma_{jk}^{*}} \geq \epsilon \right)                       \\
     & \leq P\left(\max_{j,k} \abs{S_{\hat{f}(X_k)X_j} - S_{f(X_k)X_j}} > \frac{\epsilon}{4L} \right)                        \\
     & \quad + P\left(\max_k \abs{\sigma^{(n)}_{\hat{f}(X_k)} - \sigma^{(n)}_{f(X_k)}} > \frac{\epsilon}{4LC_\Gamma} \right) \\
     & \quad + P\left(\max_{jk}\abs{S_{{f}(X_k)X_j} - S^*_{{f}(X_k)X_j}} \geq \frac{\epsilon}{4L} \right)                    \\
     & \quad +  P\left(\max_k \abs{(\sigma^{(n)}_{{f}(X_k)})^2 - 1} \geq \frac{\epsilon}{4lC_\Gamma} \right)
\end{align*}


The conclusion of Theorem 3.3
follows by plugging in the corresponding concentration bounds and simplifying.


\section{Additional simulation setup and results}\label{sec::additional_setup_results}

\subsection{FPR and TPR results for binary and general mixed data}

This section provides additional simulation results for the binary-continuous and general mixed data setting. In particular, we report TPR and FPR for the latent Gaussian model and the LGCM with \(f_j(x) = x^{1/3}\) for all \(j\). The results are depicted in Figures \ref{fig:bench_tpr_fpr_binary} and \ref{fig:bench_tpr_fpr_general} for the binary-continuous and general mixed data settings, respectively.
\begin{figure}
    \centering
    \includegraphics[width=\textwidth]{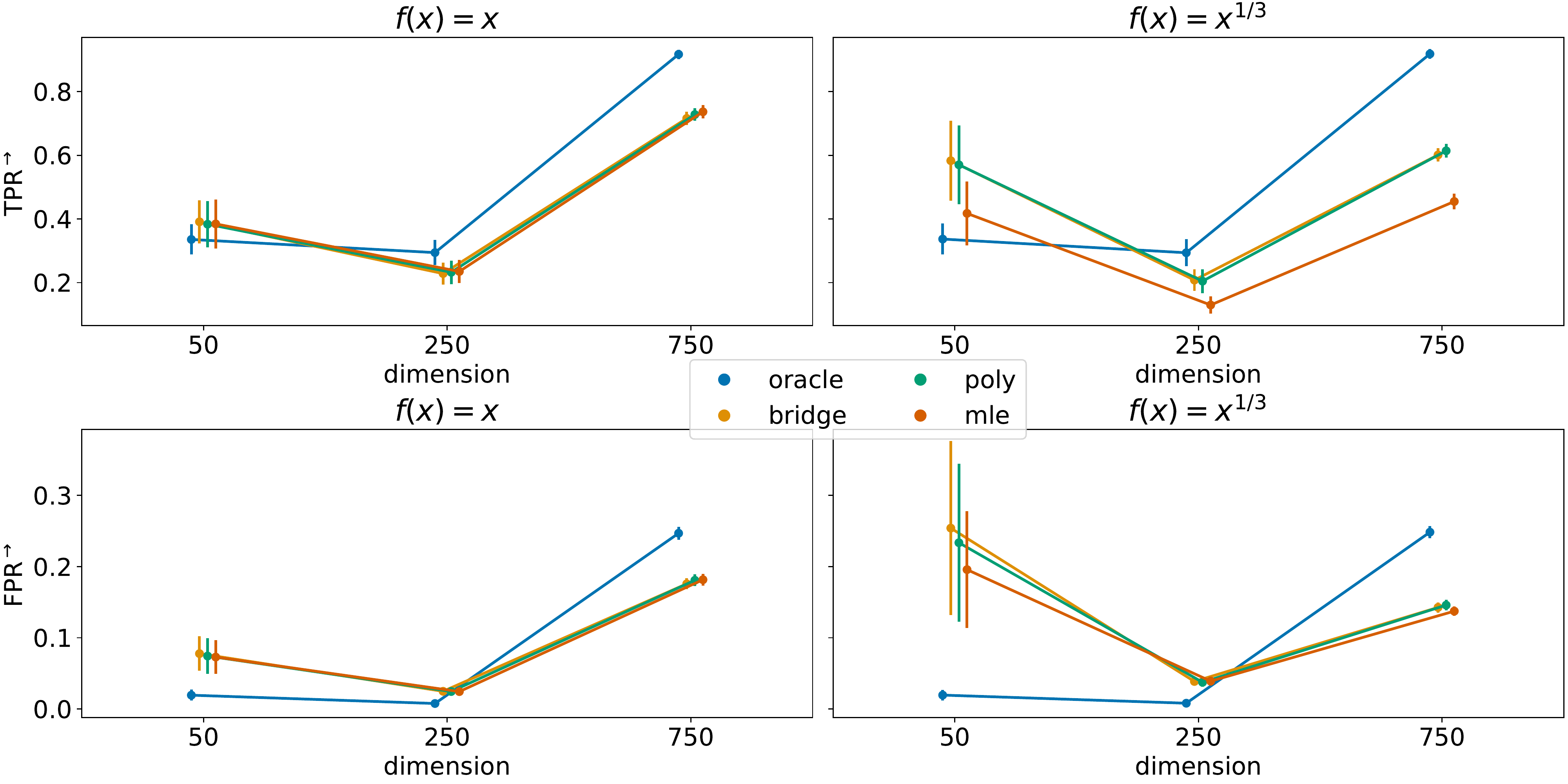}
    \caption{Simulation results for the binary-continuous data setting based on \(100\) simulation runs. The left column corresponds to the latent Gaussian model, where the transformation function is the identity. The right colum depicts results for the LGCM with \(f_j(x) = x^{1/3}\) for all \(j\). The top and bottom rows report mean and standard deviation of the TPR and FPR along simulation runs, respectively. The y-axis labels have superscript arrows attached to indicate the direction of improvement: \(\rightarrow\) implies that larger values are better.}
    \label{fig:bench_tpr_fpr_binary}
\end{figure}

\figref{fig:bench_tpr_fpr_binary} illustrates that our \texttt{poly} estimator performs almost identical to the gold standard \texttt{bridge} estimator proposed by \citet{Fan17}. In the right column, the \texttt{mle} estimator is misspecified, and TPR is noticeably lower than for the remaining estimators. Interestingly, it appears that the misspecified \texttt{mle} estimator does not infer additional edges, which is reflected in the FPR holding level with the other estimators.

The general mixed data setting results reported in \figref{fig:bench_tpr_fpr_general} are similar. The \texttt{poly} estimator performs almost identically to the \texttt{mle} estimator under the latent Gaussian model in terms of FPR and TPR. The \texttt{bridge} estimator has higher FPR and TPR in the \(d=50\) case and lower FPR and TPR in the \(d=750\) case. Under the LGCM, similar to the binary-continuous data setting, the \texttt{mle} estimator is misspecified and has a lower TPR than the other estimators while the FPR holds level. The \texttt{bridge} ensemble estimator has an even higher FPR and TPR in the \(d=50\) case and a slightly lower FPR and TPR in the \(d=750\) case.
\begin{figure}
    \centering
    \includegraphics[width=\textwidth]{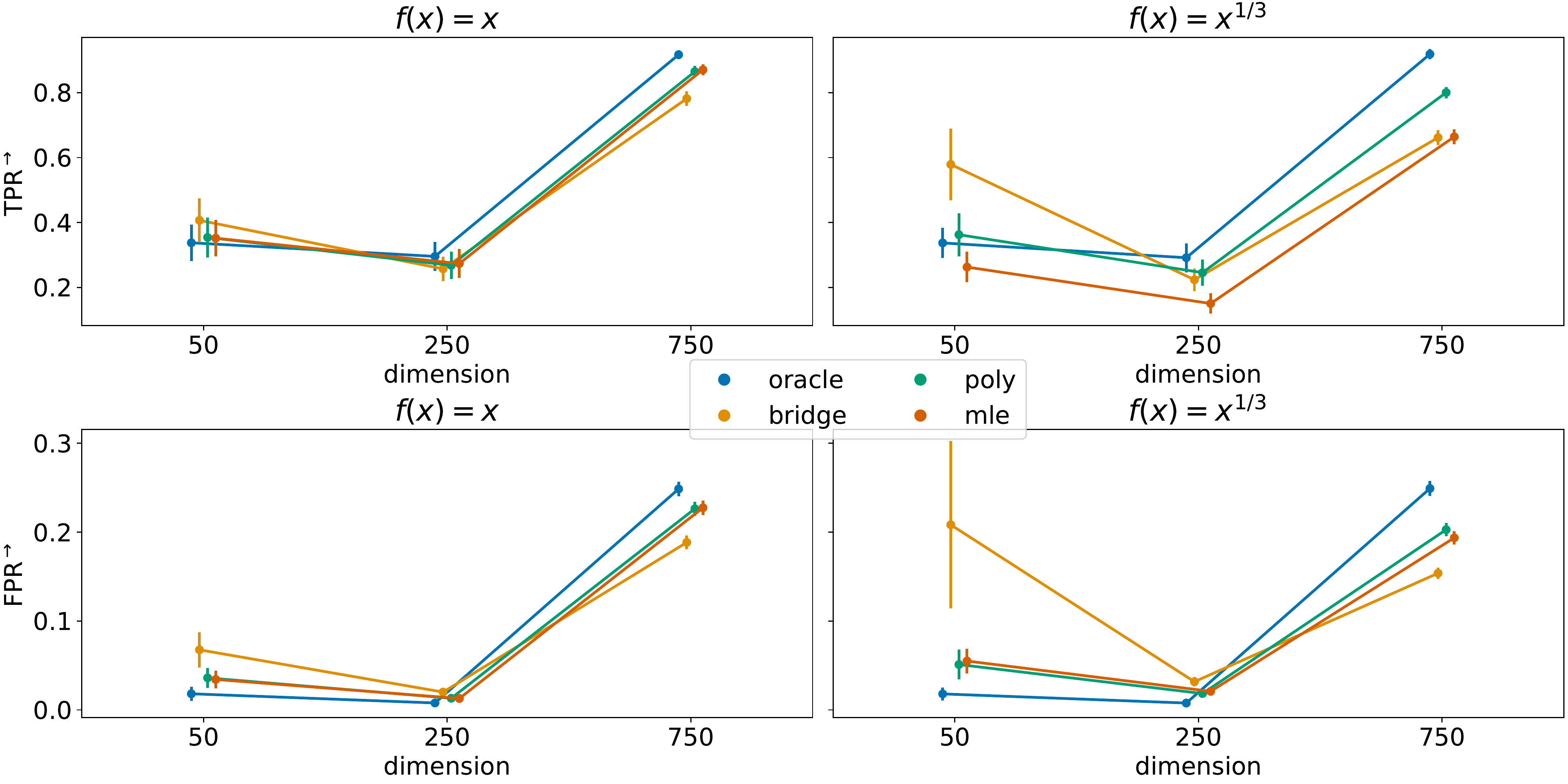}
    \caption{Simulation results for the general mixed data setting based on \(100\) simulation runs. The left column corresponds to the latent Gaussian model, where the transformation function is the identity. The right colum depicts results for the LGCM with \(f_j(x) = x^{1/3}\) for all \(j\). The top and bottom rows report mean and standard deviation of the TPR and FPR along simulation runs, respectively. The y-axis labels have superscript arrows attached to indicate the direction of improvement: \(\rightarrow\) implies that larger values are better.}
    \label{fig:bench_tpr_fpr_general}
\end{figure}

\subsection{Ternary mixed data results}

We additionally compare our \texttt{poly} and \texttt{mle} estimators with a generalization of the bridge function approach proposed by \citet{Quan18} given a mix of ternary, binary, and continuous data. In this case, $ \binom{2+2}{2} = 6$ individual bridge functions are needed.

The $(d,n)$-setup is analogous to the binary-continuous data setting from the main text. Starting with \texttt{mle}, the pattern from the binary-continuous setting continues to show in the ternary-binary-continuous mix. Whenever $f(x)=x$, the \texttt{mle} estimator generally performs best, in particular concerning graph recovery.

    \begin{longtable}[c]{@{}*{6}{>{\arraybackslash}p{0.135\linewidth}}@{}}
    \caption{Ternary mixed data structure learning; Simulated data with $100$ simulation runs. Standard errors in brackets \label{ternary_table}}
    \\[-1.8ex]\hline 
    \hline \\[-1.8ex] 
    $d,n,f(x)$ && Oracle $\hat{\Omega}$ & ternary $\hat{\Omega}_{\tau}$ & $\hat{\Omega}_{\text{MLE}}$ & $\hat{\Omega}_r$ \\ 
    \hline \\[-1.8ex] 
    \multirow{8}{*}{$50,200,x$} & $\Norm{\hat{\Omega} - \Omega}_F$ & $2.860$ & $2.936$ & $2.935$ & $2.930$ \\ [-.25em]
    & & \footnotesize{($0.098$)} & \footnotesize{($0.105$)} & \footnotesize{($0.106$)} & \footnotesize{($0.109$)} \\ [.15em]
    & FPR & $0.016$ & $0.067$ & $0.071$ & $0.075$ \\ [-.25em] 
    & & \footnotesize{($0.005$)} & \footnotesize{($0.017$)} & \footnotesize{($0.021$)} & \footnotesize{($0.023$)} \\ [.15em] 
    & TPR & $0.340$ & $0.370$ & $0.381$ & $0.389$ \\ [-.25em]
    & & \footnotesize{($0.046$)} & \footnotesize{($0.061$)} & \footnotesize{($0.068$)} & \footnotesize{($0.070$)} \\ [.15em] 
    & AUC & $0.880$ & $0.758$ & $0.769$ & $0.764$ \\ [-.25em]
    & & \footnotesize{($0.013$)} & \footnotesize{($0.019$)} & \footnotesize{($0.019$)} & \footnotesize{($0.020$)} \\    [1em]
    \multirow{8}{*}{$50,200,x^3$} & $\Norm{\hat{\Omega} - \Omega}_F$ & $2.856$ & $2.942$ & $3.053$ & $2.935$ \\ [-.25em]
    & & \footnotesize{($0.116$)} & \footnotesize{($0.102$)} & \footnotesize{($0.098$)} & \footnotesize{($0.108$)} \\ [.15em]
    & FPR & $0.016$ & $0.068$ & $0.076$ & $0.075$ \\ [-.25em]
    & & \footnotesize{($0.007$)} & \footnotesize{($0.019$)} & \footnotesize{($0.020$)} & \footnotesize{($0.022$)} \\ [.15em]
    & TPR & $0.342$ & $0.372$ & $0.280$ & $0.391$ \\ [-.25em]
    & & \footnotesize{($0.051$)} & \footnotesize{($0.059$)} & \footnotesize{($0.051$)} & \footnotesize{($0.066$)} \\ [.15em]
    & AUC & $0.882$ & $0.759$ & $0.691$ & $0.768$ \\ [-.25em]
    & & \footnotesize{($0.015$)} & \footnotesize{($0.019$)} & \footnotesize{($0.020$)} & \footnotesize{($0.019$)} \\  [1em]
    \multirow{8}{*}{$250,200,x$} & $\Norm{\hat{\Omega} - \Omega}_F$ & $3.185$ & $3.742$ & $3.709$ & $3.711$ \\ [-.25em]
    & & \footnotesize{($0.097$)} & \footnotesize{($0.090$)} & \footnotesize{($0.089$)} & \footnotesize{($0.091$)} \\ [.15em] 
    & FPR & $0.006$ & $0.025$ & $0.024$ & $0.025$ \\ [-.25em]
    & & \footnotesize{($0.001$)} & \footnotesize{($0.003$)} & \footnotesize{($0.003$)} & \footnotesize{($0.003$)} \\ [.15em]
    & TPR & $0.308$ & $0.238$ & $0.237$ & $0.235$ \\ [-.25em]
    & & \footnotesize{($0.034$)} & \footnotesize{($0.033$)} & \footnotesize{($0.031$)} & \footnotesize{($0.030$)} \\ [.15em]
    & AUC & $0.884$ & $0.759$ & $0.773$ & $0.768$ \\ [-.25em]
    & & \footnotesize{($0.014$)} & \footnotesize{($0.018$)} & \footnotesize{($0.018$)} & \footnotesize{($0.018$)} \\   [1em]
    \multirow{8}{*}{$250,200,x^3$} & $\Norm{\hat{\Omega} - \Omega}_F$ & $3.199$ & $3.757$ & $3.894$ & $3.724$ \\ [-.25em]
    & & \footnotesize{($0.096$)} & \footnotesize{($0.096$)} & \footnotesize{($0.096$)} & \footnotesize{($0.087$)} \\ [.15em]
    & FPR & $0.006$ & $0.025$ & $0.026$ & $0.025$ \\ [-.25em]
    & & \footnotesize{($0.001$)} & \footnotesize{($0.003$)} & \footnotesize{($0.003$)} & \footnotesize{($0.003$)} \\ [.15em]
    & TPR & $0.302$ & $0.239$ & $0.143$ & $0.237$ \\ [-.25em]
    & & \footnotesize{($0.034$)} & \footnotesize{($0.032$)} & \footnotesize{($0.027$)} & \footnotesize{($0.032$)} \\ [.15em] 
    & AUC & $0.882$ & $0.759$ & $0.691$ & $0.767$ \\ [-.25em]
    & & \footnotesize{($0.012$)} & \footnotesize{($0.016$)} & \footnotesize{($0.016$)} & \footnotesize{($0.015$)} \\  [1em]
    \multirow{8}{*}{$750,300,x$} & $\Norm{\hat{\Omega} - \Omega}_F$ & $11.181$ & $10.830$ & $10.640$ & $10.659$ \\ [-.25em] 
    & & \footnotesize{($0.134$)} & \footnotesize{($0.129$)} & \footnotesize{($0.122$)} & \footnotesize{($0.118$)} \\ [.15em]
    & FPR & $0.256$ & $0.179$ & $0.180$ & $0.179$ \\ [-.25em]
    & & \footnotesize{($0.006$)} & \footnotesize{($0.006$)} & \footnotesize{($0.006$)} & \footnotesize{($0.005$)} \\ [.15em] 
    & TPR & $0.937$ & $0.723$ & $0.744$ & $0.736$ \\ [-.25em]
    & & \footnotesize{($0.009$)} & \footnotesize{($0.016$)} & \footnotesize{($0.017$)} & \footnotesize{($0.016$)} \\ [.15em]
    & AUC & $0.939$ & $0.820$ & $0.831$ & $0.828$ \\ [-.25em]
    & & \footnotesize{($0.006$)} & \footnotesize{($0.009$)} & \footnotesize{($0.009$)} & \footnotesize{($0.009$)} \\ 
    [1em]
    \multirow{8}{*}{$750,300,x^3$} & $\Norm{\hat{\Omega} - \Omega}_F$ & $11.196$ & $10.838$ & $11.250$ & $10.646$ \\ [-.25em]
    & & \footnotesize{($0.130$)} & \footnotesize{($0.129$)} & \footnotesize{($0.130$)} & \footnotesize{($0.137$)} \\ [.15em] 
    & FPR & $0.256$ & $0.180$ & $0.173$ & $0.179$ \\ [-.25em]
    & & \footnotesize{($0.006$)} & \footnotesize{($0.006$)} & \footnotesize{($0.006$)} & \footnotesize{($0.006$)} \\ [.15em]
    & TPR & $0.937$ & $0.724$ & $0.590$ & $0.737$ \\ [-.25em]
    & & \footnotesize{($0.009$)} & \footnotesize{($0.016$)} & \footnotesize{($0.020$)} & \footnotesize{($0.016$)} \\ [.15em]
    & AUC & $0.939$ & $0.820$ & $0.743$ & $0.828$ \\ [-.25em]
    & & \footnotesize{($0.006$)} & \footnotesize{($0.008$)} & \footnotesize{($0.011$)} & \footnotesize{($0.009$)} \\ 
    \hline \\[-1.8ex] 
    \end{longtable}

However, when $f_j(x)=x^3$, performance drops notably, which is driven not by an increased FPR but by a decreased TPR.
Again, results for \texttt{bridge} and \texttt{poly} behave similarly across metrics. No performance reduction, neither in estimation error nor in graph recovery, can be detected when comparing \texttt{poly} to the ternary \texttt{bridge} estimator.


\section{Application to COVID-19 data}\label{sec::empirical_application}

This section presents the results of an analysis of real-world health data (from the UK Biobank). We are interested in investigating associations between the severity of a COVID-19 infection and various potential risk factors.
This analysis is intended to illustrate the use of the proposed methods in a real-world, mixed variable type example.

\begin{table}[ht]
  \begin{floatrow}
  \ttabbox
  {\caption{Estimated partial correlations between COVID-19 severity and the listed variables for data sets A, B and C.}\label{tab:covid_associations}}
  {\begin{tabular*}{\textwidth}{l @{\extracolsep{\fill}} cccc}
    \\[-1.8ex]\hline 
    \hline \\[-1.8ex]   
    Covid-19 severity assoc. Variables & Data set A & Data set B & Data set C\\
        \hline
        age				 & 0.162  & 0.134   &0.140\\
        waist circ.	     & 0.031  & 0.009	&0.011\\
        deprev. idx      & 0.016  & - 	& -	  \\
        sex				 & 0.007  &	 -      & -	  \\
        hypertension     & 0.075  & 0.035	&0.037\\
        heart attack     &  -     & 0.073   &0.065\\
        diabetes		 & -     & 0.062   &0.055\\
        chr. bronch.     & - &- &   0.012\\
        wisd. teeth surg.& - &  -0.003    & -\\
        \hline
  \end{tabular*}}%
  \end{floatrow}
\end{table}

\subsection{Data set and variables}

We first describe the data set used here, which is a part of the UK Biobank COVID-19 resource in which UK Biobank data were linked to clinical COVID data. To construct an indicator of COVID-19 severity, we consider subjects who tested positive for COVID-19 at some point in 2020. Based on that, we created an indicator variable (COVID severity) to capture whether each subject had a severe outcome within six weeks of infection (meaning either hospitalized, hospitalized, receiving critical care, or died). Around $14\%$ experienced such a severe outcome. \begin{figure}[ht]
    \centering
    \includegraphics[scale=.6]{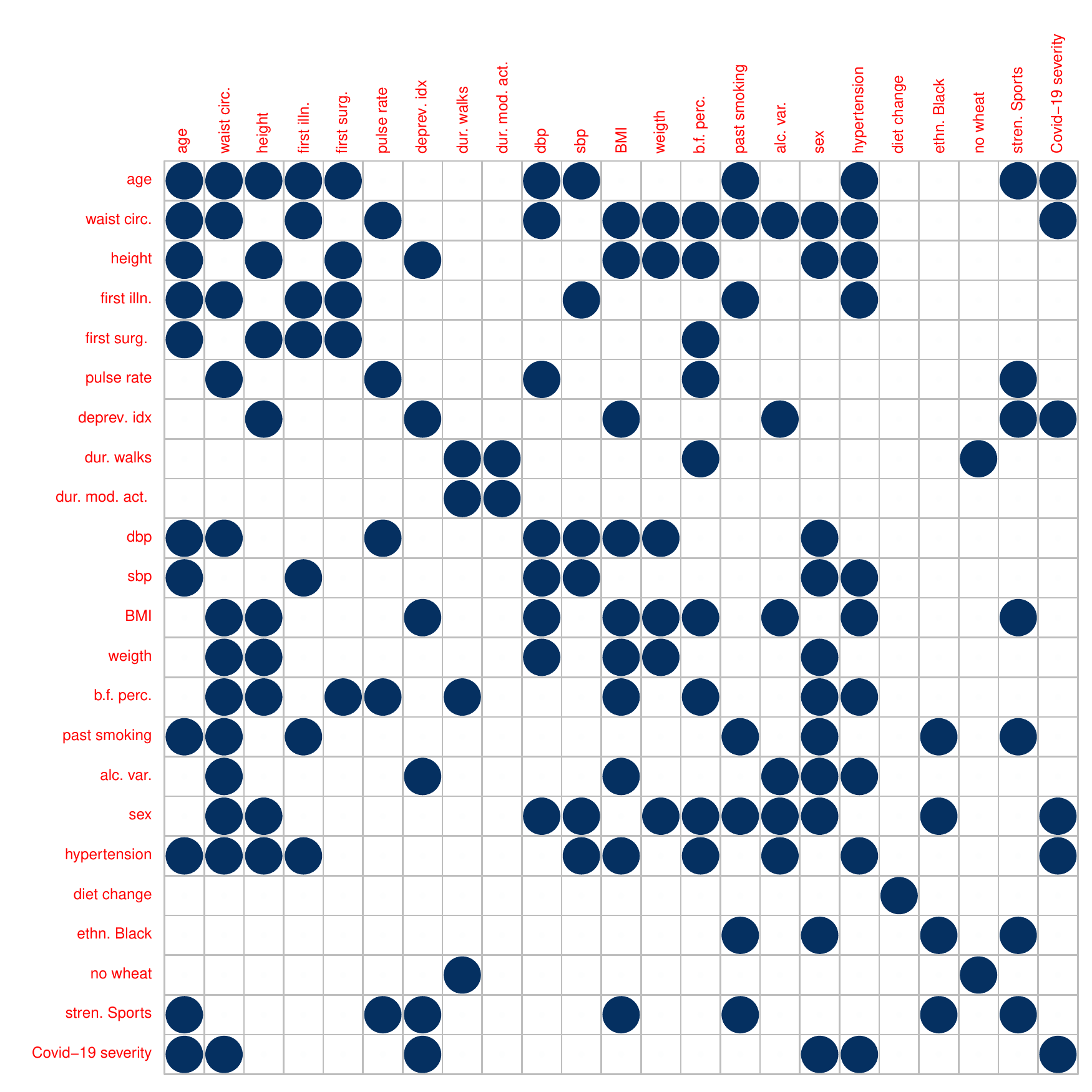}
    \caption{Plot of the estimated adjacency matrix of data set A.}
    \label{fig:corrplot_admat_A}
\end{figure}
\begin{figure}[ht]
    \centering
    \includegraphics[scale=.6]{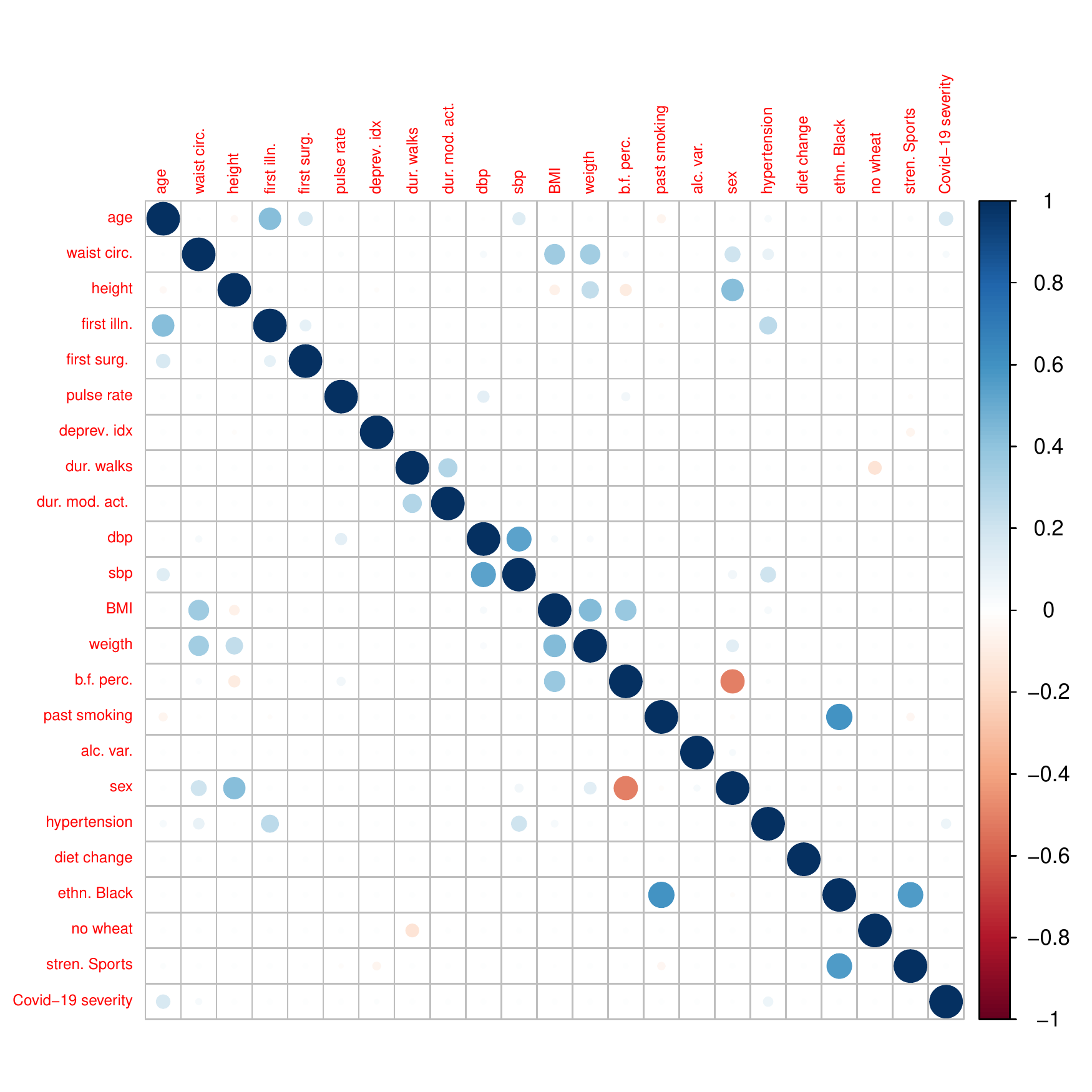}
    \caption{Plot of the estimated precision matrix of data set A.}
    \label{fig:corrplot_omega_A}
\end{figure}
The analysis includes $n=8672$ observations on $d=712$ variables (risk factors and covariates with less than 40\% missingness). Missing values were imputed using \texttt{missForest} R-package using default settings. Variables expressing more than $20$ states were treated as continuous. The remaining data include $665$ binary variables, $25$ count variables, and $8$ categorical variables. Many of the binary variables represent the status for relatively rare conditions. This means that the share of the minority class of these indicators (i.e., the fraction of samples with the least frequent value of the variable) can often be minimal. To understand the effects of such rare events on the analysis, we define
three data sets (named A, B, and C) with inclusion rules requiring respectively at least a $25\%, 2\%, 1\%$ share of observations falling into the minority class.

\subsection{Results}

We present the results of a joint analysis of the variables using real UK Biobank data. We emphasize that the analysis aims to illustrate the proposed estimators' behavior and not fully understand the risk factors for severe COVID-19. There has been much work done on factors influencing the risk of severe COVID-19 and its treatment
\citep[see, among others,][]{williamson2020,berlin2020} and we direct the interested reader to the references for further information.

Table \ref{tab:covid_associations} gives a summary of the estimated links (indicated as a visualization of the partial correlations) between the variables (including COVID-19 severity). Considering in particular links to COVID-19 severity, we see that \texttt{age}, \texttt{waist circ.}, \texttt{hypertension}, \texttt{heart attack} and \texttt{diabetes} are quite stable links throughout the different data sets. The effect sizes in terms of partial correlations are penalized and should be interpreted in relative terms. In particular, \texttt{age} retains a relatively large signal, which is in line with the known strong influence of age on COVID-19 severity \citep[see, e.g.][]{williamson2020}.

Finally, we present more detailed results of the analysis of data set A. Figure \ref{fig:corrplot_admat_A} shows the estimated adjacency matrix and Figure \ref{fig:corrplot_omega_A} depicts the estimated precision matrix $\hat{\boldsymbol\Omega}_{A}$. These results highlight the type of output, spanning different kinds of variables, that is readily available from the proposed method.

\subsection{Variable Description for real-world data application}

Table \ref{tab:var_descr} gives an overview of the variables in the UK Biobank data set.

    \begin{longtable}[c]{lp{8cm}}
    \caption{Variable description of the real world application \label{tab:var_descr}}\\    
    \\[-1.8ex]\hline 
    \hline \\[-1.8ex] 
    Variable Name 	&	Description	\\ 
    \hline \\[-1.8ex] 
        age	&	age in. years in 2020	\\
        waist circ.	&	waist circumference in cm	\\
        height	&	standing in height in cm 	\\
        first illn.	&	age at which illness first occurred 	\\
        first surg. 	&	age at which operation was done first	\\
        pulse rate	&	pulse rate measured in bpm 	\\
        deprev. idx	&	Townsend deprivation index at recruitment	\\
        dur. walks	&	duration of walks in minutes per day	\\
        dur. mod. act. 	&	duration of moderate activity in minutes per day	\\
        dbp	&	diastolic blood pressure in mmHg	\\
        sbp	&	systolic blood pressure in mmHg	\\
        BMI	&	in kg/m2	\\
        weigth	&	in kg	\\
        b.f. perc.	&	body fat percentage in \%	\\
        walking	&	number of days per week walked 10+ minutes 	\\
        mod. phys. act.	&	number of days per week of moderate physical activity  10+ minutes 	\\
        vig. phys. act.	&	number of days per week of vigorous physical activity 10+ minutes 	\\
        cheese	&	answer to ``How often do you eat cheese per week?" 	\\
        stair climb.	&	answer to "At home, during the last 4 weeks, about how many times a DAY do you climb a flight of stairs? (approx 10 steps)" 	\\
        curr. smoking	&	categorial, "Do you smoke tobacco now?" (yes, no, occasionally)	\\
        past smoking	&	categorial, "How often did you smoke tobacco?" (never, once/twice, occasionally, on most days)	\\
        diet var.	&	categorial, "Does your diet change?" (never, sometimes, often)	\\
        alc. freq.	&	categorial, "How often do you drink alcohol?" (never, special occasions only, 1-3 per month, 1-2 per week, 3-4 per week, almost daily)	\\
        alc. var.	&	categorial, "Compared to 10 years ago, do you drink?" (more, about the same, less)	\\
        sex	&	binary indicator with 0=female, 1=male	\\
        hypertension	&	hypertension, binary indicator with 0=no, 1=yes	\\
        angina	&	angina, binary indicator with 0=no, 1=yes	\\
        heart attack	&	heart attack, binary indicator with 0=no, 1=yes	\\
        stroke	&	stroke, binary indicator with 0=no, 1=yes	\\
        dvt	&	deep venous thrombosis, binary indicator with 0=no, 1=yes	\\
        asthma	&	asthma, binary indicator with 0=no, 1=yes	\\
        chr. bronch.	&	emphysema/chronic bronchitis,  binary indicator with 0=no, 1=yes	\\
        gord	&	gastro-oesophageal reflux/gastric reflux, binary indicator with 0=no, 1=yes	\\
        ibs	&	irritable bowel syndrome, binary indicator with 0=no, 1=yes	\\
        gall stones	&	cholelithiasis/gall stones, binary indicator with 0=no, 1=yes	\\
        kidn./bladder stone 	&	kidney stone/ureter stone/bladder stone, binary indicator with 0=no, 1=yes	\\
        diabetes	&	diabetes, binary indicator with 0=no, 1=yes	\\
        diabtes 2	&	type 2 diabetes, binary indicator with 0=no, 1=yes	\\
        myxoedema	&	hypothyroidism/myxoedema, binary indicator with 0=no, 1=yes	\\
        migraine	&	migraine, binary indicator with 0=no, 1=yes	\\
        glaucoma	&	glaucoma, binary indicator with 0=no, 1=yes	\\
        cataract	&	cataract, binary indicator with 0=no, 1=yes	\\
        depression	&	depression, binary indicator with 0=no, 1=yes	\\
        panic attacks	&	anxiety/panic attacks, binary indicator with 0=no, 1=yes	\\
        back probl. 	&	back problems, binary indicator with 0=no, 1=yes	\\
        osteoporosis	&	osteoporosis, binary indicator with 0=no, 1=yes	\\
        spine arthr.	&	spine arthritis/spondylitis, binary indicator with 0=no, 1=yes	\\
        slipped disc	&	prolapsed disc/slipped disc, binary indicator with 0=no, 1=yes	\\
        anaemia 	&	iron deficiency anaemia, binary indicator with 0=no, 1=yes	\\
        ut. fibroids	&	uterine fibroids, binary indicator with 0=no, 1=yes	\\
        allerg. rhinitis	&	heyfever/allergic rhinitis, binary indicator with 0=no, 1=yes	\\
        enlarged prost.	&	enlarged prostate, binary indicator with 0=no, 1=yes	\\
        pneumonia	&	pneumonia, binary indicator with 0=no, 1=yes	\\
        endometr.	&	endometriosis, binary indicator with 0=no, 1=yes	\\
        ear disor.	&	ear/vestibular disorder, binary indicator with 0=no, 1=yes	\\
        headaches	&	headaches (not migraine), binary indicator with 0=no, 1=yes	\\
        ecz./dermat.	&	eczema/dermatitis, binary indicator with 0=no, 1=yes	\\
        psoriasis	&	psoriasis, binary indicator with 0=no, 1=yes	\\
        div. disease	&	diverticular disease/diverticulitis, binary indicator with 0=no, 1=yes	\\
        osteoarthr.	&	osteoarthritis, binary indicator with 0=no, 1=yes	\\
        gout	&	gout, binary indicator with 0=no, 1=yes	\\
        high chol.	&	high cholesterol, binary indicator with 0=no, 1=yes	\\
        hiat. hern.	&	hiatus hernia, binary indicator with 0=no, 1=yes	\\
        sciatica	&	sciatica, binary indicator with 0=no, 1=yes	\\
        appendic.	&	appendicitis, binary indicator with 0=no, 1=yes	\\
        back pain	&	back pain, binary indicator with 0=no, 1=yes	\\
        arthritis	&	arthritis (nos), binary indicator with 0=no, 1=yes	\\
        measles	&	measles/morbillivirus, binary indicator with 0=no, 1=yes	\\
        chickpox	&	chickenpox, binary indicator with 0=no, 1=yes	\\
        tonsillitis	&	tonsillitis, binary indicator with 0=no, 1=yes	\\
        ptca	&	coronary angioplasty (ptca)+/-stent, binary indicator with 0=no, 1=yes	\\
        ear surg.	&	ear surgery, binary indicator with 0=no, 1=yes	\\
        sinus surg.	&	nasal/sinus,nose surgery, binary indicator with 0=no, 1=yes	\\
        vasectomy	&	vasectomy, binary indicator with 0=no, 1=yes	\\
        soft tiss. surg.	&	mucsle/soft tissue surgery, binary indicator with 0=no, 1=yes	\\
        hip repl.	&	hip replacement/revision, binary indicator with 0=no, 1=yes	\\
        knee repl.	&	knee replacement/revision, binary indicator with 0=no, 1=yes	\\
        spine surg.	&	spine or back surgery, binary indicator with 0=no, 1=yes	\\
        bil. ooph.	&	bilateral oophorectomy, binary indicator with 0=no, 1=yes	\\
        hysterect.	&	hysterectomy, binary indicator with 0=no, 1=yes	\\
        steril.	&	sterilisation, binary indicator with 0=no, 1=yes	\\
        lumpect.	&	lumpectomy, binary indicator with 0=no, 1=yes	\\
        ing. hernia rep.	&	inguinal/femoral hernia repair, binary indicator with 0=no, 1=yes	\\
        umb. hernia rep.	&	umbilical hernia repair, binary indicator with 0=no, 1=yes	\\
        cataract extr.	&	catarct extraction/lens implant, binary indicator with 0=no, 1=yes	\\
        red./fix. bone frac.	&	reduction or fixationof bone fracture, binary indicator with 0=no, 1=yes	\\
         cholecystect. 	&	cholecystectomy/gall bladder removal, binary indicator with 0=no, 1=yes	\\
        appendicect.	&	appendicectomy, binary indicator with 0=no, 1=yes	\\
        c-sec.	&	caesarian section, binary indicator with 0=no, 1=yes	\\
        tonsillest.	&	tonsillectomy, binary indicator with 0=no, 1=yes	\\
        var. vein surg.	&	varicose vein surgery, binary indicator with 0=no, 1=yes	\\
        wisd. teeth surg.	&	wisdom teeth surgery, binary indicator with 0=no, 1=yes	\\
        piles surg.	&	haemorroidectomy/piles surgery/banding of piles, binary indicator with 0=no, 1=yes	\\
        male circ.	&	male circumcision, binary indicator with 0=no, 1=yes	\\
        squint corr.	&	squint correction, binary indicator with 0=no, 1=yes	\\
        arthrosc.	&	arthroscopy (nos), binary indicator with 0=no, 1=yes	\\
        foot surg.	&	foot surgery, binary indicator with 0=no, 1=yes	\\
        knee surg.	&	knee surgery (not replacement), binary indicator with 0=no, 1=yes	\\
        shoulder surg.	&	shoulder surgery, binary indicator with 0=no, 1=yes	\\
        car. tunn. surg.	&	carpal tunnel surgery, binary indicator with 0=no, 1=yes	\\
        valg. surg.	&	bunion/hallus valgus surgery, binary indicator with 0=no, 1=yes	\\
        rem. mole	&	removal of mole/skin lesion, binary indicator with 0=no, 1=yes	\\
        ov. cyst. rem.	&	ovarian cyst removal/surgery, binary indicator with 0=no, 1=yes	\\
        d+c	&	dilatation and curettage, binary indicator with 0=no, 1=yes	\\
        cone biops.	&	cone biopsy, binary indicator with 0=no, 1=yes	\\
        endosc.	&	endoscopy/gastroscopy, binary indicator with 0=no, 1=yes	\\
        colonosc.	&	colonoscopy/sigmoidoscopy, binary indicator with 0=no, 1=yes	\\
        laparosc.	&	laparoscopy, binary indicator with 0=no, 1=yes	\\
        rhinoplast.	&	rhinoplasty/nose surgery, binary indicator with 0=no, 1=yes	\\
        tonsil surg.	&	tonsillectomy/tonsil surgery, binary indicator with 0=no, 1=yes	\\
        ing. hern. rep.	&	inguinal hernia repair, binary indicator with 0=no, 1=yes	\\
        illn. ind. diet	&	Major dietary changes in the last 5 years because of illness, binary indicator with 0=no, 1=yes	\\
        diet change	&	Major dietary changes in the last 5 years  because of other reason, binary indicator with 0=no, 1=yes	\\
        ethn. Mixed	&	Ethnicity - mixed, binary indicator with 0=no, 1=yes	\\
        ethn. Asian	&	Ethnicity - Asian, binary indicator with 0=no, 1=yes	\\
        ethn. Black	&	Ethnicity - Black, binary indicator with 0=no, 1=yes	\\
        no eggs	&	Never eat eggs or foods containing eggs, binary indicator with 0=no, 1=yes	\\
        no dairy	&	Never dairy products, binary indicator with 0=no, 1=yes	\\
        no wheat	&	Never eat wheat, binary indicator with 0=no, 1=yes	\\
        no sugar	&	Never eat sugar or foods/drinks containing sugar, binary indicator with 0=no, 1=yes	\\
        walk. f. pleas.	&	Types of physical activity in last 4 weeks - walking for pleasure, binary indicator with 0=no, 1=yes	\\
        exercises	&	Types of physical activity in last 4 weeks - other exercises (swimming, bowling etc.), binary indicator with 0=no, 1=yes	\\
        stren. Sports	&	Types of physical activity in last 4 weeks -  strenuous sports, binary indicator with 0=no, 1=yes	\\
        Covid-19 severity	&	Covid-19 severity, binary indicator with 0=mild outcome and 1=severe outcome	\\
    \hline
\end{longtable}











\clearpage\newpage
\bibliographystyle{imsart-nameyear}
\bibliography{refs}       

\end{document}